\documentclass{article}  

\usepackage{etoolbox}
\usepackage{xfrac}
\newcommand{\arxiv}[1]{\iftoggle{iclr}{}{#1}}
\newcommand{\iclr}[1]{\iftoggle{iclr}{#1}{}}
\newtoggle{iclr}
\newcommand{\loose}{\looseness=-1}

\usepackage{adjustbox}

\usepackage[utf8]{inputenc} %
\usepackage[T1]{fontenc}    %
\usepackage{url}            %
\usepackage{booktabs}       %
\usepackage{amsfonts}       %
\usepackage{nicefrac}       %
\usepackage{microtype}      %

\usepackage{tocloft}            %

\usepackage{enumitem}

\usepackage{breakcites}
\usepackage{nicefrac}

\newtoggle{draft}
\togglefalse{draft}

\usepackage{mathrsfs}

\usepackage{algorithm}
\usepackage{verbatim}
\usepackage[noend]{algpseudocode}
\newcommand{\multiline}[1]{\parbox[t]{\dimexpr\linewidth-\algorithmicindent}{#1}}

\usepackage{multicol}

\usepackage{colortbl}

\usepackage{setspace}

\usepackage{transparent}

\usepackage{inconsolata}
\usepackage[scaled=.90]{helvet}
\usepackage{xspace}

\usepackage{pifont}

\iclr{
\usepackage{iclr2024_conference}
\usepackage{times}
}

\arxiv{
\usepackage[letterpaper, left=1in, right=1in, top=1in, bottom=1in]{geometry}
\PassOptionsToPackage{hypertexnames=false}{hyperref}  %
\usepackage{parskip}
\usepackage[dvipsnames]{xcolor}
\usepackage[colorlinks=true, linkcolor=blue!70!black, citecolor=blue!70!black,urlcolor=black,breaklinks=true]{hyperref}
\hypersetup{linktoc=all}
}

\iclr{
\PassOptionsToPackage{hypertexnames=false}{hyperref}  %
\usepackage{parskip}
\usepackage[dvipsnames]{xcolor}
\usepackage[colorlinks=true, linkcolor=blue!70!black, citecolor=blue!70!black,urlcolor=black,breaklinks=true]{hyperref}
\hypersetup{linktoc=all}
}

\usepackage{microtype}
\usepackage{hhline}

\makeatletter
\newcommand{\neutralize}[1]{\expandafter\let\csname c@#1\endcsname\count@}
\makeatother

\newenvironment{thmmod}[2]
  {\renewcommand{\thetheorem}{\ref*{#1}#2}%
   \neutralize{theorem}\phantomsection
   \begin{theorem}}
  {\end{theorem}}

   \newenvironment{asmmod}[2]
  {\renewcommand{\theassumption}{\ref*{#1}#2}%
   \neutralize{assumption}\phantomsection
   \begin{assumption}}
  {\end{assumption}}

\usepackage{algorithm}

\arxiv{
\usepackage{natbib}
\bibliographystyle{plainnat}
\bibpunct{(}{)}{;}{a}{,}{,}
}

\iclr{
\bibliographystyle{plainnat}
\bibpunct{(}{)}{;}{a}{,}{,}
}

\usepackage{amsthm}
\usepackage{mathtools}
\usepackage{amsmath}
\usepackage{bbm}
\usepackage{amsfonts}
\usepackage{amssymb}

\usepackage{xpatch}

\usepackage{thmtools}
\usepackage{thm-restate}
\declaretheorem[name=Theorem,parent=section]{theorem}
\declaretheorem[name=Lemma,parent=section]{lemma}
\declaretheorem[name=Assumption, parent=section]{assumption}
\declaretheorem[name=Condition, parent=section]{condition}
\declaretheorem[qed=$\triangleleft$,name=Example,style=definition, parent=section]{example}
\declaretheorem[name=Remark, parent=section]{remark}
\declaretheorem[name=Proposition, parent=section]{proposition}

\makeatletter
  \renewenvironment{proof}[1][Proof]%
  {%
   \par\noindent{\bfseries\upshape {#1.}\ }%
  }%
  {\qed\newline}
  \makeatother

\theoremstyle{definition}  %

\theoremstyle{plain}
\newtheorem{corollary}{Corollary}[section]

\newtheorem{definition}{Definition}[section]

\xpatchcmd{\proof}{\itshape}{\normalfont\proofnameformat}{}{}
\newcommand{\proofnameformat}{\bfseries}

\usepackage[nameinlink,capitalize]{cleveref}

\newcommand{\pref}[1]{\cref{#1}}
\newcommand{\pfref}[1]{Proof of \pref{#1}}

\renewcommand{\eqref}[1]{\texorpdfstring{\hyperref[#1]{Eq.~(\ref*{#1})}}{Eq.~(\ref*{#1})}}

\crefformat{equation}{#2(#1)#3}
\Crefformat{equation}{#2(#1)#3}

\Crefformat{figure}{#2Figure #1#3}
\Crefname{assumption}{Assumption}{Assumptions}
\Crefformat{assumption}{#2Assumption #1#3}
\Crefname{subsubsection}{Section}{Sections}
\crefformat{subsubsection}{#2Section #1#3}
\Crefformat{subsubsection}{#2Section #1#3}

\usepackage{crossreftools}
\newcommand{\creftitle}[1]{\crtcref{#1}}

\usepackage{xparse}

\ExplSyntaxOn
\DeclareDocumentCommand{\XDeclarePairedDelimiter}{mm}
 {
  \__egreg_delimiter_clear_keys: %
  \keys_set:nn { egreg/delimiters } { #2 }
  \use:x %
   {
    \exp_not:n {\NewDocumentCommand{#1}{sO{}m} }
     {
      \exp_not:n { \IfBooleanTF{##1} }
       {
        \exp_not:N \egreg_paired_delimiter_expand:nnnn
         { \exp_not:V \l_egreg_delimiter_left_tl }
         { \exp_not:V \l_egreg_delimiter_right_tl }
         { \exp_not:n { ##3 } }
         { \exp_not:V \l_egreg_delimiter_subscript_tl }
       }
       {
        \exp_not:N \egreg_paired_delimiter_fixed:nnnnn 
         { \exp_not:n { ##2 } }
         { \exp_not:V \l_egreg_delimiter_left_tl }
         { \exp_not:V \l_egreg_delimiter_right_tl }
         { \exp_not:n { ##3 } }
         { \exp_not:V \l_egreg_delimiter_subscript_tl }
       }
     }
   }
 }

\keys_define:nn { egreg/delimiters }
 {
  left      .tl_set:N = \l_egreg_delimiter_left_tl,
  right     .tl_set:N = \l_egreg_delimiter_right_tl,
  subscript .tl_set:N = \l_egreg_delimiter_subscript_tl,
 }

\cs_new_protected:Npn \__egreg_delimiter_clear_keys:
 {
  \keys_set:nn { egreg/delimiters } { left=.,right=.,subscript={} }
 }

\cs_new_protected:Npn \egreg_paired_delimiter_expand:nnnn #1 #2 #3 #4
 {%
  \mathopen{}
  \mathclose\c_group_begin_token
   \left#1
   #3
   \group_insert_after:N \c_group_end_token
   \right#2
   \tl_if_empty:nF {#4} { \c_math_subscript_token {#4} }
 }
\cs_new_protected:Npn \egreg_paired_delimiter_fixed:nnnnn #1 #2 #3 #4 #5
 {
  \mathopen{#1#2}#4\mathclose{#1#3}
  \tl_if_empty:nF {#5} { \c_math_subscript_token {#5} }
 }
\ExplSyntaxOff

\XDeclarePairedDelimiter{\supnorm}{
  left=\lVert,
  right=\rVert,
  subscript=\infty
  }

\DeclarePairedDelimiter{\abs}{\lvert}{\rvert} %
\DeclarePairedDelimiter{\brk}{[}{]}
\DeclarePairedDelimiter{\crl}{\{}{\}}
\DeclarePairedDelimiter{\prn}{(}{)}
\DeclarePairedDelimiter{\nrm}{\|}{\|}
\DeclarePairedDelimiter{\tri}{\langle}{\rangle}

\DeclarePairedDelimiter{\ceil}{\lceil}{\rceil}

\DeclareMathOperator{\En}{\mathbb{E}}

\newcommand{\Ehat}{\wh{\bbE}}

\DeclareMathOperator*{\argmin}{arg\,min} %
\DeclareMathOperator*{\argmax}{arg\,max}

\newcommand{\wt}[1]{\widetilde{#1}}
\newcommand{\wh}[1]{\widehat{#1}}
\newcommand{\wb}[1]{\widebar{#1}}

\def\ddefloop#1{\ifx\ddefloop#1\else\ddef{#1}\expandafter\ddefloop\fi}
\def\ddef#1{\expandafter\def\csname bb#1\endcsname{\ensuremath{\mathbb{#1}}}}
\ddefloop ABCDEFGHIJKLMNOPQRSTUVWXYZ\ddefloop
\def\ddefloop#1{\ifx\ddefloop#1\else\ddef{#1}\expandafter\ddefloop\fi}
\def\ddef#1{\expandafter\def\csname b#1\endcsname{\ensuremath{\mathbf{#1}}}}
\ddefloop ABCDEFGHIJKLMNOPQRSTUVWXYZ\ddefloop
\def\ddef#1{\expandafter\def\csname sf#1\endcsname{\ensuremath{\mathsf{#1}}}}
\ddefloop ABCDEFGHIJKLMNOPQRSTUVWXYZ\ddefloop
\def\ddef#1{\expandafter\def\csname c#1\endcsname{\ensuremath{\mathcal{#1}}}}
\ddefloop ABCDEFGHIJKLMNOPQRSTUVWXYZ\ddefloop
\def\ddef#1{\expandafter\def\csname h#1\endcsname{\ensuremath{\widehat{#1}}}}
\ddefloop ABCDEFGHIJKLMNOPQRSTUVWXYZ\ddefloop
\def\ddef#1{\expandafter\def\csname hc#1\endcsname{\ensuremath{\widehat{\mathcal{#1}}}}}
\ddefloop ABCDEFGHIJKLMNOPQRSTUVWXYZ\ddefloop
\def\ddef#1{\expandafter\def\csname t#1\endcsname{\ensuremath{\widetilde{#1}}}}
\ddefloop ABCDEFGHIJKLMNOPQRSTUVWXYZ\ddefloop
\def\ddef#1{\expandafter\def\csname tc#1\endcsname{\ensuremath{\widetilde{\mathcal{#1}}}}}
\ddefloop ABCDEFGHIJKLMNOPQRSTUVWXYZ\ddefloop
\def\ddefloop#1{\ifx\ddefloop#1\else\ddef{#1}\expandafter\ddefloop\fi}
\def\ddef#1{\expandafter\def\csname scr#1\endcsname{\ensuremath{\mathscr{#1}}}}
\ddefloop ABCDEFGHIJKLMNOPQRSTUVWXYZ\ddefloop

\newcommand{\ind}{\mathbbm{1}}    %

\newcommand{\veps}{\varepsilon}

\newcommand{\ldef}{\vcentcolon=}
\newcommand{\rdef}{=\vcentcolon}

\newcommand{\fakepar}[1]{\arxiv{\paragraph{#1}}\iclr{\noindent\textbf{#1.}}}

\newcommand{\Pilatent}{\Pi_{\mathrm{latent}}}

\newcommand{\phistar}{\phi_{\star}}

\newcommand{\mustar}{\mu^{\star}}

\newcommand{\covdist}{coverability distribution\xspace}

\newcommand{\cWbar}{\wb{\cW}}

\newcommand{\wtil}{\check{w}}

\newcommand{\Risk}{\mathrm{\mathbf{Risk}}}
\newcommand{\RiskOff}{\mathrm{\mathbf{Risk}}_{\mathsf{off}}}

\renewcommand{\c}{\mathrm{c}}

\renewcommand{\emptyset}{\varnothing}

\newcommand{\filt}{\mathscr{F}}

\newcommand{\M}[1]{^{{\scriptscriptstyle M}}}  %

\newcommand{\pistar}{\pi^{\star}}

\newcommand{\pihat}{\wh{\pi}}

\newcommand{\dbar}{\bar{d}}

\newcommand{\Reg}{\mathrm{\mathbf{Reg}}}

\arxiv{
\newcommand{\algcomment}[1]{\textcolor{blue!70!black}{\transparent{0.5}\footnotesize{\texttt{\textbf{//\hspace{2pt}#1}}}}}

}
\iclr{
\newcommand{\algcomment}[1]{\textcolor{blue!70!black}{\transparent{0.5}\scriptsize{\texttt{\textbf{//\hspace{2pt}#1}}}}}

}

\newcommand{\approxleq}{\lesssim}

\renewcommand{\ind}[1]{^{\scriptscriptstyle(#1)}}

\newcommand{\bigoh}{O}
\newcommand{\bigoht}{\wt{O}}

\newcommand{\bigomt}{\wt{\Omega}}

\newcommand{\indic}{\mathbb{I}}

\newcommand{\poly}{\mathrm{poly}}

\newcommand{\Qstar}{Q^{\star}}

\newcommand{\unif}{\mathsf{Unif}}

\newcommand{\supp}{\mathrm{supp}}

\def\multiset#1#2{\ensuremath{\left(\kern-.3em\left(\genfrac{}{}{0pt}{}{#1}{#2}\right)\kern-.3em\right)}}

\renewcommand{\emptyset}{\varnothing}

\input{widebar}

\renewcommand{\epsilon}{\varepsilon}

\newcommand\numberthis{\addtocounter{equation}{1}\tag{\theequation}} 

\newcommand{\clip}[2]{\mathsf{clip}_{#2}\brk*{#1}}
\newcommand{\clipX}[3]{\mathsf{clip}_{#3}\brk[#1]{#2}}  
\newcommand{\Df}{[\Delta_h f]} 
\newcommand{\Dfh}{[\Delta_h \wh f]} 
\newcommand{\whDwhf}{[\wh \Delta_h \wh f]} 
\newcommand{\whDf}{[\wh \Delta_h f]} 
\newcommand{\Dft}{[\Delta_h f\ind{t}]}

\newcommand{\cw}{\check{w}}  

\newcommand{\epsapp}{\epsilon_{\textnormal{apx}}} 

\DeclareMathOperator{\Cond}{\mathsf{Assumption}}

\newcommand\restr[2]{{%
  \left.\kern-\nulldelimiterspace %
  #1 %
  \vphantom{\big|} %
  \right|_{#2} %
  }}

\newcommand{\HtO}{\HybridAlg}

\newcommand{\Mixture}{\textsf{Mixture}} 
\newcommand{\bcW}{\widebar{\cW}}  
\newcommand{\barw}{\bar{w}} 

\newtheorem{algoshort}[theorem]{Algorithm}

\usepackage[suppress]{color-edits}

 \addauthor{df}{ForestGreen}
 \addauthor{as}{Red} 
 \addauthor{pa}{Magenta}
 \addauthor{tx}{violet}
 \addauthor{cut}{Orange}
  \newcommand{\tx}[1]{}
 \addauthor{nj}{Blue}
  \newcommand{\df}[1]{}
  \renewcommand{\pacomment}[1]{}
  \renewcommand{\ascomment}[1]{}
  \renewcommand{\dfcomment}[1]{}
  
\arxiv{
\usepackage[final]{showlabels}

}

\makeatletter
\let\OldStatex\Statex
\renewcommand{\Statex}[1][3]{%
  \setlength\@tempdima{\algorithmicindent}%
  \OldStatex\hskip\dimexpr#1\@tempdima\relax}
\makeatother

 \usepackage{accents}

\addtocontents{toc}{\protect\setcounter{tocdepth}{0}}

\let\oldparagraph\paragraph
\renewcommand{\paragraph}[1]{\oldparagraph{#1.}}

\newcommand{\dtil}{\wt{d}}
\newcommand{\Ccov}{C_{\mathsf{cov}}}
\newcommand{\Cstar}{C_{\star}}
\newcommand{\Cc}{C}  
\newcommand{\Cconc}{C_{\mathsf{conc}}}
\newcommand{\Cconct}{C_{\mathsf{conc},2,h}}

\newcommand{\algo}{\mathrm{\mathbf{Alg}}_{\mathsf{off}}} 
\newcommand{\Ccc}[5]{\mathsf{CC}_{{#2}}(#3,#4,#5)} %

\newcommand{\piterm}{\mathsf{err}_{\mathsf{off}}}
\newcommand{\ascale}{\mathfrak{a}_\gamma}
\newcommand{\bscale}{\mathfrak{b}_\gamma}
\newcommand{\fa}{\mathfrak{a}}
\newcommand{\fb}{\mathfrak{b}}
\newcommand{\Ccbounded}{$\mathsf{CC}$-bounded\xspace}

\newcommand{\Alg}{{\normalfont\textsc{Glow}}\xspace} %

\newcommand{\HyAlg}{{\normalfont\textsc{HyGlow}}\xspace}

\newcommand{\Fqi}{{\normalfont\textsc{Fqi}}\xspace}
\newcommand{\Crmabo}{{\normalfont\textsc{Mabo.cr}}\xspace}

\newcommand{\HybridAlg}{\textsc{H$_2$O}\xspace}

\allowdisplaybreaks

\usepackage{yfonts}

\usepackage{amsthm}
\usepackage{thmtools,thm-restate}

\arxiv{
\title{Harnessing Density Ratios for Online Reinforcement Learning}
\author{
Philip Amortila\thanks{Authors listed in alphabetical order.} 
\\
\normalsize
\href{mailto:philipa4@illinois.edu}{\texttt{philipa4@illinois.edu}}
\and
Dylan J. Foster
\\
\normalsize
\href{mailto:dylanfoster@microsoft.com}{\texttt{dylanfoster@microsoft.com}}
\and
Nan Jiang
\\
\normalsize
\href{mailto:nanjiang@illinois.edu}{\texttt{nanjiang@illinois.edu}}
\and
Ayush Sekhari
\\
\normalsize
\href{mailto:sekhari@mit.edu}{\texttt{sekhari@mit.edu}}
\and
Tengyang Xie
\\
\normalsize
\href{mailto:tengyangxie@microsoft.com}{\texttt{tengyangxie@microsoft.com}}
}
\date{} 
}

\iclr{
   \title{Harnessing Density Ratios for \\ Online Reinforcement Learning}

\iclrfinalcopy
\author{\href{mailto:philipa4@illinois.edu}{Philip Amortila}\thanks{Authors listed in alphabetical order. Full version appears at \href{https://arxiv.org/abs/2401.09681}{{\color{blue}[arXiv:2401.09681]}}.}
\\
UIUC %
\And
\href{mailto:dylanfoster@microsoft.com}{Dylan J. Foster}
\\
Microsoft Research %
\And
\href{mailto:nanjiang@illinois.edu}{Nan Jiang}
\\
UIUC %
\And
\href{mailto:sekhari@mit.edu}{Ayush Sekhari}
\\
\normalsize
MIT %
\And
\href{mailto:tengyangxie@microsoft.com}{Tengyang Xie}
\\
\normalsize
Microsoft Research %
} 
}

\begin{document} 
\maketitle 

\begin{abstract} 
 
The theories of offline and online reinforcement learning, despite having evolved in parallel, have \dfedit{begun to show signs of the possibility for a} unification, with algorithms and analysis techniques for one setting often having natural counterparts in the other. However, the notion of \textit{density ratio modeling}, an emerging paradigm in offline RL, has been largely absent from online RL, perhaps for good reason: the very existence and boundedness of density ratios relies on access to an exploratory dataset with good coverage, but the core challenge in online RL is to collect such a dataset without having one to start.

In this work we show---perhaps surprisingly---that density ratio-based algorithms have online counterparts.  Assuming only the existence of an exploratory distribution with good coverage, a structural condition known as \emph{coverability} \citep{xie2022role}, we give a new algorithm (\Alg) that uses density ratio realizability and value function realizability to perform sample-efficient online exploration. \Alg addresses unbounded density ratios via careful use of \emph{truncation}, and combines this with optimism to guide exploration.

\Alg is computationally inefficient; we complement it with a \dfedit{more} efficient counterpart, \HyAlg, for the \emph{Hybrid RL} setting \citep{song2022hybrid} wherein online RL is augmented with additional offline data.
\HyAlg is derived as a special case of a more general meta-algorithm that provides a provable black-box reduction from hybrid RL to offline RL, which may be of independent interest.

\end{abstract}  

\section{Introduction}
\label{sec:intro}
A fundamental problem in reinforcement learning (RL) is to understand what
modeling assumptions and algorithmic principles lead to
sample-efficient learning guarantees. Investigation into algorithms
for sample-efficient reinforcement learning
has primarily focused on two separate
formulations: \emph{Offline reinforcement learning}, where a learner must
optimize a policy from logged transitions and rewards, and
\emph{online reinforcement learning}, where the learner can gather new
data by interacting with the environment; both formulations share the
common goal of learning a near-optimal policy. For the most part, the bodies of 
research on offline and online reinforcement have evolved in parallel,
but they exhibit a number of curious 
similarities.  Algorithmically, many design principles for offline RL
(e.g., pessimism) have online counterparts (e.g.,
optimism), and statistically efficient algorithms for both frameworks typically
require similar \emph{representation conditions} (e.g., ability to
model state-action value functions). Yet, the frameworks have notable differences: online RL algorithms
require \emph{exploration conditions} to address the issue
of distribution shift \citep{russo2013eluder,jiang2017contextual,sun2019model,wang2020provably,du2021bilinear,jin2021bellman,foster2021statistical}, while offline RL algorithms require
conceptually distinct \emph{coverage conditions} to ensure the data
logging distribution
sufficiently covers the state space \citep{munos2003error,antos2008learning,chen2019information,xie2020q,xie2021batch, jin2021pessimism,rashidinejad2021bridging,foster2022offline,zhan2022offline}.

\arxiv{
\dfcomment{Above: Could still include ``sample efficiency is a major
  problem in practice'' and cite standard empirical papers.}
\dfcomment{rewrite below so there is less overlap w/ abstract}
}

Recently, \citet{xie2022role} exposed a deeper connection between
online and offline RL by showing that \emph{coverability}---that is, \emph{existence} of
a data distribution with good coverage for offline RL---is itself a
sufficient condition that enables sample-efficient
exploration in online RL, even when the learner has no prior knowledge
of said distribution. This suggests the possibility of a
theoretical unification of online and offline RL, but the picture remains
incomplete, and there are many gaps in our understanding. Notably, a promising emerging paradigm in offline RL makes
    use of the ability to model \emph{density ratios} (also referred
    to as \emph{marginalized importance weights} or simply \emph{weight functions}) for the
    underlying MDP. Density ratio modeling offers an alternative to
    classical value function approximation (or, approximate dynamic programming) methods
    \citep{munos2007performance,munos2008finite,chen2019information},
    as it avoids
    instability and typically succeeds under weaker representation
    conditions (requiring only realizability conditions as
    opposed to Bellman completeness-type assumptions). Yet despite extensive investigation into density
    ratio methods for offline RL---both in theory
\citep{liu2018breaking,uehara2020minimax,yang2020off,uehara2021finite,jiang2020minimax,xie2020q,zhan2022offline,chen2022offline,rashidinejad2022optimal,ozdaglar2023revisiting}
    and practice \citep{nachum2019algaedice,kostrikov2019imitation,nachum2020reinforcement,zhang2020gendice,lee2021optidice}---density ratio modeling has been conspicuously absent in
    the \emph{online} reinforcement learning. This leads us to ask:
    \begin{center}
  \emph{Can online reinforcement learning benefit from the ability
    to model density ratios?}
\end{center}
Adapting density ratio-based methods to the
online setting with provable guarantees presents a number of conceptual and technical
challenges. First, since the data distribution in online RL is
constantly changing, it is unclear \emph{what} densities one should
even attempt to model. Second, most offline reinforcement learning
algorithms require relatively stringent notions of coverage for the
data distribution. In online RL, it is unreasonable to expect data 
gathered early in the learning process to have good coverage, and 
naive algorithms may cycle or fail to explore as a result. As such, it
may not be reasonable to expect density ratio modeling to benefit 
online RL in the same fashion as offline.

\fakepar{Our contributions} 
We show that in spite of these challenges, density ratio modeling enables guarantees
for online reinforcement learning that were previously out of
reach.
\begin{itemize}[leftmargin=1em]
\item \textbf{Density ratios for online RL.} We show
  (\cref{sec:online}) that for any MDP with low
  \emph{coverability}, density ratio realizability
  and realizability of the optimal state-action value function are sufficient for
  sample-efficient online RL. This result is obtained through a new algorithm, \Alg,
  which addresses the issue of distribution shift via careful use of
  \emph{truncated} density ratios, which it combines with optimism to
  drive exploration. This complements \citet{xie2022role}, who gave \arxiv{sample complexity }guarantees for
  coverable MDPs under a stronger Bellman completeness
  assumption for the value function class.\loose
  
\item \textbf{Density ratios for hybrid RL.} Our algorithm for online RL is computationally 
  inefficient. We complement it
  (\cref{sec:hybrid}) with a \dfedit{more} efficient
  counterpart, \HyAlg, for the \emph{hybrid RL} framework \citep{song2022hybrid},
  \arxiv{in which}\iclr{where} the learner has access to additional offline data that covers a
  high-quality policy.\loose
\item \textbf{Hybrid-to-offline reductions.} To achieve the result above, we investigate a broader question:
  when can offline RL algorithms be adapted \textit{as-is} to online
  settings? We provide a new
  meta-algorithm, \HtO, which reduces hybrid RL to offline RL by repeatedly
  calling a given offline RL algorithm as a black box.
  We show that \HtO enjoys low regret whenever the black-box offline algorithm satisfies certain conditions, and demonstrate that these conditions are satisfied by a range of existing offline algorithms, thus lifting them to the hybrid RL setting.
\end{itemize}
While our results are theoretical in nature, we are optimistic that
they will lead to further investigation into the power of density
ratio modeling in online RL and inspire practical algorithms.

\arxiv{
\fakepar{Paper organization}
\cref{sec:assumption} contains
necessary background, introducing density ratio modeling and the
notion of coverability.
\cref{sec:online} presents our
main results for the online reinforcement learning framework, and \cref{sec:hybrid} presents our main
results for the hybrid framework. We conclude with discussion in
\cref{sec:discussion}. Proofs, examples, and additional discussion are
deferred to the appendix.
}

\subsection{Preliminaries}
\label{sec:setting}

\fakepar{Markov Decision Processes}
We consider an episodic reinforcement learning setting. A Markov Decision Process (MDP) is a tuple $\cM = ( \cX, \cA, P, R, H, d_1)$, where $\cX$ is the (large/potentially infinite) state space, $\cA$ is the action space, $H \in \bbN$ is the horizon, $R = \{R_h\}_{h=1}^H$ is the reward function (where $R_h : \cX \times \cA \rightarrow \Delta([0,1])$), $P = \{P_h\}_{h\leq1}$ is the transition distribution (where $P_h : \cX \times \cA \rightarrow \Delta(\cX)$), and $d_1$ is the initial state distribution. A randomized policy is a sequence of functions $\pi = \{ \pi_h: \cX \rightarrow \Delta(\cA)\}_{h=1}^H$. When a policy is executed, it generates a trajectory $(x_1,a_1,r_1), \dots, (x_H, a_H, r_h)$ via the process $a_h \sim \pi_h(x_h), r_h \sim R_h(x_h,a_h), x_{h+1} \sim P_h(x_h,a_h)$, initialized at $x_1 \sim d_1$ (we use $x_{H+1}$ to denote a deterministic terminal state with zero reward). We write $\bbP^{\pi}\brk*{\cdot}$ and $\En^{\pi}\brk*{\cdot}$ to denote the law and corresponding expectation for the trajectory under this process.\loose

For a policy \(\pi\), the expected reward for is given by $J(\pi) \coloneqq \bbE^\pi\brk[\big]{ \sum_{h=1}^H r_h}$, and the value functions given by \iclr{$V_h^\pi(x) \coloneqq \bbE^\pi\brk[\big]{\sum_{h'=h}^H r_{h'} \mid x_h=x}$, and $Q_h^\pi(x,a) \coloneqq \bbE^\pi\brk[\big]{\sum_{h'=h}^H r_{h'} \mid x_h=x,a_h=a}$.}
\arxiv{
\[
V_h^\pi(x) \coloneqq \bbE^\pi\brk*{\sum_{h'=h}^H r_{h'} \mid x_h=x}, \quad \text{ and } \quad Q_h^\pi(x,a) \coloneqq \bbE^\pi\brk*{\sum_{h'=h}^H r_{h'} \mid x_h=x,a_h=a}.
\]}
We write $\pi^\star = \{\pi^\star_h\}_{h=1}^H$ to denote an optimal deterministic policy, which maximizes $V^{\pi}$ at all states. We let $\cT_h$ denote the Bellman (optimality) operator for layer $h$, defined via \iclr{$[\cT_h f](x,a) = \En\brk[\big]{r_h+\max_{a'} f(x_{h+1},a')\mid{}x_h=x,a_h=a}$ for $f : \cX \times \cA \rightarrow \bbR$.} \arxiv{ \[
		[\cT_h f](x,a) = \En\brk*{r_h+\max_{a'} f(x_{h+1},a')\mid{}x_h=x,a_h=a}
		\]
	for $f : \cX \times \cA \rightarrow \bbR$.} %

\fakepar{Online RL} 
In the online reinforcement learning framework, the learner repeatedly interacts with an unknown MDP by executing a policy and observing the resulting trajectory. The goal is to maximize total reward.  Formally, the protocol proceeds in $N$ rounds, where at each round $t \in [N]$, the learner	 selects a policy $\pi\ind{t} = \{\pi_h\ind{t}\}_{h=1}^H$ in the (unknown) underlying MDP $M^\star$ and observes the trajectory $\{(x_h\ind{t},a_h\ind{t},r_h\ind{t})\}_{h=1}^H$. Our results are most naturally stated in terms of PAC guarantees. Here, after the $N$ rounds of interaction conclude, the learner can use all of the data collected to produce a final policy $\pihat$, with the goal of minimizing 
\begin{align}
  \label{eq:risk}
  \Risk \ldef \En_{\wh \pi \sim p}  \brk*{J(\pistar) -  J(\pihat)}, 
\end{align}
where \(p \in \Delta(\Pi)\) denotes a distribution that the algorithm can use to randomize the final policy.

\arxiv{
\fakepar{Offline RL}
In offline reinforcement learning, the learner does not directly interact with $M^\star$, and is instead given a dataset of tuples $(x_h,a_h,r_h,x_{h+1})$ collected i.i.d.~according to 
$(x_h,a_h) \sim \mu_h$, $r_h \sim R_h(x_h,a_h)$, $x_{h+1} \sim P_h(x_h,a_h)$, \dfedit{where $\mu_h$ is the offline data distribution for layer $h$}. Based on the dataset, the offline algorithm produces a policy $\widehat{\pi}$ whose performance is measured by its risk, as in Equation \cref{eq:risk}; we write $\RiskOff$ when we are in the offline interaction protocol. 
}

\fakepar{Additional definitions and assumptions} We assume that rewards are normalized such that $\sum_{h=1}^H r_h \in [0,1]$ almost surely for all trajectories \citep{jiang2018open,wang2020long,zhang2021reinforcement,jin2021bellman}. To simplify presentation, we assume that $\cX$ and $\cA$ are
countable; \cutedit{ we expect that our results extend to handle continuous variables with an appropriate measure-theoretic treatment.} We define the occupancy measure for policy $\pi$ via $d_h^\pi(x,a) \coloneqq \bbP^\pi[x_h = x, a_h=a]$.\loose

\section{Problem Setup: Density Ratio
  Modeling and Coverability} 
\label{sec:assumption}

To investigate the power of density ratio modeling in online RL, we make use of \emph{function approximation}, and aim to provide sample complexity guarantees with no explicit dependence on the size of the state space. We begin by appealing to \emph{value function approximation}, a standard approach in online and offline reinforcement learning, and assume access to a value-function class $\cF\subset(\cX\times\cA\times\brk{H}\to\brk{0,1})$ that can realize the optimal value function $\Qstar$.
\begin{assumption}[Value function realizability] 
  \label{ass:Q_realizability} We have $\Qstar\in\cF$.
\end{assumption}

For $f \in \cF$, we define the greedy policy $\pi_f$ via $\pi_{f,h}(x) = \argmax_a f_h(x,a)$, with ties broken in an arbitrary consistent fashion. For the remainder of the paper, we define $\Pi \coloneqq \{ \pi_f \mid f \in \cF\}$ as the policy class induced by $\cF$, unless otherwise specified. %

\fakepar{Density ratio modeling}
While value function approximation is a natural modeling approach, prior works in both online \citep{du2019good,weisz2021exponential,wang2021exponential} and offline RL \citep{wang2020statistical,zanette2021exponential,foster2022offline} have shown that value function realizability alone is not sufficient for statistically tractable learning in many settings. As such, value function approximation methods in online \citep{zanette2020learning,jin2021bellman,xie2022role} and offline RL \citep{antos2008learning,chen2019information} typically require additional representation conditions that may not be satisfied in practice, such as the stringent \emph{Bellman completeness} assumption (i.e., $\cT_h \cF_{h+1} \subseteq \cF_{h}$).

In offline RL, a promising emerging paradigm that goes beyond pure value function approximation is to model \emph{density ratios} (or, marginalized important weights), which typically take the form \iclr{$\nicefrac{d^{\pi}_h(x,a)}{\mu_h(x,a)}$}
\arxiv{\begin{align}
  \label{eq:ratio_basic}
\frac{d^{\pi}_h(x,a)}{\mu_h(x,a)} 
\end{align}}
for a policy $\pi$, where \(\mu_h\) denotes the offline data distribution.\iclr{\footnote{Formally, in offline reinforcement learning the learner does not interact with $M^\star$ but is given a dataset of tuples $(x_h,a_h,r_h,x_{h+1})$ collected i.i.d. according to 
$(x_h,a_h) \sim \mu_h$, $r_h \sim R_h(x_h,a_h)$, $x_{h+1} \sim P_h(x_h,a_h)$.}}
 A recent line of work \citep{xie2020q,jiang2020minimax,zhan2022offline} shows, that given access to a realizable value function class and a weight function class $\cW$ that can realize the ratio \iclr{$\nicefrac{d^{\pi}_h}{\mu_h}$}\arxiv{\cref{eq:ratio_basic}} (typically either for all policies $\pi$, or for the optimal policy $\pistar$), one can learn a near-optimal policy offline in a sample-efficient fashion; such results sidestep the need for stringent value function representation conditions like Bellman completeness. To explore whether density ratio modeling has similar benefits in online RL, we make the following assumption.\loose

\begin{assumption}[Density ratio realizability]
\label{ass:W_realizability_proper} 
The learner has access to a weight function class
$\cW \subset (\cX\times\cA\times\brk{H}\to\bbR_{+})$
such that for any policy pair $\pi,\pi'\in\Pi$, and \(h \in [H]\), we have\iclr{ $  w^{\pi;\pi'}_h(x,a) \coloneqq \nicefrac{d^{\pi}_h(x,a)}{d^{\pi'}_h(x,a)}\in\cW_h$.}\footnote{We adopt the convention that $\nicefrac{x}{0} = +\infty$ when $x > 0$ and $\nicefrac{0}{0} = 1$.}
\loose
\arxiv{\[
  w^{\pi;\pi'}_h(x,a) \coloneqq \frac{d^{\pi}_h(x,a)}{d^{\pi'}_h(x,a)}\in\cW. 
\]}
\end{assumption}
\cref{ass:W_realizability_proper} \textit{does not} assume that the density ratios under consideration are finite. That is, we do not assume boundedness of the weights, and our results do not pay for their range; our algorithm will only access certain \emph{clipped} versions of the weight functions \iclr{(\cref{rem:clipped_realizability})}\arxiv{(in fact, it is sufficient to only realize certain ``clipped'' weight functions; cf. \cref{rem:clipped_realizability})}.

Compared to density ratio approaches in the offline setting, which typically require either realizability of $\nicefrac{d^{\pi}_h}{\mu_h}$ for all $\pi\in\Pi$ or realizability of $\nicefrac{d^{\pistar}_h}{\mu_h}$, where $\mu$ is the offline data distribution, we require realizability of $\nicefrac{d^{\pi}_h}{d_h^{\pi'}}$ for all pairs of policies $\pi,\pi'\in\Pi$. This assumption is natural because it facilitates transfer between historical data (which is algorithm-dependent) and future policies. \arxiv{Notably, it is weaker than assuming realizability of $\nicefrac{d^{\pi}_h}{\nu_h}$ for all $\pi\in\Pi$ and any fixed distribution $\nu$ (\cref{rem:ratio-fixed-dist}), \dfedit{and is also weaker than model-based realizability.} }We refer the reader to \cref{app:realizability} for a detailed comparison to alternative assumptions.%

\fakepar{Coverability} In addition to realizability assumptions, online RL methods require \emph{exploration conditions} \citep{russo2013eluder,jiang2017contextual,sun2019model,wang2020provably,du2021bilinear,jin2021bellman,foster2021statistical} that allow deliberately designed algorithms to control distribution shift or extrapolate to unseen states. Towards lifting density ratio modeling from offline to online RL, we make use of \emph{coverability} \citep{xie2022role}, an exploration condition inspired by the notion of coverage in the offline setting.

\begin{definition}[Coverability coefficient \citep{xie2022role}]
  The coverability coefficient $\Ccov > 0$ for a policy class $\Pi$ is given by
  \iclr{$	\Ccov \coloneqq \inf_{\mu_1,\dots,\mu_H \in \Delta(\cX \times \cA)} \sup_{\pi \in \Pi, h \in [H]} \left\| \nicefrac{d_h^\pi}{\mu_h}\right\|_\infty$.}
\arxiv{\[
	\Ccov \coloneqq \inf_{\mu_1,\dots,\mu_H \in \Delta(\cX \times \cA)} \sup_{\pi \in \Pi, h \in [H]} \left\| \frac{d_h^\pi}{\mu_h}\right\|_\infty.
\] }
We refer to the distribution $\mu_h^\star$ that attains the minimum for $h$ as the \covdist. 
\end{definition}
Coverability is a structural property of the underlying MDP, and can be interpreted as the best value one can achieve for the \emph{concentrability coefficient} $\Cconc(\mu)\ldef\sup_{\pi \in \Pi,h\in\brk{H}}\,
\nrm*{\nicefrac{d_h^\pi}{\mu_h}}_{\infty}$ (a standard coverage parameter in offline RL \citep{munos2007performance,munos2008finite,chen2019information}) by optimally designing the offline data distribution $\mu$. However, in our setting the agent has no prior knowledge of $\mustar$ and no way to explicitly search for it. Examples that admit low coverability include tabular MDPs and \tx{(exogenous)?}Block MDPs \citep{xie2022role}, linear/low-rank MDPs \citep{huang2023reinforcement}, and analytically sparse low-rank MDPs \citep{golowich2023exploring}; see \cref{app:examples} for further examples.\loose

  Concretely, we aim for sample complexity guarantees scaling as $\poly(H,\Ccov,\log\abs{\cF},\log\abs{\cW},\veps^{-1})$, where $\veps$ is the desired bound on the risk in \eqref{eq:risk}.\iclr{\footnote{To simplify presentation as much as possible, we assume finiteness of $\cF$ and $\cW$, but our results extend to infinite classes via standard uniform convergence arguments. Likewise, we do not require exact realizability, and an extension to misspecified classes is given in \cref{app:online}.}} Such a guarantee complements \citet{xie2022role}, who achieved similar sample complexity under the Bellman completeness assumption, and parallels the fashion in which density ratio modeling allows one to remove completeness in offline RL. \arxiv{To simplify presentation as much as possible, we assume finiteness of $\cF$ and $\cW$, but our results extend to infinite classes via standard uniform convergence arguments. Likewise, we do not require exact realizability, and an extension to misspecified classes is given in \cref{app:online}.}

\fakepar{Additional notation} For $n \in \bbN$, we write $[n] = \{1, \dots,
n\}$. For a countable set $\cZ$, we write $\Delta(\cZ)$ for the set of 
probability distributions on $\cZ$. We adopt standard
big-oh notation, and use $\bigoht(\cdot)$ and $\bigomt(\cdot)$ to suppress factors polylogarithmic in $H$, $T$, $\veps^{-1}$, $\log\abs{\cF}$, $\log\abs{\cW}$, and other problem parameters. For each $h\in\brk{H}$, we define $\cF_h=\crl{f_h\mid{}f\in\cF}$ and $\cW_h=\crl{w_h\mid{}w\in\cW}$. \arxiv{For any function \(u: \cX \times \cA \mapsto \bbR\) and distribution \(\rho \in  \Delta(\cX \times \cA)\), we define the norms \(\nrm{u}_{1, \rho} = \En_{(x, a) \sim \rho} \brk*{\abs{u(x, a)}}\) and  \(\nrm{u}_{2, \rho} = \sqrt{\En_{(x, a) \sim \rho} \brk*{u^2(x, a)}}\). }

\section{Online RL with Density Ratio 
  Realizability}
\label{sec:online}
This section presents our main results for the online RL setting. We first introduce our main algorithm, \Alg
(\cref{alg:main_alg}), and explain the intuition behind its design
(\cref{sec:alg_overview}). We then show (\cref{sec:sample_complexity})
that \Alg obtains polynomial sample complexity
guarantees (\cref{thm:online_main_faster,thm:online_main_basic}) under
density ratio realizability and coverability, and conclude with a proof sketch\arxiv{ (\cref{sec:proof_sketch})}. 

\subsection{Algorithm and Key Ideas}
\label{sec:alg_overview}

Our algorithm, \Alg (\cref{alg:main_alg}), is
based on the principle of optimism in the face of uncertainty. For
each iteration $t\leq{}T\in\bbN$, the algorithm uses the density ratio class
$\cW$ to construct a confidence set (or, version space) $\cF\ind{t}\subseteq\cF$ with the
property that $\Qstar\in\cF\ind{t}$. It then chooses a policy $\pi\ind{t}=\pi_{f\ind{t}}$ based on the value function
$f\ind{t}\in\cF\ind{t}$ with the most optimistic estimate
$\En\brk*{f_1(x_1,\pi_{f,1}(x_1))}$ for the initial value. Then, it uses the policy \(\pi\ind{t}\) to gather $K\in\bbN$ trajectories, which are used to update the confidence set for subsequent iterations. 

Within the scheme above, the main novelty to our approach lies in the
confidence set construction. \Alg
appeals to \emph{global optimism} \citep{jiang2017contextual,zanette2020learning,du2021bilinear,jin2021bellman,xie2022role}, and constructs the confidence set
$\cF\ind{t}$ by searching for value functions $f\in\cF$ that satisfy
certain Bellman residual constraints for all
layers $h\in\brk{H}$ simultaneously. For MDPs with low coverability, previous such
approaches \citep{jin2021bellman,xie2022role} make use of constraints
based on \emph{squared Bellman error}, which requires Bellman completeness. The confidence set construction in
\Alg (\eqref{eq:alg1}) departs from this approach, and aims to find $f\in\cF$ such that
the \emph{average Bellman error} is small for all weight functions. At the
population level, this (informally) corresponds to requiring that
for all $h\in\brk{H}$ and $w\in\cW$,\footnote{Average Bellman error 
  \emph{without weight functions} is used in algorithms such at
  \textsc{Olive} \citep{jiang2017contextual} and \textsc{BiLin-UCB}
  \citet{du2021bilinear}. Without weighting, this approach is
  insufficient to derive guarantees based on coverability; see
  discussion in \citet{xie2022role}.}
\begin{align}
  \En_{\dbar\ind{t}}\brk[\big]{w_h(x_h,a_h)\prn*{f_h(x_h,a_h)-\brk*{\cT_hf_{h+1}}(x_h,a_h)}} 
-\alpha\ind{t}\cdot \En_{\dbar\ind{t}}\brk*{\prn{ w_h(x_h,a_h)}^2} \leq
  \beta\ind{t}.%
  \label{eq:constraint_idealized}
\end{align}
where $\dbar_h\ind{t}\ldef{}\frac{1}{t-1}\sum_{i<t}d_h^{\pi\ind{t}}$
is the historical data distribution and $\alpha\ind{t}>0$ and
$\beta\ind{t}>0$ are algorithm parameters; this is motivated by the
fact that the optimal value function satisfies 
\iclr{$\En^{\pi}\brk*{w_h(x_h,a_h)(\Qstar_h(x_h,a_h)-\brk*{\cT_h\Qstar_{h+1}}(x_h,a_h))}=0$}
\arxiv{\[
\En^{\pi}\brk*{w_h(x_h,a_h)(\Qstar_h(x_h,a_h)-\brk*{\cT_h\Qstar_{h+1}}(x_h,a_h))}=0
\]}
for all functions $w$ and policies $\pi$. Our analysis uses that
\eqref{eq:constraint_idealized} holds for the weight function
$w_h\ind{t}\ldef{}\nicefrac{d^{\pi\ind{t}}_h}{\dbar\ind{t}_h}$, which
allows to transfer bounds on the off-policy Bellman error for the historical
distribution $\dbar\ind{t}$ to the on-policy Bellman error for
$\pi\ind{t}$.\loose

\begin{remark}
  Among density ratio-based
algorithms for \emph{offline} reinforcement learning 
\citep{jiang2020minimax,xie2020q,zhan2022offline,chen2022offline,rashidinejad2022optimal},
the constraint \cref{eq:constraint_idealized} is most directly inspired by the Minimax Average Bellman Optimization ({\normalfont \textsc{Mabo}}) algorithm \citep{xie2020q},
which uses a similar minimax approximation to the average Bellman
error.
\end{remark}

\fakepar{Partial coverage and clipping}
Compared to the offline setting, much extra work is required to handle
the issue of \emph{partial coverage}. Early in the learning process, the ratio
$w_h\ind{t}\ldef{}\nicefrac{d^{\pi\ind{t}}_h}{\dbar\ind{t}_h}$ may be
unbounded, which prevents the naive empirical approximation to 
\eqref{eq:constraint_idealized} from concentrating. To address this
issue, \Alg carefully truncates the weight functions under consideration.\loose
\begin{definition}[Clipping operator]
For any $w: \cX \times \cA \rightarrow \bbR \cup \{\infty\}$ and
$\gamma \in \bbR$, we define the \textit{clipped weight function} (at
scale $\gamma$) via \iclr{$\clip{w}{\gamma}(x,a) \coloneqq \min\{w(x,a),\gamma\}$.}
\arxiv{\[
	\clip{w}{\gamma}(x,a) \coloneqq \min\{w(x,a),\gamma\}. 
\]}
\end{definition}
\arxiv{Within \Alg, we}\iclr{We} replace the weight functions in
\eqref{eq:constraint_idealized} with clipped counterparts \iclr{$\clip{w}{\gamma\ind{t}}(x,a)$}\arxiv{given by
$\check{w}(x,a)\ldef{}\clip{w}{\gamma\ind{t}}(x,a)$}, where $\gamma\ind{t} \ldef
\gamma\cdot t$ for a parameter $\gamma \in \brk{0,1}$.\arxiv{\footnote{All of our results
  carry over to the alternate clipping operator $\clip{w}{\gamma}(x,a) = w(x,a) \indic\{w(x,a) \leq
  \gamma\}$.}} For a given iteration $t$, clipping in this fashion may
render \eqref{eq:constraint_idealized} a poor approximation to the
on-policy Bellman error. The crux of our analysis is to show---via
coverability---that \emph{on average} across all iterations, the
approximation error is small.

An important difference relative to \textsc{Mabo} is that the weighted
Bellman error in \eqref{eq:constraint_idealized} incorporates a
quadratic penalty $-\alpha\ind{t}\cdot \En_{\dbar\ind{t}}\brk{\prn{
    w_h(x_h,a_h)}^2} $ for the weight function. This is not essential to
derive polynomial sample complexity guarantees, but is critical to
attain the $\nicefrac{1}{\veps^2}$-type rates we achieve under our
strongest realizability assumption. Briefly, regularization is
beneficial because it allows us to appeal to variance-dependent Bernstein-style
concentration; our analysis shows that while the variance of the
weight functions under consideration may not be small on a per-iteration
basis, it is small on average across all iterations (again, via
coverability). Interestingly, similar quadratic penalties have been
used within empirical offline RL algorithms based on density ratio
modeling \citep{yang2020off,lee2021optidice}, as well as recent theoretical
results \citep{zhan2022offline}, but for considerations
seemingly unrelated to concentration.

\begin{algorithm}[tp]
    \caption{\Alg: Global Optimism via Weight Function Realizability}  
    \label{alg:main_alg} 
\begin{algorithmic}[1] 
\Statex[0] \textbf{input:} Value function class \(\cF\), Weight
function class \(\cW\), Parameters $T,K \in\bbN$,  \(\gamma \in \brk{0,1}\).\\
\algcomment{For \arxiv{\pref{thm:online_main_faster}} \iclr{Thm.~\ref{thm:online_main_faster}},   set    \(T = \wt  \Theta \prn{\prn{\nicefrac{H^2 \Ccov}{\epsilon^2}} \cdot \log\prn{\nicefrac{\abs*{\cF}\abs*{\cW}}{\delta}} }\), \(K = 1\), and \(\gamma = \sqrt{\nicefrac{\Ccov}{(T \log\prn{\nicefrac{\abs*{\cF}\abs*{\cW}}{\delta})}}}\).}  
\State \algcomment{For \arxiv{\pref{thm:online_main_basic}} \iclr{Thm.~\ref{thm:online_main_basic}}, set \(T = \wt  \Theta  \prn{\nicefrac{H^2 \Ccov}{\epsilon^2}}\), \(K = \wt  \Theta \prn{T  \log\prn{\nicefrac{\abs*{\cF}\abs*{\cW}}{\delta}}}\), and \(\gamma = \sqrt{\nicefrac{\Ccov}{T}}\).}  
\State Set \(\gamma\ind{t} = \gamma\cdot t\), \(\alpha\ind{t} = \nicefrac{8}{\gamma\ind{t}}\) and \(\beta\ind{t} = \prn{\nicefrac{36 \gamma\ind{t}}{K(t-1)}} \cdot \log\prn*{\nicefrac{6 \abs{\cF} \abs{\cW} T H }{ \delta}}\). 
\State Initialize \(\cD\ind{1}_h = \emptyset\) for all \(h \leq H\). 
\For{$t=1,\ldots,T$} 
 \State  \multiline{Define confidence set based on (regularized)
   minimax average Bellman error:
  {\small \begin{align*} 
   \hspace{-0.1in} \cF\ind{t}=\crl*{ 
    f\in\cF
    \mid{}\forall{}h: 
    \sup_{w\in\cW_h} \wh \En_{\cD\ind{t}_h} \brk*{
    \prn[\big]{ 
    \whDf(x, a, r, x')} \cdot 
    \wt w_h(x,a) - \alpha\ind{t}\cdot \prn[\big]{ \wt w_h(x,a)}^2} \leq{} \beta\ind{t} 
    }, 
    \numberthis \label{eq:alg1} 
  \end{align*}}where \(\wt w \ldef \clip{w}{\gamma\ind{t}}\) and \(\whDf(x, a, r, x') \ldef{} f_h(x,a)- r - \max_{a'}f_{h+1}(x', a')\).} 
  \State \multiline{Compute optimistic value function and policy:
{\small \begin{align*} 
f\ind{t}\ldef\argmax_{f\in\cF\ind{t}} \Ehat_{x_1\sim\cD_1\ind{t}}\brk*{f_1(x_1,\pi_f(x_1))}, \quad \text{and} \quad \pi\ind{t}\ldef\pi_{f\ind{t}}. \numberthis \label{eq:alg2}  
\end{align*}}
}  
\arxiv{  \State \algcomment{Online data collection.} }
  \State Initialize \(\cD\ind{t+1}_h \leftarrow \cD\ind{t}_h\) for \(h \in [H]\). 
\For {$k = 1, \dots, K$} \iclr{\hfill \algcomment{Online data collection.} }
\State Collect a trajectory  
  $(x_1,a_1,r_1),\ldots, (x_H,a_H,r_H)$ by executing \(\pi \ind{t}\). 
  \State Update \(\cD\ind{t+1}_h \leftarrow \cD\ind{t+1}_h \cup \crl{(x_h, a_h, r_h, x_{h+1})}\) for each \(h \in [H]\). 
  \EndFor
\EndFor 
\State \textbf{output:} policy $\pihat=\unif (\pi\ind{1}, \dots, \pi\ind{T})$.\hfill   \algcomment{For PAC guarantee only.} 
   \end{algorithmic} 
\end{algorithm}

\subsection{Main Result: Sample Complexity Bound for \Alg} \label{sec:sample_complexity} 

We now present the main sample complexity guarantees for \Alg. The
first result we present, which gives
the tightest sample complexity bound, is stated
under a form of density ratio realizability that strengthens
\cref{ass:W_realizability_proper}. Concretely, we assume that the class $\cW$ can realize density ratios
for certain \textit{mixtures} of policies. For $t \in \bbN$, we write
$\pi\ind{1:t}$ as a shorthand for a sequence\arxiv{ of policies} $(\pi\ind{1},
\cdots, \pi\ind{t})$, where $\pi\ind{i} \in \Pi$, and let $d^{\pi\ind{1:t}} \ldef
\frac{1}{t}\sum_{i=1}^t d^{\pi\ind{i}}$.\loose 

  \begin{asmmod}{ass:W_realizability_proper}{$'$}[Density ratio 
    realizability, mixture version]
  \label{ass:W_realizability_mixture}
Let $T$ be the parameter to \Alg (\cref{alg:main_alg}). For all $h \in [H]$, $\pi \in \Pi$, $t \leq T$, and $\pi\ind{1}, \ldots, \pi\ind{t} \in 
\Pi$, we have 
\[
	w_h^{\pi;\pi\ind{1:t}}(x,a) \coloneqq \frac{d^\pi_h(x,a)}{d^{\pi\ind{1:t}}_h(x,a)} \in \cW.
\] 
\end{asmmod}
This assumption directly
facilitates transfer from the algorithm's \emph{historical distribution}
$\dbar_h\ind{t}\ldef{}\frac{1}{t-1}\sum_{i<t}d_h^{\pi\ind{t}}$ to on-policy error. 
\arxiv{Naturally, it is implied
by the stronger-but-simpler-to-state assumption that we can
realize density ratios $\nicefrac{d^{\pi}_h}{d^{\rho}_h}$ for all
$\pi\in\Pi$ and \emph{all mixture policies} $\rho\in\Delta(\Pi)$.}  
Under \cref{ass:W_realizability_mixture}, \arxiv{we show that }\Alg obtains
$\nicefrac{1}{\veps^2}$-PAC sample complexity and $\sqrt{T}$-regret.

\begin{theorem}[Risk bound for \Alg under strong density ratio realizability]
\label{thm:online_main_faster} 
Let \(\epsilon > 0\) be given, and suppose that   \cref{ass:Q_realizability,ass:W_realizability_mixture} hold. Then, \Alg, with hyperparameters 
   \(T = \wt  \Theta \prn[\big]{\prn{\nicefrac{H^2 \Ccov}{\epsilon^2}} \cdot
     \log\prn{\nicefrac{\abs*{\cF}\abs*{\cW}}{\delta}} }\), \(K = 1\),
   and \(\gamma = \sqrt{\nicefrac{\Ccov}{(T
       \log\prn{\nicefrac{\abs*{\cF}\abs*{\cW}}{\delta})}}}\) returns
   an \(\epsilon\)-suboptimal policy \(\wh \pi\) with probability at least \(1 - \delta\) after collecting  
\begin{align} 
  \label{eq:online_main_faster_pac} 
  N =  \wt O\prn[\bigg]{\frac{H^2 \Ccov}{\epsilon^2} \log\prn*{\nicefrac{\abs*{\cF}\abs*{\cW}}{\delta}}}  
\end{align}
trajectories. 
Additionally, for any $T\in\bbN$, with the same choice for $K$ and
$\gamma$ as above, \Alg enjoys the regret bound %
\iclr{$\Reg \ldef{}  \sum_{t=1}^T J(\pistar) - J(\pi \ind{t}) = \wt  O \prn[\big]{H \sqrt{\Ccov T \log\prn*{\nicefrac{\abs*{\cF}\abs*{\cW}}{\delta}}}}.$}
\arxiv{\[\Reg \ldef{}  \sum_{t=1}^T J(\pistar) - J(\pi \ind{t}) = \wt  O \prn[\big]{H \sqrt{\Ccov T \log\prn*{\nicefrac{\abs*{\cF}\abs*{\cW}}{\delta}}}}.\]} 
\end{theorem}
Next, we provide our main result, which gives a sample complexity
guarantee under density ratio realizability for pure policies
(\cref{ass:W_realizability_proper}). To obtain the result, we begin
with a class $\cW$ that satisfies \cref{ass:W_realizability_proper},
then expand it to obtain an augmented class $\wb{\cW}$ that satisfies mixture realizability (\cref{ass:W_realizability_mixture}). This reduction increases
$\log\abs*{\cW}$ by a $T$ factor, which we offset by increasing the
batch size $K$; this leads to a polynomial increase in sample complexity.\footnote{This reduction
  also prevents us from obtaining a regret bound directly under
  \cref{ass:W_realizability_proper}, though a (slower-than-$\sqrt{T}$)
  regret bound can be attained using an explore-then-commit strategy.} 
\begin{theorem}[Risk bound for \Alg under weak density ratio realizability]
\label{thm:online_main_basic}	
Let \(\epsilon > 0\) be given, and suppose that
\cref{ass:Q_realizability,ass:W_realizability_proper} hold for the classes $\cF$ and $\cW$. Then, \Alg, when executed with a modified class
\(\cWbar\) defined in \eqref{eq:Wbar_def} in \cref{app:online},  with
hyperparameters  \(T = \wt  \Theta \prn{\nicefrac{H^2
    \Ccov}{\epsilon^2}}\), \(K = \wt  \Theta\prn{T
  \log\prn{\nicefrac{\abs*{\cF}\abs*{\cW}}{\delta}}}\), and \(\gamma =
\sqrt{\nicefrac{\Ccov}{T}}\), returns an \(\epsilon\)-suboptimal
policy \(\wh \pi\) with probability at least \(1 - \delta\) after
collecting $N$ trajectories, for  
\begin{align}
  \label{eq:online_main_slower_pac}
N = \wt  O\prn[\bigg]{\frac{H^4 \Ccov^2}{\epsilon^4}  \log\prn*{\nicefrac{\abs*{\cF}\abs*{\cW}}{\delta}}}.
\end{align} \end{theorem}
\cref{thm:online_main_faster,thm:online_main_basic} show for the first
time that value function realizability and density ratio realizability
alone are sufficient for sample-efficient online RL under
coverability. In particular, the sample complexity and regret bound in
\cref{thm:online_main_faster} 
match the coverability-based guarantees obtained in \citet[Theorem
1]{xie2022role} under the complementary Bellman
completeness assumption, with the only difference being
that they scale with $\log(\abs{\cF}\abs{\cW})$ instead of
$\log\abs{\cF}$; as discussed in \citet{xie2022role}, this rate is
tight for the special case of contextual bandits ($H=1$). Interesting
open questions include (i) whether the sample complexity for learning with density ratio
realizability for pure policies can be improved to
$\nicefrac{1}{\veps^2}$, and (ii) whether value realizability and
coverability alone are
sufficient for sample-efficient RL.
\iclr{Extensions to \cref{thm:online_main_faster,thm:online_main_basic}
under misspecification are given in
  \cref{app:online}; see also \cref{rem:clipped_realizability}}
\arxiv{
Extensions to \cref{thm:online_main_faster,thm:online_main_basic}
under misspecification are given in
  \cref{app:online}.
 }
  We further refer to \cref{app:examples} for examples instantiating 
these results\arxiv{.}\iclr{,  \cref{sec:proof_sketch} for a proof sketch, and \cref{rem:connection-golf} for a connection of our algorithm to the \textsc{Golf} algorithm \citep{jin2021bellman,xie2022role}.} \arxiv{In particular, our results establish a positive
  result for \dfedit{a generalized class of }Block MDPs with coverable
  latent spaces, while only requiring (for the first time) function
  approximation conditions that concern the latent space (\cref{ex:bmdp}).}

\arxiv{
Like other algorithms based on global optimism
\citep{jiang2017contextual,zanette2020learning,du2021bilinear,jin2021bellman,xie2022role},
\Alg is not computationally efficient. As a step toward developing
practical online RL algorithms based on density ratio modeling, we
give a \dfedit{more} efficient counterpart for the hybrid RL model in the \cref{sec:hybrid}.
}

\arxiv{
  \begin{remark}[Connection to \textsc{Golf}]
    Prior work \citep{xie2022role} analyzed the {\normalfont \textsc{Golf}}
    algorithm of \citeauthor{jin2021bellman} and established positive
    results under coverability and Bellman completeness. We remark
    that by allowing for weight functions that take negative
    values,\footnote{In this context, $\cW$ can be thought of more
      generally as a class of \emph{test functions}.}
   \Alg can be viewed as a generalization of {\normalfont\textsc{Golf}}, and can
    be configured to obtain comparable results. Indeed, given a value
    function class $\cF$ that satisfies Bellman completeness, the
    weight function class  $\cW \ldef{} \{f - f' \mid f, f' \in \cF\}$
    leads to a confidence set construction at least as tight as that
    of {\normalfont \textsc{Golf}}. To see this, observe that if we set $\gamma \geq
    2$ so that no clipping occurs, our construction for
    $\cF\ind{t}$ (\eqref{eq:alg1}) implies (after standard
    concentration arguments) that $Q^\star \in \cF\ind{t}$ and that
    in-sample squared Bellman errors are small with high
    probability. These ingredients are all that is required to repeat
    the analysis of {\normalfont\textsc{Golf}} from \cite{xie2022role}.
  \end{remark}
	\pacomment{added this ``glow subsumes golf''
          paragraph. thoughts? too technical? not technical enough?}
        \dfcomment{seems fine, i am curious though about 1) do we get
          the fast rate?, and 2) how to set $\alpha$ and $\beta$. I
          think we may actually get the fast rate using the
          regularization which is cool}
}

\iclr{

  }

\arxiv{
\subsection{Proof Sketch}
\label{sec:proof_sketch}

\arxiv{We now}\iclr{In this section we} give a proof sketch for \cref{thm:online_main_faster}, highlighting the role of truncated weight functions in addressing partial coverage. We focus on the regret bound; the sample complexity bound in \eqref{eq:online_main_faster_pac} is an immediate consequence.%

By design, the constraint in \eqref{eq:alg1} ensures that \(Q^\star \in \cF\ind{t}\) for all \(t \leq T\) with high probability. Thus, by a standard regret decomposition for optimistic algorithms (\pref{lem:pdl}\arxiv{ in the appendix}), we have \loose
\begin{align*} 
\Reg  &= \sum_{t=1}^T J(\pistar) - J(\pi\ind{t})  %
\lesssim  \sum_{t=1}^T  \sum_{h=1}^H \underbrace{ \En_{d\ind{t}_h} \brk*{f\ind{t}_h(x_h, a_h) - [\cT f\ind{t}_{h+1}](x_h, a_h)} }_{\text{On-policy Bellman error for \(f\ind{t}\) under $\pi\ind{t}$}}, %
\numberthis \label{eq:sketch0} 
\end{align*}
up to lower-order terms, where we abbreviate $d\ind{t}=d^{\pi\ind{t}}$.
Defining $\Dft (x, a) \ldef{}  f\ind{t}_h(x,a)-\brk{\cT_hf\ind{t}_{h+1}}(x,a)$, it remains to bound the on-policy expected bellman error \(\En_{d\ind{t}_h} \brk*{\Dft (x_h, a_h)}\). To do so, a natural approach is to relate this quantity to the weighted off-policy Bellman error under $\dbar\ind{t}\ldef{}\frac{1}{t-1}\sum_{i<t}d^{\pi\ind{i}}$ by introducing the weight function $\nicefrac{d\ind{t}}{\dbar\ind{t}}\in\cW$:
\[
  \En_{{d\ind{t}_h}}\brk*{\Dft(x_h, a_h)} \approx \En_{\dbar\ind{t}_h}\brk*{{\Dft(x_h, a_h)} \cdot \frac{{d\ind{t}_h}(x_h,a_h)}{\dbar\ind{t}_h(x_h,a_h)}}.
\]
Unfortunately, this equality is not true as-is because the ratio $\nicefrac{d\ind{t}}{\dbar\ind{t}}$ can be unbounded. We address this by replacing $\dbar\ind{t}$ by $\dbar\ind{t+1}$ throughout the analysis (at the cost of small approximation error), and work with the weight function \(w_h\ind{t} \coloneqq \nicefrac{{d\ind{t}_h}}{\dbar\ind{t+1}_h}\in\cW\), which is always bounded in magnitude $t$. However, while boundedness is a desirable property, the range $t$ is still too large to obtain non-vacuous concentration guarantees. This motivates us to introduce clipped/truncated weight functions via the following decomposition. 
\begin{align*} 
\underbrace{\En_{{d\ind{t}_h}}\brk*{\Dft(x_h, a_h)}}_{\text{On-policy Bellman error}}
&\leq \underbrace{\En_{\dbar\ind{t+1}_h} \brk*{\Dft(x_h, a_h) \cdot \clip{w_h\ind{t}}{\gamma\ind{t}}(x_h, a_h)}}_{\text{($A_t$): Clipped off-policy Bellman error}}  +  \underbrace{\En_{d\ind{t}_h}\brk*{\indic\crl*{w\ind{t}_h(x_h, a_h) \geq \gamma\ind{t}} }}_{\text{($B_t$): Loss due to clipping}}.  %
\end{align*}

Recall that $\check{w}_h\ind{t}\ldef{}\clip{w_h\ind{t}}{\gamma\ind{t}}(x_h, a_h)$. As $w\ind{t}\in\cW$, it follows from the constraint in \eqref{eq:alg1} and Freedman-type concentration that the clipped Bellman error in term ($A_t$) has order $\alpha\ind{t}\cdot\En_{\dbar\ind{t+1}}\brk*{(\check{w}\ind{t}_h)^2}+\beta\ind{t}$, so that $\sum_{t=1}^{T}A_t\leq{}
\sum_{t=1}^{T}\alpha\ind{t}\cdot\En_{\dbar\ind{t+1}}\brk*{(\check{w}\ind{t}_h)^2}+\beta\ind{t}$. Since we clip to  $\gamma\ind{t}=\gamma{}t$, we have $\sum_{t=1}^{T}\beta\ind{t}\approxleq{}\gamma{}\cdot{}T\log(\abs{\cF}\abs{\cW}/\delta)$; bounding the sum of weight functions requires a more involved argument that we defer for a moment.\loose

We now focus on bounding the terms ($B_t$)\padelete{, which we find to be most illuminating}. Each term ($B_t$) captures the extent to which the weighted off-policy Bellman error at iteration $t$ fails to approximate the true Bellman error due to clipping. This occurs when $\dbar\ind{t+1}$ has poor coverage relative to $d\ind{t}$, which happens when $\pi\ind{t}$ visits a portion of the state space not previously covered. We begin by applying Markov's inequality (\(\indic\crl{u \geq v} \leq \nicefrac{u}{v}\) for  \(u, v \geq 0\)) to bound
\begin{align}
  \label{eq:sketchx}
  B_t \leq  \frac{1}{\gamma\ind{t}} \En_{{d\ind{t}_h}}\brk*{w_h\ind{t}(x_h, a_h)} = \frac{1}{\gamma\ind{t}} \En_{{d\ind{t}_h}}\brk*{\frac{{d\ind{t}_h}(x_h,a_h)}{\dbar\ind{t+1}_h(x_h,a_h)}}
  = \frac{1}{\gamma} \En_{{d\ind{t}_h}}\brk*{\frac{{d\ind{t}_h}(x_h,a_h)}{\dtil\ind{t+1}_h(x_h,a_h)}},
\end{align}
where the equality uses that $\gamma\ind{t}\ldef{}\gamma{}\cdot{}t$ and $\dtil\ind{t+1}\ldef\dbar\ind{t+1}\cdot{}t$. Our most important insight is that even though each term in \eqref{eq:sketchx} might be large on a given iteration $t$ (if a previously unexplored portion of the state space is visited), \emph{coverability} implies that on average across all iterations the error incurred by clipping must be small. In particular, using a variant of a coverability-based potential argument from \citet{xie2022role} (\pref{lem:elliptical-potential}), we show that
\begin{align*}
  \sum_{t=1}^{T}\En_{{d\ind{t}_h}}\brk*{\frac{{d\ind{t}_h}(x_h,a_h)}{\dtil\ind{t+1}_h(x_h,a_h)}}
  \leq \bigoh\prn*{\Ccov\cdot\log(T)
  },
\end{align*}
so that $\sum_{t=1}^{T}B_t\leq\bigoht(\Ccov/\gamma)$. To conclude the proof, we use an analogous potential argument to show the sum of weight functions in our bound on $\sum_{t=1}^{T}A_t$ also satisfies $\sum_{t=1}^{T}\alpha\ind{t}\cdot\En_{\dbar\ind{t+1}}\brk*{(\check{w}\ind{t}_h)^2}\leq\bigoht(\Ccov/\gamma)$. The intuition is similar: the squared weight functions (corresponding to variance of the weighted Bellman error) may be large in a given round, but cannot be large for all rounds under coverability. Altogether, combining the bounds on $A_t$ and $B_t$ gives \iclr{$\Reg = \wt O\prn[\big]{ H \prn[\big]{  \frac{\Ccov}{\gamma} + \gamma{}\cdot T \log(\abs{\cF} \abs{\cW} HT \delta^{-1})}}$.}
\arxiv{
\begin{align*}\textstyle
\Reg  &= \wt O\prn[\bigg]{ H \prn[\bigg]{  \frac{\Ccov}{\gamma} + \gamma\cdot T \log(\abs{\cF} \abs{\cW} HT \delta^{-1})}}. \numberthis \label{eq:sketch4}  
\end{align*}}
The final result follows by choosing $\gamma>0$ to balance the terms.

We find it interesting that the way in which this proof makes use of coverability---to handle the cumulative loss incurred by clipping---is quite different from the analysis in \citet{xie2022role}, where it more directly facilitates a change-of-measure argument.

}

\section{Efficient Hybrid RL with Density Ratio Realizability} 
\label{sec:hybrid}

\arxiv{Our results in the prequel show that density ratio realizability and coverability suffice for sample-efficient online RL. However, like other algorithms for sample-efficient exploration under general function approximation \citep{jiang2017contextual,du2021bilinear,jin2021bellman}, \Alg is not \textit{computationally} efficient.}
Toward overcoming the challenges of intractable computation 
in online exploration, a number of recent works show that including additional offline data in online RL can lead to computational benefits in theory \citep[e.g.,][]{xie2021policy,wagenmaker2023leveraging,song2022hybrid, zhou2023offline} and in practice \citep[e.g.,][]{cabi2020scaling,nair2020awac,ball2023efficient, song2022hybrid, zhou2023offline}. Notably, combining offline and online data can enable algorithms that provably explore without having to appeal to optimism or pessimism, both of which are difficult to implement efficiently under general function approximation.

\citet{song2022hybrid}~formalize a version of this setting%
---in which online RL is augmented with offline data---as \emph{hybrid reinforcement learning}. Formally, in hybrid RL, the learner interacts with the MDP online (as in \cref{sec:setting}) but is additionally given an offline dataset $\cD_{\mathrm{off}}$ %
	collected from a data distribution $\nu$. The data distribution $\nu$ is typically assumed to provide coverage for the optimal policy $\pistar$ (formalized in \cref{ass:pi-star-concentrability}), but not on all policies, and thus additional online exploration is required (see \cref{sec:comparison-offline-rl} for further discussion).\loose

\subsection[H2O: A Provable Black-Box Hybrid-to-Offline Reduction]{$\HtO$: A Provable Black-Box Hybrid-to-Offline Reduction}\label{sec:hybridreductions} 
Interestingly, many of the above approaches for the hybrid setting simply apply offline algorithms (with relatively little modification) on a mixture of online and offline data \citep[e.g.,][]{cabi2020scaling,nair2020awac,ball2023efficient}. This raises the question: \textit{when can we use a given offline algorithm as a black box to solve the problem of hybrid RL (or, more generally, of online RL?)}. To answer this, we give a general meta-algorithm, $\HtO$, which provides a provable black-box reduction to solve the hybrid RL problem by repeatedly invoking a given offline RL algorithm on a mixture of offline data and freshly gathered online trajectories.\arxiv{ \cutedit{We instantiate the meta-algorithm using a simplified offline counterpart to \Alg as a black box to obtain \HyAlg, a density ratio-based algorithm for the hybrid RL setting that improves upon the computational efficiency of \Alg
(\cref{sec:applying-reduction}).}} To present the result, we first describe the class of offline RL algorithms with which it will be applied.

\fakepar{Offline RL and partial coverage}
We refer to a collection of distributions $\mu = \{ \mu_h\}_{h=1}^H$, where $\mu_h \in \Delta(\cX \times \cA)$, as a \textit{data distribution}, and we say that a dataset $\cD = \{\cD_h\}_{h=1}^H$ has \textit{$H\cdot{}n$ samples from data distributions $\mu\ind{1}, \dots, \mu\ind{n}$} if $\cD_h = \{ (x\ind{i}_h,a\ind{i}_h,r\ind{i}_h,x\ind{i}_{h+1})\}_{i=1}^n$ where %
$(x\ind{i}_h,a\ind{i}_h)\sim\mu\ind{i}_h$, $r\ind{i}_h\sim{}R_h(x_h\ind{i},a_h\ind{i})$, $x_{h+1}\ind{i}\sim{}P_h(x_h\ind{i},a_h\ind{i})$. We denote the \textit{mixture distribution} via $\mu\ind{1:n} = \{\mu\ind{1:n}_h\}_{h=1}^H$, where $\mu\ind{1:n}_h \coloneqq \frac{1}{n} \sum_{i=1}^n \mu\ind{i}_h$.\iclr{ An \textit{offline RL algorithm} $\algo$ takes as input a dataset $\cD = \{\cD_h\}_{h=1}^H$ of $H\cdot{}n$ samples from $\mu\ind{1},\ldots,\mu\ind{n}$ and outputs a policy $\pi = (\pi_h)_{h=1}^H$.\footnote{When parameters are needed, $\algo$ should instead be thought of as the algorithm for a fixed choice of parameter. Likewise, we treat function approximators ($\cF$, $\cW$, etc.) as part of the algorithm. We consider adaptively chosen $(\mu\ind{i})_{i=1}^n$ because, in the reduction, $\algo$ will be invoked on history-dependent datasets.}} \iclr{An offline algorithm is measured by its risk, as in Equation \cref{eq:risk}; we write $\RiskOff$ when we are in the offline interaction protocol.}

\arxiv{\begin{definition}[Offline RL algorithm]
An \textit{offline RL} algorithm $\algo$ takes as input a dataset $\cD = \{\cD_h\}_{h=1}^H$ of $H\cdot{}n$ samples from $\mu\ind{1}, \ldots \mu\ind{n}$ and outputs a policy $\pi = \{\pi_h\}_{h=1}^H$.\footnote{ $\algo$ does not have any parameters; when parameters are needed, $\algo$ should instead be thought of as the algorithm for a fixed choice of parameter. Likewise, we treat $\cF$ and $\cW$ (or other input function classes) as part of the algorithm.} We allow $\mu\ind{1}, \dots \mu\ind{n}$ to be adaptively chosen, i.e. each $\mu\ind{i+1}$ may be a function of the samples generated from $\mu\ind{1}\dots \mu\ind{i}$.\footnote{Guarantees for offline RL in the i.i.d. setting can often be extended to the adaptive setting (\cref{sec:applying-reduction}).} \end{definition}
}
An immediate problem with directly invoking offline RL algorithms in the hybrid model is that typical algorithms---particularly, those that do not make use of pessimism \citep[e.g.,][]{xie2020q}---require relatively \emph{uniform} notions of coverage (e.g., coverage for all policies as opposed to just coverage for $\pistar$) to provide guarantees, leading one to worry that their behaviour might be completely uncontrolled when applied with non-exploratory datasets. Fortunately, we will show that for a large class algorithms whose risk scales with a measure of coverage we refer to as \emph{clipped concentrability}, this phenomenon cannot occur. \arxiv{\cutedit{Below, for any distribution $\rho \in \Delta(\cX \times \cA)$, we write $\nrm{\cdot}_{1,\rho}$ and $\nrm{\cdot}_{2,\rho}$ for the $L_1(\rho)$ and $L_2(\rho)$ norms.}}

\begin{definition}[Clipped concentrability coefficient] 
  The \textit{clipped concentrability coefficient} (at scale $\gamma \in \bbR_+$) for $\pi \in \Pi$ relative to a data distribution $\mu = \{\mu_h\}_{h=1}^H$, where $\mu_h \in \Delta(\cX \times \cA)$, is defined as\loose 
\[
\Ccc{n}{h}{\pi}{\mu}{\gamma} \coloneqq \left\|\clip{\frac{d^{\pi}_h}{\mu_h}}{\gamma}\right\|_{1,d^\pi_h}.
\]
\end{definition}
This coefficient should be thought of as a generalization of the standard (squared) $L_2(\mu)$ concentrability coefficient $
\Cconct^2(\pi, \mu) \coloneqq \left\| \nicefrac{d^\pi_h}{\mu_h}\right\|^2_{2,\mu_h}=\left\| \nicefrac{d^\pi_h}{\mu_h}\right\|_{1,d^{\pi}_h}$, a fundamental object in the analysis of offline RL algorithms \citep[e.g.,][]{farahmand2010error}, but incorporates clipping to better handle partial coverage. 
We consider offline RL algorithms with the property that for any offline distribution $\mu$, the algorithm's risk can be bounded by the clipped concentrability coefficients for (i) the output policy $\pihat$, and (ii) the optimal policy $\pi^\star$. For the following definition, we recall the notation $\gamma\ind{n} \coloneqq \gamma \cdot n$. 
\loose
\begin{definition}[\Ccbounded offline RL algorithm]
  \label{ass:offline-risk-cc}
\iclr{An }\arxiv{We say that an }offline algorithm $\algo$ is \Ccbounded at scale $\gamma \in \bbR_+$ under an assumption $\Cond(\cdot)$ if there exists scalars $\ascale, \bscale$ such that for all $n \in \bbN$ and data distributions $\mu\ind{1},\ldots,\mu\ind{n}$, $\algo$ outputs \arxiv{a distribution }$p \in \Delta(\Pi)$ satisfying
\iclr{
\begin{align}
\RiskOff = \En_{\wh \pi \sim p}  \brk*{J(\pistar) -  J(\pihat)} \leq \sum_{h=1}^H \frac{\ascale}{n} \left( \Ccc{n}{h}{\pi^\star}{\mu\ind{1:n}}{\gamma\ind{n}} + \En_{\wh \pi \sim p} \brk*{\Ccc{n}{h}{\wh{\pi}}{\mu\ind{1:n}}{\gamma\ind{n}}} \right)+ \bscale\label{eq:offline_risk}
\end{align}
}
\arxiv{
\begin{align}
\RiskOff =\En_{\wh \pi \sim p}  \brk*{J(\pistar) -  J(\pihat)} \leq \sum_{h=1}^H \frac{\ascale}{n} \left( \Ccc{n}{h}{\pi^\star}{\mu\ind{1:n}}{\gamma\ind{n}} + \En_{\wh \pi \sim p} \brk*{\Ccc{n}{h}{\wh{\pi}}{\mu\ind{1:n}}{\gamma\ind{n}}} \right)+ \bscale\label{eq:offline_risk}
\end{align}
}
with probability at least $1-\delta$, when given \arxiv{a dataset of }$H\cdot{}n$ samples from $\mu\ind{1},\ldots,\mu\ind{n}$ such that $\Cond(\mu\ind{1:n},M^\star)$ is satisfied.%
\end{definition}

This definition does not automatically imply that the offline algorithm has low offline risk, but simply that the risk can be bounded in terms of clipped coverage (which may be large if the dataset has poor coverage). In the sequel, we will show that many natural offline RL algorithms %
    have this property (\cref{app:hybrid_examples}).\iclr{\footnote{Examples of assumptions for $\Cond$ include value function completeness (e.g., for \textsc{Fqi} \citep{chen2019information}) and realizability of value functions and density ratios (e.g., for \textsc{Mabo} \citep{xie2020q}).}} \arxiv{Examples of assumptions for $\Cond$ include value function completeness (e.g., for \textsc{Fqi} \citep{chen2019information}) and realizability of value functions and density ratios (e.g., for \textsc{Mabo} \citep{xie2020q}).}

\arxiv{Offline RL algorithms based on pessimism \citep{jin2021pessimism,rashidinejad2021bridging,xie2021bellman} typically enjoy risk bounds that only require coverage for $\pistar$. Crucially, by allowing the risk bound in \cref{ass:offline-risk-cc} to scale with coverage for $\pihat$ in addition to $\pistar$, we can accommodate non-pessimistic offline RL algorithms such as \textsc{Fqi} and \textsc{Mabo} that are weaker statistically, yet more computationally efficient.\loose} 

\arxiv{
\begin{remark}
  While the bound in \eqref{eq:offline_risk} might seem to suggest a $\nicefrac{1}{n}$-type rate, it will typically lead to a $\nicefrac{1}{\sqrt{n}}$-type rate after choosing $\gamma>0$ to balance the $\frac{\ascale}{n}$ and $\bscale$ terms.
\end{remark}
}

\fakepar{The \HtO algorithm}
Our reduction, \HtO, is given in \cref{alg:offline-to-hybrid}. For any dataset $\cD$, we will write $\restr{\cD}{1:t}$ for the subset consisting of its first $t$ elements. The algorithm is initialized with an offline dataset $\cD_{\mathrm{off}}=\{\cD_{\mathrm{off},h}\}_{h=1}^H$, and at each iteration $t\in\brk{T}$ invokes the black-box offline RL algorithm $\algo$ with a dataset $\cD_{\mathrm{hybrid}}= \{\cD_{\mathrm{hybrid},h}\}_{h=1}^H$ that mixes the first $t$ elements of $\cD_{\mathrm{off},h}$ %
with all of the online data gathered so far. This produces a policy $\pi\ind{t}$, which is executed to gather trajectories that are then added to the online dataset and used at the next iteration.\arxiv{

}
\HtO is inspired by empirical methods for the hybrid setting \citep[e.g.,][]{cabi2020scaling,nair2020awac,ball2023efficient}\iclr{, and in particular the meta-algorithm is efficient whenever $\algo$ is.\loose}\arxiv{. The total computational cost is simply that of running the base algorithm $\algo$ for $T$ rounds, and in particular the meta-algorithm is efficient whenever $\algo$ is.\loose}

\begin{algorithm}[h] 
    \caption{\HtO: Hybrid-to-Offline Reduction} 
    \label{alg:offline-to-hybrid} 
\begin{algorithmic}[1] 
\Statex[0] \textbf{input:} Parameter $T\in\bbN$, offline algorithm $\algo$, offline datasets $\cD_{\mathrm{off}} = \{\cD_{\mathrm{off},h}\}_h$ each of size $T$. %
\State Initialize \(\cD\ind{1}_{\mathrm{on},h} = \cD\ind{1}_{\mathrm{hybrid},h} = \emptyset\) for all \(h \in [H]\). 
\For{$t=1,\ldots,T$}
\State Get policy $\pi\ind{t}$ from $\algo$ on dataset $\cD\ind{t}_{\mathrm{hybrid}} = \{\cD\ind{t}_{\mathrm{hybrid},h}\}_h$.
\State \mbox{Collect \arxiv{trajectory }$(x_1,a_1,r_1),\ldots, (x_H,a_H,r_H)$ using \(\pi \ind{t}\); \(\cD\ind{t+1}_{\mathrm{on},h} \coloneq \cD\ind{t}_{\mathrm{on},h} \cup \crl{(x_h, a_h, r_h, x_{h+1})}\).}
\State Aggregate offline and online data: $\cD\ind{t+1}_{\mathrm{hybrid},h}\coloneqq \restr{\cD_{\mathrm{off},h}}{1:t} \cup \cD\ind{t+1}_{\mathrm{on},h}$ for all $h \in [H]$. 
\label{line:data-agg}

\EndFor 
\State \textbf{output:} policy $\pihat=\unif (\pi\ind{1}, \dots, \pi\ind{T})$.\hfill   %
   \end{algorithmic} 
\end{algorithm}

\fakepar{Main risk bound for \HtO}\label{sec:hto-regret}
We now present the main result for this section: a risk bound for the \HtO reduction. Our bound depends on the coverability parameter for the underlying MDP, as well as the quality of the offline data distribution $\nu$, quantified by \emph{single-policy concentrability}.%
\begin{definition}[Single-policy concentrability]\label{ass:pi-star-concentrability}
A data distribution $\nu = \{\nu_h\}_{h=1}^H$ satisfies $\Cstar$-single-policy concentrability if \iclr{$\max_h \nrm[\big]{ \nicefrac{d^{\pi^\star}_h}{\nu_h} }_\infty \leq \Cstar$.}
\arxiv{
\[ 
\max_h \left\| \frac{d^{\pi^\star}_h}{\nu_h} \right\|_\infty \leq \Cstar .
\] 
}
\end{definition}

\begin{restatable}[Risk bound for \HtO]{theorem}{htoregret}\label{thm:htoregret}
  Let $T\in\bbN$ be given, let $\cD_{\mathrm{off}}$ consist of $H \cdot T$ samples from data distribution $\nu$, and suppose that $\nu$ satisfies $\Cstar$-single-policy concentrability. Let $\algo$ be \Ccbounded at scale $\gamma \in (0,1)$ under $\Cond(\cdot)$, with parameters $\ascale$ and $\bscale$. Suppose that \arxiv{for all} $\iclr{\forall{}\;}t \in [T]$ and $\pi\ind{1},\ldots,\pi\ind{t} \in \Pi$, $\Cond(\mu\ind{t},M^\star)$ holds for $\mu\ind{t} \ldef \{ \nicefrac{1}{2}(\nu_h + \nicefrac{1}{t}\textstyle\sum_{i=1}^{t}d^{\pi\ind{i}}_h)\}_{h=1}^H$.  
Then, with probability at least $1-\delta{}T$, the risk of \HtO (\cref{alg:offline-to-hybrid}) with \arxiv{inputs }$T$, $\algo$, and $\cD_\mathrm{off}$ is bounded as\iclr{\footnote{We define risk for the hybrid setting as in \eqref{eq:risk}. Our result is stated as a bound on the risk to the optimal policy $\pistar$, but extends to give a bound on the risk of any comparator $\pi$ with $\Cstar$ replaced by coverage for $\pi$.}}
\begin{equation}\label{eq:htoregret}
 \Risk \leq \widetilde{O}\left(H\left(\frac{\ascale(\Cstar+\Ccov)}{T} + \bscale \right) \right).
  \end{equation}

\end{restatable}
For the algorithms we consider, one can take $\ascale \propto \nicefrac{\mathfrak{a}}{\gamma}$ and $\bscale \propto \mathfrak{b} \gamma$ for scalar-valued problem parameters $\mathfrak{a},\mathfrak{b}>0$, so that choosing $\gamma$ optimally gives $\Risk \leq \widetilde{O}\prn[\big]{H \sqrt{\nicefrac{(\Cstar + \Ccov)\fa\fb}{T}}}$ \arxiv{\cutedit{and a sample complexity of $\widetilde{O}\prn[\big]{\nicefrac{H^2(\Cstar + \Ccov)\fa\fb}{\veps^2}}$ to find an $\veps$-optimal policy (see \cref{cor:alphabetagamma})}}.\iclr{ This is a special case of a more general result, \cref{thm:htoregret-general}, which handles the general case where $\nu$ need not satisfy single-policy concentrability.}\arxiv{\footnote{Note that \HtO does not depend on the clipping scale $\gamma$, and thus if the algorithm is \Ccbounded for multiple scales simultaneously we can in fact take the minimum of the right-hand-side over all such $\gamma$.}}

\iclr{
}

\arxiv{The basic idea behind the proof of \cref{thm:htoregret} is as follows: using a standard regret decomposition based on average Bellman error, we can bound the risk of \HtO by the average of the two clipped concentrability terms in \cref{eq:offline_risk} across all iterations. Coverage for $\pistar$ is automatically handled by \cref{ass:pi-star-concentrability}, and we use a potential argument similar to the online setting (cf. \cref{sec:proof_sketch}) to show that the $\pihat$-coverage terms can be controlled by coverability. This is similar in spirit to the analysis of \citet{song2022hybrid}, with coverability taking the place of bilinear rank \citep{du2021bilinear}. \loose}

\arxiv{Our result is stated as a bound on the risk to the optimal policy $\pistar$, but extends to give a bound on the risk of any comparator $\pi^c$ with \cref{ass:offline-risk-cc} and \cref{ass:pi-star-concentrability} replaced by coverage for $\pi^c$.}
\arxiv{This is a special case of a more general result, \cref{thm:htoregret-general}, which handles the general case where $\nu$ need not satisfy single-policy concentrability.}

\subsection{Applying the Reduction: \HyAlg}\label{sec:applying-reduction}

We now apply \HtO to give a hybrid counterpart to \Alg (\cref{alg:main_alg}), using a variant of \textsc{Mabo} \citep{xie2020q} as the black box offline RL algorithm $\algo$ in \HtO. Further examples, which apply Fitted Q-Iteration (\Fqi) and Model-Based Maximum Likelihood Estimation as the black box, are deferred to \cref{app:hybrid_examples}.

\arxiv{As discussed in \cref{sec:online}, the construction for the confidence set of \Alg (\eqref{eq:alg1}) bears some resemblance to the \textsc{Mabo} algorithm \citep{xie2020q} from offline RL, save for the important additions of clipping and regularization. }\iclr{More specifically,}\arxiv{For our main example,} the offline RL algorithm we consider is a variant of \textsc{Mabo} that incorporates clipping and regularization in the same fashion \iclr{as \Alg}, which we call \Crmabo.  Our algorithm takes as input a dataset $\cD = \{\cD_h\}$ with $H\cdot{}n$ samples, has parameters consisting of a value function class $\cF$, a weight function class $\cW$, and a clipping scale $\gamma$, and computes the following estimator:
\begin{equation}\label{eq:clip-mabo}
	\hspace{-0.09in} {\small \wh{f} \in \argmin_{f \in \cF} \max_{w \in \cW} \sum_{h=1}^H \left| \wh{\bbE}_{\cD_h} \left[\check{w}_h(x_h,a_h)\whDf(x_h, a_h, r_h, x'_{h+1})\right]\right| - \alpha\ind{n} \wh{\bbE}_{\cD_h}\left[\check{w}^2_h(x_h,a_h) \right]},
      \end{equation}
  where $\alpha\ind{n}\ldef{}\nicefrac{8}{\gamma\ind{n}}$ and $\check{w}_h \ldef{}\clip{w_h}{\gamma\ind{n}}$.
 \iclr{By instantiating \HtO with this algorithm, we obtain a density ratio-based algorithm for the hybrid RL setting (\HyAlg; \cref{alg:hyglow}) that is statistically and computationally efficient.}
 \arxiv{We will show that this algorithm is \Ccbounded under $Q^\star$-realizability and a density ratio realizability assumption.} 

\arxiv{
\begin{restatable}[\Crmabo is \Ccbounded]{theorem}{mabooffline}\label{thm:mabooffline}
Let $\cD=\crl{\cD_h}_{h=1}^{H}$ consist of $H\cdot{}n$ samples from $\mu\ind{1}, \ldots, \mu\ind{n}$. %
	For any $\gamma \in \bbR_+$, the {\normalfont \Crmabo} algorithm (\eqref{eq:clip-mabo}) with parameters $\cF$, augmented class $\widebar{\cW}$ defined in \eqref{eq:augmented-w} in \cref{app:hyglow}, and $\gamma$ is \Ccbounded at scale $\gamma$ under the $\Cond$ that $Q^\star \in \cF$ and that for all $\pi \in \Pi$ and $h \in [H]$, $\nicefrac{d^{\pi}_h}{\mu\ind{1:n}_h} \in \cW$.  

\end{restatable}

We will show that this algorithm is \Ccbounded under $Q^\star$-realizability and a density ratio realizability assumption. We remark that, following the same arguments in \cref{rem:clipped_realizability}, it suffices to instead only realize the \textit{clipped} density ratios for the optimal scale $\gamma$. By instantiating \HtO with this algorithm, we obtain a density ratio-based algorithm for the hybrid RL setting that is statistically efficient and improves the computational efficiency of \Alg by removing the need for optimism. We call the end-to-end hybrid algorithm \HyAlg, and the full pseudocode can be found in \cref{alg:hyglow}. 

\begin{restatable}[\HyAlg Risk bound]{corollary}{hyglow}\label{cor:hyglow}
	Let $\varepsilon > 0$ be given, $\cD_{\mathrm{off}}$ consist of $H \cdot T$ samples from data distribution $\nu$, where $\nu$ satisfies $\Cstar$-single-policy concentrability.  Suppose that $Q^\star \in \cF$ and that for all $t \in [T]$, $\pi \in \Pi$, and $h \in [H]$, we have $\nicefrac{d^\pi_h}{\mu\ind{t}_h} \in \cW$, where $\mu\ind{t}_h \coloneqq \nicefrac{1}{2}(\nu_h + \nicefrac{1}{t}\textstyle\sum_{i=1}^{t}d^{\pi\ind{i}}_h)$. %
	Then, \HyAlg with inputs $T = \wt  \Theta \prn{\prn{\nicefrac{H^4 (\Ccov + \Cstar)}{\epsilon^2}} \cdot \log\prn{\nicefrac{\abs*{\cF}\abs*{\cW}}{\delta}}}$, $\cF$, augmented $\widebar{\cW}$ defined in \eqref{eq:augmented-w}, $\gamma = \widetilde{\Theta}\left(\sqrt{\nicefrac{(\Cstar + \Ccov)}{T H^2 \log(\nicefrac{|\cF||\cW|}{\delta})}}\right)$, and $\cD_{\mathrm{off}}$ returns an $\varepsilon$-suboptimal policy with probability at least $1- \delta T$ after collecting \loose
	\[
		N = \wt O \prn*{\frac{H^2 (\Ccov + \Cstar)}{\epsilon^2} \log\prn{\nicefrac{\abs*{\cF}\abs*{\cW}}{\delta}}}
	\]
		trajectories.\loose

\end{restatable}
}

\arxiv{
\begin{algorithm} 
  \caption{\HyAlg: \HybridAlg $+$ \Crmabo} 
    \label{alg:hyglow} 
\begin{algorithmic}[1]  
\Statex[0] \textbf{input:} Parameter $T \in \bbN$, value function class \(\cF\), weight
function class \(\cW\), parameter  \(\gamma \in \brk{0,1}\), offline datasets $\cD_{\mathrm{off}} = \{\cD_{\mathrm{off},h}\}_h$ each of size $T$. 
\State \algcomment{For \arxiv{\pref{cor:hyglow}} \iclr{Thm.~\ref{thm:online_main_faster}},   set    \(T = \wt  \Theta \prn{\prn{\nicefrac{H^4 (\Ccov + \Cstar)}{\epsilon^2}} \cdot \log\prn{\nicefrac{\abs*{\cF}\abs*{\cW}}{\delta}}} \) , and $\gamma = \widetilde{\Theta}\left(\sqrt{\nicefrac{(\Cstar + \Ccov)}{TH^2 \log(\nicefrac{|\cF||\cW|}{\delta})}}\right)$.}  
\State Set \(\gamma\ind{t} = \gamma\cdot t\), and \(\alpha\ind{t} = \nicefrac{8}{\gamma\ind{t}}\).
\State Initialize \(\cD\ind{1}_{\mathrm{on},h} = \cD\ind{1}_{\mathrm{hybrid},h} = \emptyset\) for all \(h \in [H]\).
\For{$t=1,\ldots,T$}
\State \multiline{Compute value function \(f\ind{t}\) such that 
  {\small \begin{equation}
      \label{eq:hyglow_alg}
f\ind{t}\in \argmin_{f \in \cF} \max_{w \in \cW} \sum_{h=1}^H \left| \wh{\bbE}_{\cD\ind{t}_{\mathrm{hybrid},h}} \left[\check w_h(x_h,a_h)\whDf(x_h, a_h, r_h, x'_{h+1})\right]\right| - \alpha\ind{t} \wh{\bbE}_{\cD\ind{t}_{\mathrm{hybrid},h}}\left[\check w^2_h(x_h,a_h) \right], 
\end{equation}}
where \(\check w \ldef \clip{w}{\gamma\ind{t}}\) and \(\whDf(x, a, r, x') \ldef{} f_h(x,a)- r - \max_{a'}f_{h+1}(x', a')\).} 
\State Compute policy \(\pi\ind{t} \leftarrow \pi_{f\ind{t}}\).
\State \mbox{Collect trajectory  
  $(x_1,a_1,r_1),\ldots, (x_H,a_H,r_H)$ using \(\pi \ind{t}\); \(\cD\ind{t+1}_{\mathrm{on},h} \coloneq \cD\ind{t}_{\mathrm{on},h} \cup \crl{(x_h, a_h, r_h, x_{h+1})}\).}
\State Aggregate offline and online data: $\cD\ind{t+1}_{\mathrm{hybrid},h}\coloneqq \restr{\cD_{\mathrm{off},h}}{1:t} \cup \cD\ind{t+1}_{\mathrm{on},h}$ for all $h \in [H]$.  \label{line:data-agg} 
\EndFor
\State \textbf{output:} policy $\pihat=\unif(\pi\ind{1}, \dots, \pi\ind{T})$.\hfill 
   \end{algorithmic} 
\end{algorithm} 
}
\iclr{
\pacomment{changed these theorem statements in the arxiv version -- make it consistent}
\begin{restatable}[Risk bound for \HyAlg]{theorem}{hyglow}\label{cor:hyglow}
	Let $\varepsilon > 0$ be given, let $\cD_{\mathrm{off}}$ consist of $H \cdot T$ samples from data distribution $\nu$, and suppose that $\nu$ satisfies $\Cstar$-single-policy concentrability.  Suppose that $Q^\star \in \cF$ and that for all $t \in [T]$, $\pi \in \Pi$, and $h \in [H]$, we have $\nicefrac{d^\pi_h}{\mu\ind{t}_h} \in \cW$, where $\mu\ind{t}_h \coloneqq \nicefrac{1}{2}(\nu_h + \nicefrac{1}{t}\textstyle\sum_{i=1}^{t}d^{\pi\ind{i}}_h)$. %
	Then, \HyAlg (\cref{alg:hyglow}) with inputs $T = \wt  \Theta \prn{\prn{\nicefrac{H^4 (\Ccov + \Cstar)}{\epsilon^2}} \cdot \log\prn{\nicefrac{\abs*{\cF}\abs*{\cW}}{\delta}}}$, $\cF$, augmented $\widebar{\cW}$ defined in \eqref{eq:augmented-w}, $\gamma = \widetilde{\Theta}\left(\sqrt{\nicefrac{(\Cstar + \Ccov)}{T H^2 \log(\nicefrac{|\cF||\cW|}{\delta})}}\right)$, and $\cD_{\mathrm{off}}$ returns an $\varepsilon$-suboptimal policy with probability at least $1- \delta T$ after collecting $
		N = \wt O \prn*{\frac{H^2 (\Ccov + \Cstar)}{\epsilon^2} \log\prn{\nicefrac{\abs*{\cF}\abs*{\cW}}{\delta}}}
	$
		trajectories.
\end{restatable}
}

\arxiv{
The realizability assumptions parallel those required that \Alg to obtain the analogous $\nicefrac{\Ccov}{\varepsilon^2}$-type sample complexity for the purely online setting (cf. \cref{ass:W_realizability_mixture}).}
\iclr{This result is obtained from \cref{thm:htoregret} by showing that \Crmabo is \Ccbounded under $Q^\star$-realizability and a density ratio realizability assumption (cf. \cref{thm:mabooffline}). }While the sample complexity \iclr{and realizability assumptions parallel those of}\arxiv{matches that of} \Alg, the computational efficiency is improved because we remove the need for global optimism. More specifically, note that when clipping and the absolute value signs are removed from \eqref{eq:clip-mabo}, the optimization problem is concave-convex in the function class $\cF$ and weight function class $\cW$, a desirable property shared by standard density-ratio based offline algorithms \citep{xie2020q,zhan2022offline}.\arxiv{\footnote{Thus, if $\cF$ and $\cW$ are parameterized as linear functions, 
\arxiv{(i.e., $\cF=\crl*{(x,a)\mapsto{}\tri*{\phi(x,a),\theta}\mid{}\theta\in\Theta_{\cF}}$ and $\cW=\crl*{(x,a)\mapsto{}\tri*{\psi(x,a),\theta}\mid{}\theta\in\Theta_{\cW}}$ for feature maps $\phi$ and $\psi$)} it can be solved in polynomial time using standard tools for minimax optimization \citep[e.g.,][]{nemirovski2004prox}.}}\arxiv{Thus, if $\cF$ and $\cW$ are parameterized as linear functions, 
\arxiv{(i.e., $\cF=\crl*{(x,a)\mapsto{}\tri*{\phi(x,a),\theta}\mid{}\theta\in\Theta_{\cF}}$ and $\cW=\crl*{(x,a)\mapsto{}\tri*{\psi(x,a),\theta}\mid{}\theta\in\Theta_{\cW}}$ for feature maps $\phi$ and $\psi$)} it can be solved in polynomial time using standard tools for minimax optimization \citep[e.g.,][]{nemirovski2004prox}.} To accommodate clipping and the absolute value signs efficiently, we note that \eqref{eq:clip-mabo} can be written as a convex-concave program in which the max player optimizes over the set $\wt{\cW}_{\gamma, h} \ldef{} \{\pm\clip{w_h}{\gamma n} \mid w_h \in \cW_h\}$\iclr{ and the result continues to hold if the max player optimizes over any expanded weight function class \arxiv{$\cW'$ that satisfies} $\widetilde{\cW}_\gamma \subseteq \cW'$ such that $\nrm*{w}_{\infty}\leq\gamma{}n$ for all \(w \in \cW'\); we defer the details to \cref{app:comp-eff-hyglow}.  \loose}
\arxiv{. While this set may be complex, the result continues to hold if the max player optimizes over any expanded weight function class \arxiv{$\cW'$ that satisfies} $\widetilde{\cW}_\gamma \subseteq \cW'$ and $\nrm*{w}_{\infty}\leq\gamma{}n$ for all \(w \in \cW'\), thus allowing for the use of, e.g., convex relaxations or alternate parameterizations. We defer the details to \cref{app:comp-eff-hyglow}.  \loose}
\arxiv{
\begin{remark}[Comparison to offline RL]\label{rem:offline-comparison}
  It is instructive to compare the performance of \HyAlg to existing results for purely offline RL, which assume access to a data distribution $\nu$ with single-policy concentrability (\cref{ass:pi-star-concentrability}).
  Let us write $w^\pi \coloneqq \nicefrac{d^{\pi}}{\nu}$, $w^\star \coloneqq \nicefrac{d^{\pi^\star}}{\nu}$ and $V^\star$ for the optimal value function. The most relevant work is the {\normalfont\textsc{Pro-Rl}} algorithm of \citet{zhan2022offline}. Their algorithm is computationally efficient and enjoys a polynomial sample complexity bound under the realizability of certain \textit{regularized versions} of $w^\star$ and $V^\star$. By contrast, our result requires $Q^\star$-realizability and the density ratio realizability of $w^\pi$ for all $\pi \in \Pi$, but for the unregularized problem. These assumptions are not comparable, 
  and thus these results are best thought of as complementary.\footnote{The sample complexity bound in \citet{zhan2022offline} is also slightly larger, scaling roughly as $\wt O \prn*{\nicefrac{H^6 \Cstar^4 C_{\star,\varepsilon}^2}{\varepsilon^6}}$, where $C_{\star,\varepsilon}$ is the single-policy concentrability for the regularized problem, as opposed to our $\widetilde{O}\prn[\big]{\nicefrac{H^2(\Cstar + \Ccov)}{\veps^2}}$.} However, our approach requires additional online access, while their algorithm does not. \loose

To the best of our knowledge, all other algorithms for the purely offline setting that only require single-policy concentrability either need stronger representation conditions (such as value-function completeness \citep{xie2021bellman}), or are not known to be computationally efficient in the general function approximation setting due to the need for implementing pessimism \citep[e.g.,][]{chen2022offline}. \loose
\end{remark}
}

\arxiv{
  \subsection{Generic Reductions from Online to Offline RL?}
Our hybrid-to-offline reduction
\HybridAlg and the $\mathsf{CC}$-boundedness definition also shed light on the question of when offline RL methods can be lifted to the \textit{purely online setting}. Indeed, observe that any offline algorithm which satisfies $\mathsf{CC}$-boundedness (\cref{ass:offline-risk-cc}) with only a $\pihat$-coverage term, namely which satisfies an offline risk bound of the form
    \begin{equation}\label{eq:offline_online_risk}
\RiskOff \leq \sum_{h=1}^H \frac{\ascale}{n} \En_{\wh \pi \sim p}\brk*{\Ccc{n}{h}{\wh{\pi}}{\mu\ind{1:n}}{\gamma n}}+ \bscale,
    \end{equation}
    can be repeatedly invoked within $\HybridAlg$ (with $\cD_{\mathrm{off}} = \emptyset$) to achieve a small $\sqrt{\nicefrac{\Ccov}{T}}$-type risk bound for the purely online setting, with no hybrid data. This can be seen immediately by inspecting our proof for the hybrid setting (\cref{thm:htoregret-general} and \cref{thm:htoregret}).
     
   We can think of algorithms satisfying \eqref{eq:offline_online_risk} as \textit{optimistic offline RL algorithms}, since their risk only scales with a term depending on their own output policy; this is typically achieved using optimism. In particular, it is easy to see, that \Alg and \textsc{Golf} \citep{jin2021bellman,xie2022role} can be interpreted as repeatedly invoking such an optimistic offline RL algorithm within the \HybridAlg reduction. This class of algorithms has not been considered in the offline RL literature since they inherit both the computational drawbacks of pessimistic algorithms and the statistical drawbacks of ``neutral'' (i.e. non-pessimistic) algorithms (at least, when viewed only in the context of offline RL). %

   In more detail, as with pessimism, optimism is often not computationally efficient, although it furthermore requires all-policy concentrability (as opposed to single-policy concentrability) to obtain low \textit{offline} risk. On the other hand, neutral (non-pessimistic) algorithms such as \textsc{Fqi} \citep{chen2019information} and \textsc{Mabo} \citep{xie2020q} also require all-policy concentrability, but are more computationally efficient. %
    		However, our reduction shows that these algorithms might merit further investigation. In particular, it uncovers that they can automatically solve the online setting (without hybrid data) under coverability and when repeatedly invoked on datasets generated from their previous policies. We find that this reduction advances the fundamental understanding of sample-efficient algorithms in both the online and offline settings, and are optimistic that this understanding can be used for future algorithm design. \loose
              }

\arxiv{\section{Discussion}}
\iclr{\section{Conclusion}}
\label{sec:discussion}
Our work shows for the first time that density ratio modeling has
provable benefits for online reinforcement learning, and serves as
step in a broader research program that aims to
  clarify connections between online and offline reinforcement
  learning. \iclr{On the theoretical side, promising directions for future research include (i)
    resolving whether sample efficiency is possible under only coverability and
    value function realizability, and (ii) generic reductions from offline
    to \arxiv{fully }online RL (see \cref{sec:generic-offline-online} for a brief discussion). %
    In addition, we are excited to explore the possibility of developing
practical and computationally efficient online reinforcement learning
algorithms based on density ratio modeling.}\arxiv{To this end, we highlight some exciting directions for
  future research. \ascomment{I think we should make this an inline list or an actual list.} 
  \fakepar{Realizability} While our results show that density ratio realizability allows
    for sample complexity guarantees based on coverability that do not
    require Bellman completeness, the question of whether value
    function realizability alone is sufficient still remains.
  \fakepar{Generic reductions from online to offline RL} Our hybrid-to-offline reduction, \HybridAlg, sheds light on
    the question of when and how existing offline RL methods can be
    adapted \emph{as-is} to the hybrid setting. %
    Are there more general
    principles under which offline RL methods can be adapted to online settings?
    \dfcomment{unified understanding} 
   \fakepar{Practical and efficient online algorithms} Beyond the theoretical directions above, we are excited to explore the possibility of developing practical and computationally efficient online reinforcement learning algorithms based on density ratio modeling.}

\arxiv{\subsection*{Acknowledgements} 
AS thanks Sasha Rakhlin for useful discussions. AS acknowledges support from the 
Simons Foundation and NSF through award DMS-2031883, as well as from the DOE through award DE-
SC0022199. Nan Jiang acknowledges funding support from NSF IIS-2112471 and NSF CAREER IIS-2141781. 
}
\bibliography{paper} 

\clearpage

\appendix  

\renewcommand{\contentsname}{Contents of Appendix}
\addtocontents{toc}{\protect\setcounter{tocdepth}{2}}
{
  \hypersetup{hidelinks}
  \tableofcontents
}

\section{Additional Related Work}
\label{app:additional}

\paragraph{Online reinforcement learning}

\citet{xie2022role} introduce the notion of coverability and provide
regret bounds for online reinforcement learning under the assumption
of access to a value
function class $\cF$ satisfying Bellman
completeness. \citet{liu2023provable} extend their result to more
general coverage conditions under the same Bellman completeness
assumption. Our work complements these results by providing guarantees
based on coverability that do not require Bellman completeness.

To the best of our knowledge, our work is the first to provide
provable sample complexity guarantees for online reinforcement
learning that take advantage of the ability to model density
ratios, but a number of closely related works bear mentioning. Recent work of \citet{huang2023reinforcement} considers the low-rank 
MDP model and provides an algorithms which takes advantage of a form of
\emph{occupancy realizability}. Occupancy realizability, while related
to density ratio realizability, is stronger assumption in general: For
example, the Block MDP model
\citep{krishnamurthy2016pac,du2019latent,misra2019kinematic,zhang2022efficient,mhammedi2023representation}
admits a density ratio class of low complexity, but does not admit a
small occupancy class. Overall, however, their results are somewhat
complementary, as they do not require any form of value function
realizability. A number of other recent works in online reinforcement
learning also make use of occupancy measures, but restrict to linear
function approximation
\citep{neu2020unifying,neu2023efficient}. \dfcomment{someone should
 check what these gergely papers are actually doing.} Lastly, a number
of works apply density ratio modeling in online settings with an
empirical focus \citep{feng2019kernel,nachum2019algaedice}, but do not
address the exploration problem. \dfcomment{add
  \citet{feng2019kernel}, check if there are others}

  \paragraph{Offline reinforcement learning}
  Within the literature on offline reinforcement learning theory, density ratio modeling has been widely used as a means to avoid strong Bellman completeness requirements,
  with theoretical guarantees for policy evaluation
  \citep{liu2018breaking,uehara2020minimax,yang2020off,uehara2021finite} and
  policy optimization
  \citep{jiang2020minimax,xie2020q,zhan2022offline,chen2022offline,rashidinejad2022optimal,ozdaglar2023revisiting}.
  A number of additional works investigate density ratio modeling
  with an empirical focus, and do not provide finite-sample guarantees
  \citep{nachum2019algaedice,kostrikov2019imitation,nachum2020reinforcement,zhang2020gendice,lee2021optidice}. \dfcomment{double
    check that this is true}

\paragraph{Hybrid reinforcement learning}
\citet{song2022hybrid} were the first to show, theoretically, that the hybrid
reinforcement learning model can lead to computational benefits over
online and offline RL individually. Our reduction, \HybridAlg, can be
viewed as generalization of their Hybrid Q-Learning algorithm, with
their result corresponding to the special case in which \Fqi is applied
as a the base algorithm. Our guarantees under coverability complement
their guarantees based on bilinear rank. Other recent works on hybrid reinforcement learning in specialized
settings (e.g., tabular MDPs or linear MDPs) include \citet{wagenmaker2023leveraging,li2023reward, zhou2023offline}.

\section{Comparing Weight Function Realizability to Alternative
  Realizability Assumptions}
\label{app:realizability}
In this section, we compare the density ratio realizability assumption
in \cref{ass:W_realizability_proper} to a number of alternative 
realizability assumptions.

\paragraph{Comparison to Bellman completeness}
A traditional assumption in the analysis of value function 
approximation methods is to assume that \(\cF\)  satisfies a
representation condition called  \emph{Bellman completeness}, which
asserts that $\cT_h \cF_{h+1} \subseteq \cF_{h}$; this assumption is
significantly stronger than just assuming realizability of $Q^{\star}$
(\cref{ass:Q_realizability}), and has been used throughout offline RL
\citep{antos2008learning,chen2019information}, and online RL
\citep{zanette2020learning,jin2021bellman,xie2022role}.

Bellman completeness is incomparable to our density ratio
realizability assumption. For example, the low-rank MDP model
\citep{jin2020provably} in which
$P_h(x'\mid{}x,a)=\tri*{\phi_h(x,a),\psi_h(x')}$ satisfies Bellman
completeness when the feature map $\phi$ is known to the learner even
 if $\psi$ is unknown (but may not satisfy it otherwise), and
 satisfies weight function realizability when the feature map $\psi$
 is known even if $\phi$ is unknown (but may not satisfy it
 otherwise). \cref{ex:density,ex:bmdp} in \cref{app:examples} give
 further examples that satisfy weight function realizability but are
 not known to satisfy Bellman completeness.

\paragraph{Comparison to model-based realizability}
Weight function realizability is strictly weaker than model-based
realizability \citep[e.g.,][]{foster2021statistical}, in which one
assumes access to a  model class $\cM$ of MDPs that contains the true
MDP.\footnote{Up to $H$ factors, this is equivalent to assuming access
  to a class of transition distributions that can realize the true
  transition distribution, and a class of reward distributions that
  can realize the true reward distribution.} Since each MDP induces an occupancy for every policy,
it is straightforward to see that given such a class, we can construct a
weight function class $\cW$ that satisfies
\cref{ass:W_realizability_proper} with
\[
\log\abs{\cW} \leq \bigoh(\log\abs{\cM}+\log\abs{\Pi}), 
\]
as well as a realizable value function class with
$\log\abs{\cF}\leq\bigoh(\log\abs{\cM})$. On the other hand, weight
function realizability does not imply model-based realizability; a
canonical example that witnesses this is the \emph{Block MDP}.
\begin{example}[Block MDP]
  In the well-studied Block MDP model \citep{krishnamurthy2016pac,jiang2017contextual,du2019latent,misra2019kinematic,zhang2022efficient,mhammedi2023representation}, there exists realizable value function class $\cF$ and weight function class $\cW$ with $\log\abs{\cF},\log\abs{\cW}\approxleq\poly(\abs*{\cS},\abs*{\cA},\log\abs*{\Phi})$, where $\cS$ is the \emph{latent state space} and $\Phi$ is a class of \emph{decoder functions}. However, there does not exist a realizable model class with bounded statistical complexity (see discussion in, e.g., \citet{mhammedi2023representation}). 
\end{example}

  \paragraph{Alternative forms of density ratio realizability}

  \iclr{
\begin{remark}[Clipped density ratio realizability]
  \label{rem:clipped_realizability}
Since \Alg only accesses the weight function class $\cW$ through the
clipped weight functions, we can replace
\cref{ass:W_realizability_proper} with the assumption that for all
$\pi, \pi' \in \Pi$,  $t \leq T$,
we have \iclr{$  \clipX{\big}{w^{\pi;\pi'}}{\gamma t}\in\cW$, } 
\arxiv{\[
  \clip{w^{\pi;\pi'}}{\gamma t}\in\cW,
\]}
where $T\in\bbN$ and $\gamma\in\brk{0,1}$ are chosen as in
\cref{thm:online_main_basic}. Likewise, we can replace
\cref{ass:W_realizability_mixture} with the assumption that for all
$\pi \in \Pi$, and for all $t \leq T$, and $\pi_{1:t} \in \Pi^t$, we
have \iclr{$\clipX{\big}{w^{\pi;\pi\ind{1:t}}}{\gamma t} \in \cW$, }
\arxiv{\[
	\clip{w^{\pi;\pi\ind{1:t}}}{\gamma t} \in \cW,
      \]}
      for $T\in\bbN$ and $\gamma\in\brk{0,1}$ chosen as in \cref{thm:online_main_faster}.
    \end{remark}
    }
  
\arxiv{
      The following remarks concern slight variants of the density
      ratio realizability assumption.
\begin{remark}[Clipped density ratio realizability]
  \label{rem:clipped_realizability}
Since \Alg only accesses the weight function class $\cW$ through the
clipped weight functions, we can replace
\cref{ass:W_realizability_proper} with the assumption that for all
$\pi, \pi' \in \Pi$,  $t \leq T$, and \(h \in [H]\), 
we have \iclr{$  \clipX{\big}{w^{\pi;\pi'}}{\gamma t}\in\cW$, } 
\arxiv{\[
  \clip{w^{\pi;\pi'}_h}{\gamma t}\in\cW_h, 
\]}
where $T\in\bbN$ and $\gamma\in\brk{0,1}$ are chosen as in
\cref{thm:online_main_basic}. Likewise, we can replace
\cref{ass:W_realizability_mixture} with the assumption that for all
$\pi \in \Pi$, and for all $t \leq T$, \(h \in [H]\), and $\pi\ind{1:t} \in \Pi^t$, we
have \iclr{$\clipX{\big}{w^{\pi;\pi\ind{1:t}}}{\gamma t} \in \cW$, }
\arxiv{\[
	\clip{w^{\pi;\pi\ind{1:t}}_h}{\gamma t} \in \cW_h,
      \]}
      for $T\in\bbN$ and $\gamma\in\brk{0,1}$ chosen as in \cref{thm:online_main_faster}.
    \end{remark}
    }

\begin{remark}[Density ratio realizability relative to a fixed
  reference distribution]\label{rem:ratio-fixed-dist}
  \cref{ass:W_realizability_proper} is weaker than assuming access to a class $\cW$ that can realize the ratio $\nicefrac{d_h^{\pi}}{\nu_h}$ (or alternatively $\nicefrac{\nu_h}{d_h^{\pi}}$) for all $\pi\in\Pi$, where $\nu_h$ is an arbitrary fixed distribution (natural choices might include $\nu_h=\mustar_h$ or $\nu_h=d^{\pistar}_h$). Indeed, given access to such a class, the expanded class $\cW'\ldef{}\crl*{\nicefrac{w}{w'}\mid{}w,w'\in\cW}$ satisfies \cref{ass:W_realizability_proper}, and has $\log\abs{\cW'}\leq{}2\log\abs{\cW}$.
\end{remark}

\iclr{
  \begin{remark}[Connection to \textsc{Golf}]\label{rem:connection-golf}
    Prior work \citep{xie2022role} analyzed the {\normalfont \textsc{Golf}}
    algorithm of \citeauthor{jin2021bellman} and established positive
    results under coverability and Bellman completeness. We remark
    that by allowing for weight functions that take negative
    values,\footnote{In this context, $\cW$ can be thought of more
      generally as a class of \emph{test functions}.}
   \Alg can be viewed as a generalization of {\normalfont\textsc{Golf}}, and can
    be configured to obtain comparable results. Indeed, given a value
    function class $\cF$ that satisfies Bellman completeness, the
    weight function class  $\cW \ldef{} \{f - f' \mid f, f' \in \cF\}$
    leads to a confidence set construction at least as tight as that
    of {\normalfont \textsc{Golf}}. To see this, observe that if we set $\gamma \geq
    2$ so that no clipping occurs, our construction for
    $\cF\ind{t}$ (\eqref{eq:alg1}) implies (after standard
    concentration arguments) that $Q^\star \in \cF\ind{t}$ and that
    in-sample squared Bellman errors are small with high
    probability. These ingredients are all that is required to repeat
    the analysis of {\normalfont\textsc{Golf}} from \cite{xie2022role}.
  \end{remark}
	\pacomment{added this ``glow subsumes golf''
          paragraph. thoughts? too technical? not technical enough?}
        \dfcomment{seems fine, i am curious though about 1) do we get
          the fast rate?, and 2) how to set $\alpha$ and $\beta$. I
          think we may actually get the fast rate using the
          regularization which is cool}
}

\section{Examples for \Alg}
\label{app:examples}

In this section, we give two new examples in which our main results
for \Alg, \cref{thm:online_main_faster,thm:online_main_basic}, can be
applied:  MDPs with low-rank density features and general value
functions, and a class of MDPs we refer to as \dfedit{\emph{Generalized Block MDPs}}.

\begin{example}[Realizable $Q^{\star}$ with low-rank density features]
  \label{ex:density}
  Consider a setting in which (i) $\Qstar\in\cF$, and (ii), there
  is a known feature map $\psi_h:\cX\times\cA\to\bbR^{d}$ with
  $\nrm*{\psi_h(x,a)}_2\leq{}1$ such that for
  all $\pi\in\Pi$, $d^{\pi}_h(x,a)=\tri*{\psi_h(x,a),\theta^{\pi}}$
  for an unknown parameter $\theta^{\pi}\in\bbR^{d}$ with
  $\nrm*{\theta^{\pi}}_2\leq{}1$. This assumption is sometimes referred
  to as \emph{low occupancy complexity} \citep{du2021bilinear}, and
  has been studied with and without known features. In this case, we have
  $\Ccov\leq{}d$, and one can construct a weight function class $\cW$
  that satisfies \cref{ass:W_realizability_mixture} with $\log\abs{\cW}\leq{}\bigoht\prn*{dH}$ \citep{huang2023reinforcement}.\footnote{To be precise, $\cW$ is infinite, and this result requires a covering number bound. We omit a formal treatment, and refer to \citet{huang2023reinforcement} for details.}  As a result, \cref{thm:online_main_faster}
  gives sample complexity
  $\bigoht\prn*{H^{3}d^2\log\abs{\cF}/\veps^{2}}$. Note that while this setup requires that the occupancies themselves have low-rank structure, the class $\cF$ can consist of arbitrary, potentially nonlinear functions (e.g., neural networks). We remark that when the feature map $\psi$ is not known, but instead belongs to a known class $\Psi$, the result continues to hold, at the cost of expanding $\cW$ to have size $\log\abs{\cW}\leq\bigoht\prn[\big]{dH+\log\abs{\Psi}}$.

  This example is
  similar to but complementary to \citet{huang2023reinforcement}, who
  give guarantees for reward-free exploration under low-rank
  occupancies. Their results do not require any form
  of value realizability, but require a low-rank MDP
  assumption (which, in particular, implies Bellman completeness).\footnote{The low-rank MDP assumption is incomparable to the assumption we consider here, as it implies that $d_h^{\pi}(x)=\tri*{\psi_h(x),\theta^{\pi}}$, but does not necessarily imply that the \emph{state-action} occupancies (i.e. \(d_h^\pi(x, a)\)) are low-rank.}
\end{example}

Our next example concerns a generalization of the well-studied Block MDP
framework
\citep{krishnamurthy2016pac,jiang2017contextual,du2019latent,misra2019kinematic,zhang2022efficient,mhammedi2023representation}
that we refer to as the \emph{Generalized Block MDP}.

\begin{definition}[Generalized Block MDP]
A \emph{Generalized Block MDP} $\cM=(\cX,\cS,\cA,P_{\mathrm{latent}},R_{\mathrm{latent}},q, H, d_1)$ is comprised of an \emph{observation space} $\cX$, \emph{latent state space} $\cS$, \emph{action space} $\cA$, \emph{latent space transition kernel} $P_{\mathrm{latent}}:\cS\times\cA\to\Delta(\cS)$, and \emph{emission distribution} $q:\cS\to\Delta(\cX)$. The \emph{latent state space} evolves based on the agent's action $a_h\in\cA$ via the process
\begin{align}
  \label{eq:bmdp_latent}
  r_h\sim{}R_{\mathrm{latent}}(s_h,a_h),\quad s_{h+1} \sim{} P_{\mathrm{latent}}(\cdot\mid{}s_h,a_h),
\end{align}
with $s_1\sim{}d_1$; we refer to $s_h$ as the \emph{latent state}. The latent state is not observed directly, and instead we observe \emph{observations} $x_h\in\cX$ generated by the emission process
\begin{align}
  x_{h} \sim {} q(\cdot\mid{}s_h).
\end{align}
We assume that the emission process satisfies the \emph{decodability} property:
	\begin{align}
          \supp\, q(\cdot \mid s)\cap\supp\, q(\cdot \mid s')=\emptyset, \quad\forall s'\neq s\in\cS.
	\end{align}
        Decodability implies that there exists a
        (unknown to the agent) \emph{decoder} $\phi_\star\colon \cX
        \rightarrow \cS$ such that $\phi_\star(x_h)=s_h$ a.s. for all
        $h\in\brk{H}$, meaning that latent states can be uniquely
        decoded from observations. Prior work on the Block MDP
        framework \citep{krishnamurthy2016pac,jiang2017contextual,du2019latent,misra2019kinematic,zhang2022efficient,mhammedi2023representation}
assumes that the latent space $\cS$ and action space $\cA$ are
finite, but allow
the observation space $\cX$ to be large or potentially infinite. They
provide sample complexity guarantees that scale as
$\poly(\abs{\cS},\abs{\cA},H,\log\abs{\Phi},\veps^{-1})$, where $\Phi$
is a known class of decoders that contains $\phistar$.  We use the term \emph{Generalized Block MDP} to refer to Block MDPs in
which the latent space not tabular, and can be arbitrarily large.
\end{definition}

\begin{example}[Generalized Block MDPs with coverable latent states]
  \label{ex:bmdp}
We can use \Alg to give sample complexity guarantees for Generalized
Block MDPs in which the latent space is large, but has low coverability. Let $\Pilatent=(\cS\times\brk{H}\to\Delta(\cA))$
denote the set of all randomized policy that operate on the latent
space. Assume that the following conditions hold:
\begin{itemize}
\item We have a value function class $\cF_{\mathrm{latent}}$ such that
  $\Qstar_{\mathrm{latent}}\in\cF$, where $\Qstar_{\mathrm{latent}}$
  is the optimal $Q$-function for the latent space.
\item We have access to a class of \emph{latent space} density ratios
  $\cW_{\mathrm{latent}}$ such that for all $h\in\brk{H}$, and all
  $\pi,\pi'\in\Pilatent$,
  \[
    w_{\mathrm{latent},h}^{\pi,\pi'}(s,a)\ldef{}
    \frac{d^{\pi}_{\mathrm{latent},h}(s,a)}{d^{\pi'}_{\mathrm{latent},h}(s,a)} \in \cW_{\mathrm{latent}},
  \]
  where $d^{\pi}_{\mathrm{latent},h}=\bbP^{\pi}\prn*{s_h=s,a_h=a}$ is the latent occupancy measure.
  \item The \emph{latent coverability coefficient} is bounded:
    \[
      C_{\mathsf{cov}, \mathsf{latent}} \coloneqq \inf_{\mu_1,\dots,\mu_H \in \Delta(\cS \times \cA)} \sup_{\pi \in \Pilatent, h \in [H]} \left\| \frac{d_{\mathrm{latent},h}^\pi}{\mu_h}\right\|_\infty.
\]
\end{itemize}
We claim that whenever these conditions hold, analogous conditions
hold in observation space (viewing the Generalized BMDP as a large MDP), allowing \Alg and
\cref{thm:online_main_basic}  to be applied. Namely, we have:
\begin{itemize}
\item There exists a class $\cF$ satisfying \cref{ass:Q_realizability}
  in observation space such that $\log\abs*{\cF}\leq\bigoh\prn*{\log\abs*{\cF_{\mathrm{latent}}}+\log\abs*{\Phi}}$.
\item There exists a weight function class $\cW$ satisfying
  \cref{ass:W_realizability_proper} in observation space such that $\log\abs*{\cW}\leq\bigoh\prn*{\log\abs*{\cW_{\mathrm{latent}}}+\log\abs*{\Phi}}$.
\item We have $\Ccov\leq       C_{\mathsf{cov}, \mathsf{latent}}$.
\end{itemize}
As a result, \Alg attains sample complexity
$\poly(C_{\mathsf{cov.latent}},H,\log\abs{\cF_{\mathrm{latent}}},
\log\abs{\cW_{\mathrm{latent}}},\log\abs{\Phi},\veps^{-1})$.
This generalizes existing results for Block MDPs with tabular latent state
spaces, which have $\log\abs{\cF_{\mathrm{latent}}},
\log\abs{\cW_{\mathrm{latent}}}, C_{\mathsf{cov}, \mathsf{latent}}=\poly(\abs{\cS},\abs{\cA},H)$.
\end{example}

\clearpage

\section{Technical Tools}

\begin{lemma}[Azuma-Hoeffding]
\label{lem:azuma_hoeffding}  
 Let  \(M \in \bbN\) and $(Y_m)_{m\leq{M}}$ be a sequence of random variables adapted to a filtration $\prn{\filt_{m}}_{m\leq{}M}$. If \(\abs{Y_m} \leq R\) almost surely, then with probability at least \(1 - \delta\), 
 \begin{align*}
\abs*{\sum_{m=1}^M Y_m - \En_{m-1} \brk*{Y_m}} \leq R \cdot \sqrt{8 M \log(2\delta^{-1})}. 
\end{align*} 
\end{lemma} 

  \begin{lemma}[Freedman's inequality (e.g., \citealp{agarwal2014taming})]
  \label{lem:freedman}
 Let  \(M \in \bbN\) and  $(Y_m)_{m\leq{M}}$ be a real-valued martingale difference 
  sequence adapted to a filtration $\prn{\filt_m}_{m\leq{}M}$. If
  $\abs*{Y_m}\leq{}R$ almost surely, then for any $\eta\in(0,1/R)$, with probability at least $1-\delta$,
    \[ 
      \abs*{\sum_{m=1}^{M}Y_m} \leq{} \eta\sum_{m=1}^{M}\En_{m-1} \brk*{(Y_m)^{2}} + \frac{\log(2\delta^{-1})}{\eta}. 
    \] 
  \end{lemma}
  The following lemma is a standard consequence of
  \cref{lem:freedman} \citep[e.g.,][]{foster2021statistical}.
        \begin{lemma} 
      \label{lem:multiplicative_freedman} 
           Let  \(M \in \bbN\) and  $(Y_m)_{m\leq{M}}$ be a sequence of random
      variables adapted to a filtration $\prn{\filt_{m}}_{m\leq{}M}$. If
  $0\leq{}Y_m\leq{}R$ almost surely, then with probability at least 
  $1-\delta$, 
  \begin{align*}
    &\sum_{m=1}^{M}Y_m \leq{} 
                        \frac{3}{2}\sum_{m=1}^{M}\En_{m-1}\brk*{Y_m} +
                        4R\log(2\delta^{-1}),
    \intertext{and}
      &\sum_{m=1}^{M}\En_{m-1}\brk*{Y_m} \leq{} 2\sum_{m=1}^{M}Y_m + 8R\log(2\delta^{-1}).
  \end{align*}
    \end{lemma}

\subsection{Reinforcement Learning Preliminaries}  

\begin{lemma}[{\citet[Lemma 1]{jiang2017contextual}}] 
 \label{lem:pdl}
  For any value function \(f = (f_1, \dots, f_H)\), 
\begin{align*}
\En_{x_1 \sim d_1} \brk*{f_1(x_1, \pi_{f_1}(x_1))} - J(\pi_f) &= \sum_{h=1}^H \En_{d_h^{\pi_f}} \brk*{f_h(x_h, a_h) - \brk{\cT_h f_{h+1}}(x_h, a_h)}.  
\end{align*} 
\end{lemma} 

\begin{lemma}[{Per-state-action elliptic potential lemma; \citet[Lemma 4]{xie2022role}}] 
\label{lem:elliptical-potential}
Let \(d\ind{1}, \dots, d\ind{T}\) be an arbitrary sequence of distributions over a set \(\cZ\), and let \(\mu \in \Delta(\cZ)\) be a distribution such that \(\nicefrac{{d\ind{t}}(z)}{\mu(z)} \leq C\) for all \(z \in \cZ\)  and \(t \in [T]\). Then, for all \(z  \in \cZ\), we have 
\begin{align*} 
\sum_{t=1}^T \frac{{d\ind{t}}(z)}{\sum_{i < t} d\ind{m}(z) + C \mu(z)} &\leq 2 \log(1 + T). 
\end{align*} 
\end{lemma}

\clearpage

\section{Proofs from \creftitle{sec:online} (Online RL)}
\label{app:online} 
This section of the appendix is organized as follows: 
\begin{itemize}
\iclr{\item \cref{sec:proof_sketch} gives a high-level sketch of the 
  proof of \cref{thm:online_main_faster,thm:online_main_basic}, highlighting the role of clipped density ratios and coverability.}
\item \cref{sec:online_supporting} provides supporting technical
  results for \Alg, including concentration guarantees.
\item \cref{sec:glow_cumulative} presents our main technical result
  for \Alg, \pref{lem:cumulative_suboptimality}, which bounds the cumulative suboptimality of the iterates $\pi\ind{1},\ldots,\pi\ind{T}$
  produced by the algorithm for general choices of the parameters $T$,
  $K$, and $\gamma>0$.
\item Finally, in \cref{app:online_main_faster,app:online_main_basic},
  we invoke with specific parameter choices to prove
  \cref{thm:online_main_faster,thm:online_main_basic}, as well as more
  general results 
  (\cref{thm:online_main_faster_approximate,thm:online_main_basic_approximate})
  that allow for misspecification error.
\end{itemize}

\iclr{
\subsection{Overview of Proof Techniques}
\label{sec:proof_sketch}

}

\subsection{Supporting Technical Results}\label{sec:online_supporting}
For \(x, x' \in \cX\), \(a \in \cA\), \(r \in [0, 1]\), and \(h \in [H]\), recall the notation 
\begin{align*} 
&\whDf(x, a, r, x') = f_h(x, a) - r -  \max_{a'} f_{h}(x', a'),  \\ 
  &  \Df(x, a) =  f_h(x,a)-\brk{\cT_hf_{h+1}}(x,a), \\ 
  &\cw_h(x,a) = \clip{w_h}{\gamma\ind{t}}(x, a). 
\end{align*} 

\begin{restatable}[Basic concentration for \Alg]{lemma}{regconcentration}\label{lem:reg_concentration} Let \(\gamma\ind{t} \geq 0\) for \(t \in [T]\). 
With probability at least \(1 - \delta\), all of the following 
inequalities hold for all \(f \in \cF\), \(w \in \cW\), \(t \in [T]\) and \(h \in [H]\): 
\begin{enumerate}[label=\((\alph*)\), leftmargin=8mm] 
\item \mbox{$\abs[\Big]{\Ehat_{\cD\ind{t}_h} \brk*{\whDf(x_h, a_h, r_h, x'_{h+1}) \cdot \cw_h(x_h, a_h)}  -  \En_{\dbar\ind{t}_h} \brk*{\Df(x_h, a_h) \cdot \cw_h(x_h,a_h)}}  \leq \frac{10}{3 \gamma\ind{t}} \En_{\dbar\ind{t}_h} \brk*{\prn{\cw_h(x_h,a_h)}^2} + \frac{\beta\ind{t}}{12},$}
\item $\frac{1}{\gamma \ind{t}}\En_{\dbar\ind{t}_h}\brk*{\cw_h^2(x_h, a_h)} 
 \leq   \frac{2}{\gamma \ind{t}} \Ehat_{\cD\ind{t}_h}\brk*{\cw_h^2(x_h, a_h)} + \frac{2\beta\ind{t}}{9}, $ 
 \item $\frac{1}{\gamma \ind{t}}\Ehat_{\cD\ind{t}_h}\brk*{\cw_h^2(x_h, a_h)} \leq \frac{3}{2 \gamma \ind{t}}  \En_{\dbar\ind{t}_h}\brk*{\cw_h^2(x_h, a_h)} + \frac{\beta\ind{t}}{9},$ 
 \item $J(\pistar) - \En_{x_1 \sim d_1} \brk*{f_1(x_1, \pi_{f_1}(x_1))} \leq \Ehat_{x_1 \sim \cD\ind{t}_1} \brk*{ \max_a Q^\star_1(x_1, a_1)  - f_1(x_1, \pi_{f_1}(x_1))} + \sqrt{\frac{8  \log(6 \abs{\cF} \abs{\cW} T H / \delta)}{K(t-1)}}$, 
\end{enumerate} 
where \(\cw_h \ldef{} \clip{w_h}{\gamma\ind{t}}\) and \(\beta \ind{t} \ldef{}   \frac{36 \gamma \ind{t}}{K(t-1)} \log(6 \abs{\cF} \abs{\cW} T H / \delta)\). 
\end{restatable} 

\begin{proof}[\pfref{lem:reg_concentration}] Fix any \(h \in [H]\) and
  \(t \in [T]\). Let \(M = K(t-1)\) and recall that the dataset
  \(\cD_h^t\) consists of \(M\) tuples of the form \(\crl{(x_h\ind{m},
    a_h\ind{m}, r_h\ind{m}, x_{h+1}\ind{m})}_{m \leq M}\)  where
  \(x\ind{m}_{h+1} \sim P(\cdot \mid x\ind{m}_h, a\ind{m}_h)\), and
  \(a_h\ind{m} = \pi_{\tau(m)}(x_h\ind{m})\) where \(\tau(m) =
  \ceil*{\nicefrac{m}{K}}\).  Fix any \(f \in \cF\) and \(w \in \cW\).
\paragraph{Proof of \((a)\)}
For each \(m \in [M]\), define the random variable  
\begin{align*} 
Y_m %
&= \whDf(x_h\ind{m}, a_h\ind{m}, r_h\ind{m}, x_{h+1}\ind{m}) \cdot \cw_h(x_h\ind{m}, a_h\ind{m}) - \En_{d\ind{\tau(m)}_h} \brk*{\Df(x_h, a_h) \cdot \cw_h(x_h,a_h)}  
\end{align*} 

Clearly, $\En_{m-1} \brk*{Y_m} = 0$ and thus 
 \(\crl{Y_m}_{m \leq M}\) is a martingale difference sequence with 
\begin{align*}
\abs*{Y_m} &\leq 3 \sup_{x_h, a_h} \abs{ \cw_h(x_h, a_h)}  \leq 3 \gamma\ind{t}, 
\end{align*}
since \(\abs{\whDf(x_h\ind{m}, a_h\ind{m}, r_h\ind{m}, x_{h+1}\ind{m})} \leq 2\) and \(\abs{\Df(x_h, a_h)} \leq 1\). Furthermore, 
\begin{align*}
\sum_{m=1}^M Y_m &= \sum_{m=1}^M  \whDf(x_h\ind{m}, a_h\ind{m}, r_h\ind{m}, x_{h+1}\ind{m}) \cdot \cw_h(x_h\ind{m}, a_h\ind{m}) - \sum_{m=1}^M \En_{d\ind{\tau(m)}_h} \brk*{\Df(x_h, a_h) \cdot \cw_h(x_h,a_h)}  \\ 
&= 
K(t-1) \Ehat_{\cD\ind{t}_h} \brk*{\whDf(x_h, a_h, r_h, x'_{h+1}) \cdot \cw_h(x_h,a_h)}  - K \sum_{\tau=1}^{t-1} \En_{d\ind{\tau}_h} \brk*{\Df(x_h, a_h) \cdot \cw_h(x_h,a_h)} \\ 
&= K(t-1) \Ehat_{\cD\ind{t}_h} \brk*{\whDf(x_h, a_h, r_h, x'_{h+1}) \cdot \cw_h(x_h,a_h)}  - K(t-1) \En_{\dbar\ind{t}_h} \brk*{\Df(x_h, a_h) \cdot \cw_h(x_h,a_h)}. 
\end{align*} 
Additionally, we also have that 
\begin{align*}
\En_{m-1} \brk*{(Y_m)^2} &\leq  2 \En_{m-1} \brk*{\prn{\whDf(x_h\ind{m}, a_h\ind{m}, r_h\ind{m}, x_{h+1}\ind{m}) \cdot \cw_h(x_h\ind{m}, a_h\ind{m})}^2} \\
&\hspace{1in} + 2\En_{m-1} \brk*{\prn*{\En_{d\ind{\tau(m)}_h} \brk*{\Df(x_h, a_h) \cdot \cw_h(x_h,a_h)}}^2} \\
&\leq \En_{m-1} \brk*{8 {\cw_h(x_h\ind{m}, a_h\ind{m})}^2 + 2 \En_{d\ind{\tau(m)}_h} \brk*{ \prn*{\Df(x_h, a_h) \cdot \cw_h(x_h,a_h)}^2}} \\ 
&\leq  \En_{m-1} \brk*{8 {\cw_h(x_h\ind{m}, a_h\ind{m})}^2 + 2 \En_{d\ind{\tau(m)}_h} \brk*{\prn*{ \cw_h(x_h,a_h)}^2}} \\
&= 10 \En_{d\ind{\tau(m)}_h} \brk*{\prn*{ \cw_h(x_h,a_h)}^2}
\end{align*} 
where the second line follows since \(\abs{\whDf(x_h\ind{m}, a_h\ind{m}, r_h\ind{m}, x_{h+1}\ind{m})} \leq 2\) and by using Jensen's inequality, and the third line uses \(\abs{\Df(x_h, a_h)} \leq 1\).

Thus, using \pref{lem:freedman} with \(\eta = \nicefrac{1}{3 \gamma \ind{t}}\), we get that with probability at least \(1 - \delta'\),  
\begin{align*} 
\abs*{\sum_{m=1}^M Y_m} &= K(t-1) \abs*{\Ehat_{\cD\ind{t}_h} \brk*{\whDf(x_h, a_h, r_h, x'_{h+1}) \cdot \cw_h(x_h,a_h)}  -  \En_{\dbar\ind{t}_h} \brk*{\Df(x_h, a_h) \cdot \cw_h(x_h,a_h)}}  \\ 
&\leq  \frac{10 K}{3\gamma\ind{t}} \sum_{\tau=1}^{t-1} \En_{d\ind{\tau}_h} \brk*{\prn{\cw_h(x_h,a_h)}^2} + 3 \gamma \ind{t} \log(2/\delta') \\ 
&= \frac{10 K(t-1)}{3 \gamma\ind{t}} \En_{\dbar\ind{t}_h} \brk*{\prn{\cw_h(x_h,a_h)}^2} + 3 \gamma \ind{t} \log(2/\delta'). 
\end{align*} 

The above bound implies that 
\begin{align*}
\hspace{2in} &\hspace{-2in}\abs*{\Ehat_{\cD\ind{t}_h} \brk*{\whDf(x_h, a_h, r_h, x'_{h+1}) \cdot \cw_h(x_h,a_h)} - \En_{\dbar\ind{t}_h} \brk*{\Df(x_h, a_h) \cdot \cw_h(x_h,a_h)}}  \\ &\leq \frac{10}{3 \gamma\ind{t}} \En_{\dbar\ind{t}_h} \brk*{\prn{\cw_h(x_h,a_h)}^2} + \frac{3 \gamma \ind{t}}{K(t-1)} \log(2/\delta'). 
\end{align*} 

Plugging in the value of \(\beta\ind{t}\) gives the desired bound. 
The final result follows by setting \(\delta' = \nicefrac{\delta}{3\abs{\cF} \abs{\cW} T H}\), and taking another union bound over the choice of \(f, w, t\) and \(h\). 

\paragraph{Proof of \((b)\) and \((c)\)} For each \(m \in [M]\), define the random variable 
\begin{align*}
Y_m = \prn*{\cw_h(x\ind{m}_h, a\ind{m}_h)}^2. 
\end{align*}
Clearly, the sequence \(\crl{Y_m}_{m \leq M}\) is adapted to an increasing filtration, with \(Y_t \geq 0\)  and \(\abs{Y_t} = \abs{\prn{\cw_h(x\ind{m}_h, a\ind{m}_h)}^2} \leq (\gamma\ind{t})^2\). Furthermore, 
\begin{align*}
\sum_{m=1}^M \En_{m-1} \brk*{Y_m} &= K \sum_{\tau=1}^{t-1} \En_{d\ind{\tau}_h} \brk*{\prn{\cw_h(x_h, a_h)}^2} = K(t-1) \En_{\dbar\ind{t}_h} \brk*{\prn{\cw_h(x_h, a_h)}^2}, \\ 
\intertext{and, } 
\sum_{m=1}^M Y_m &=  \sum_{(x, a) \in \cD\ind{t}_h} \prn{\cw_h(x_h,a_h)}^2 = K(t-1)  \Ehat_{\cD_h\ind{t}} \brk*{\prn{\cw_h(x_h,a_h)}^2}. 
\end{align*} 

Thus, by \cref{lem:multiplicative_freedman}, we have that with
probability at least $1 - \delta'$, 
\begin{align*}
\En_{\dbar\ind{t}_h}\brk*{\prn{\cw_h(x_h,a_h)}^2}
 &\leq   2 \Ehat_{\cD\ind{t}_h}\brk*{\prn{\cw_h(x_h,a_h)}^2} +
 \frac{8(\gamma\ind{t})^2 \log(2/\delta')}{K(t-1)}, 
 \intertext{and}
  \Ehat_{\cD\ind{t}_h}\brk*{\prn{\cw_h(x_h,a_h)}^2} &\leq \frac{3}{2} \En_{\dbar\ind{t}_h}\brk*{\prn{\cw_h(x_h,a_h)}^2} +  \frac{4(\gamma\ind{t})^2 \log(2/\delta')}{K(t-1)}. 
 \end{align*}
 
 The final result follows by setting \(\delta' = \nicefrac{\delta}{3\abs{\cF} \abs{\cW} T H}\), and taking another union bound over the choice of \(f, w, t\) and \(h\). 
 
\paragraph{Proof of \((d)\)} For each \(m \in [M]\), define the random variable 
\begin{align*}
Y_m = \max_a Q^\star_1(x\ind{m}_1, a)  - f_1(x\ind{m}_1, \pi_{f_1}(x\ind{m}_1)). 
\end{align*}

Clearly, \(\abs{Y_m} \leq 1\). Thus, using \pref{lem:azuma_hoeffding}, we get that with probability at least  \(1 - \delta'\), 
 \begin{align*}
\sum_{m=1}^M \En_{m-1} \brk*{Y_m} \leq \sum_{m=1}^M \prn*{ \max_a Q^\star_1(x\ind{m}_1, a)  - f_1(x\ind{m}_1, \pi_{f_1}(x\ind{m}_1))} + \sqrt{8 M \log(2/\delta')}.  
\end{align*} 
Setting \(M = K(t-1)\) and noting that \(\En_{m-1} \brk*{Y_m} = J(\pistar) - \En_{x_1 \sim d_1} \brk*{f_1(x_1, \pi_{f_1}(x_1)}\)  since \(x\ind{m}_1 \sim d_1\) for any \(m \in [M]\), we get that 
\begin{align*}
J(\pistar) - \En_{x_1 \sim d_1} \brk*{f_1(x_1, \pi_{f_1}(x_1)} &\leq \En_{x \sim \cD\ind{t}_1} \brk*{ \max_a Q^\star_1(x_1, a)  - f_1(x_1, \pi_{f_1}(x_1))} + \sqrt{\frac{8  \log(2/\delta')}{K(t-1)}} 
\end{align*}

The final result follows by setting \(\delta' = \nicefrac{\delta}{3\abs{\cF} \abs{\cW} T H}\), and taking another union bound over the choice of \(f, w, t\) and \(h\). 
\end{proof}
 
\begin{lemma}[Properties of \Alg confidence set]
\label{lem:concentration2} Let \(\gamma\ind{t} \geq 0\) for \(t \in [T]\).  
With probability at least $1-\delta$, all of the following events hold:
  \begin{enumerate}[label=\((\alph*)\)] 
  \item For all \(t \geq 1\), $\Qstar\in\cF\ind{t}$ 
  \item For all $t\geq{}2$,  $h\in\brk{H}$, $f\in\cF\ind{t}$, and \(w \in \cW\), we have 
\begin{align*}
  \En_{\dbar\ind{t}_h}\brk*{\Df(x_h, a_h) \cdot \cw_h(x_h,a_h)} &\leq \frac{20}{\gamma\ind{t}} \En_{\dbar\ind{t}_h} \brk*{\prn{\cw_h(x_h,a_h)}^2} 
  + \frac{7 \beta\ind{t}}{18}. 
\end{align*} 
Furthermore, 
\begin{align*} 
  \En_{\dbar\ind{t + 1}_h}\brk*{\Df(x_h, a_h) \cdot \cw_h(x_h,a_h)} &\leq \frac{40}{\gamma \ind{t}} \En_{\dbar\ind{t + 1}_h} \brk*{\prn{\cw_h(x_h,a_h)}^2} + \frac{7 \beta\ind{t}}{9} + \frac{\gamma \ind{t}}{160t^2},  
\end{align*} 
\item For all \(t \geq 2\), we have 
\begin{align*}
 \En_{x_1 \sim d_1} \brk*{ \max_a Q^\star_1(x_1, a)  - f\ind{t}_1(x_1, \pi\ind{t}_1(x_1))} &\leq \sqrt{\frac{8  \log(6 \abs{\cF} \abs{\cW} T H / \delta)}{K(t-1)}}, 
\end{align*} 
where \(\cw_h \ldef{} \clip{w_h}{\gamma\ind{t}}\) and \(\beta \ind{t} =  \frac{36 \gamma \ind{t}}{K(t-1)} \log(6 \abs{\cF} \abs{\cW} T H / \delta)\).\end{enumerate}

\end{lemma} 
\begin{proof}[\pfref{lem:concentration2}] Using \pref{lem:reg_concentration}, we have that with probability at least \(1 - \delta\), for all \(f \in \cF\), \(w \in \cW\), \(t \in [T]\) and \(h \in [H]\), 
\begin{align*}
 &\hspace{-2in}  \abs*{\Ehat_{\cD\ind{t}_h} \brk*{\whDf(x_h, a_h, r_h, x'_{h+1}) \cdot \cw_h(x_h,a_h)}  -  \En_{\dbar\ind{t}_h} \brk*{\Df(x_h, a_h) \cdot \cw_h(x_h,a_h)}}   \\ & \leq \frac{10}{3 \gamma\ind{t}} \En_{\dbar\ind{t}_h} \brk*{\prn{\cw_h(x_h,a_h)}^2} + \frac{\beta\ind{t}}{12}, \numberthis \label{eq:conc1} \\  
\frac{1}{\gamma \ind{t} }\En_{\dbar\ind{t}_h}\brk*{\prn{\cw_h(x_h,a_h)}^2} 
 &\leq  \frac{ 2}{\gamma \ind{t} } \Ehat_{\cD\ind{t}_h}\brk*{\prn{\cw_h(x_h,a_h)}^2} + \frac{2\beta\ind{t}}{9}, \numberthis \label{eq:conc2} \\ 
\frac{1}{\gamma \ind{t} }\Ehat_{\cD\ind{t}_h}\brk*{\prn{\cw_h(x_h,a_h)}^2} &\leq \frac{3}{2 \gamma \ind{t}} \En_{\dbar\ind{t}_h}\brk*{\prn{\cw_h(x_h,a_h)}^2} + \frac{\beta\ind{t}}{9}, \numberthis \label{eq:conc3}
\intertext{and,} 
J(\pistar) - \En_{d_1} \brk*{f_1(x_1, \pi_{f_1}(x_1)} &\leq \Ehat_{\cD\ind{t}_1} \brk*{ \max_a Q^\star_1(x_1, a)  - f_1(x_1, \pi_{f_1}(x_1))} \\
&\hspace{1.3in} + \sqrt{\frac{8  \log(6 \abs{\cF} \abs{\cW} T H / \delta)}{K(t-1)}}.  %
\numberthis \label{eq:conc4}
\end{align*} 

For the rest of the proof, we condition on the event in which \((\ref{eq:conc1}\)-\(\ref{eq:conc4})\) hold. 

\paragraph{Proof of \((a)\)} Consider any \(t \in [T]\), and observe that the optimal state-action value function $\Qstar$ satisfies for any \(w_h \in \cW\), 
\begin{align*}
  \En_{\dbar\ind{t}_h}\brk*{
        \prn{
    \Qstar_h(x_h,a_h)-r_h-\max_{a'}\Qstar_{h+1}(x'_{h+1},a')  
    } \cdot \cw_h(x_h,a_h) } = 0, 
\end{align*}
where \(\check w _h \ldef{} \clip{w_h}{\gamma\ind{t}}\). 
Using the above relation with \pref{eq:conc1}, we get that 
\begin{align*}
 \Ehat_{\cD\ind{t}_h}\brk*{
        \prn{
    \Qstar_h(x_h,a_h)-r_h-\max_{a'}\Qstar_{h+1}(x'_{h+1},a') 
    }\cdot \cw_h(x_h,a_h)
  }  
  &\leq \frac{10}{3 \gamma\ind{t}} \En_{\dbar\ind{t}_h} \brk*{\prn{\cw_h(x_h,a_h)}^2} + \frac{\beta\ind{t}}{12} \\
    &\leq \frac{4}{\gamma\ind{t}} \En_{\dbar\ind{t}_h} \brk*{\prn{\cw_h(x_h,a_h)}^2} + \frac{\beta\ind{t}}{12} \\  
  &\leq  \frac{8}{\gamma\ind{t}} \wh \En_{\cD\ind{t}_h} \brk*{\prn{\cw_h(x_h,a_h)}^2} + \beta \ind{t}, 
\end{align*}
where the second-last inequality follows from \pref{eq:conc2}. 

Plugging in the values of \(\alpha\ind{t}\) and \(\beta\ind{t}\), rearranging the terms, we get that   
\begin{align*}
\Ehat_{\cD\ind{t}_h}\brk*{ 
        \prn{
    \Qstar_h(x_h,a_h)-r_h-\max_{a'}\Qstar_{h+1}(x'_{h+1},a') 
    }\cdot \cw_h(x_h,a_h)
   -  \alpha\ind{t} \prn{\cw_h(x_h,a_h)}^2} \leq \beta \ind{t}.  
\end{align*}
Since the above inequality holds for all \(w \in \cW\), we have that $\Qstar\in\cF\ind{t}$. 

\paragraph{Proof of \((b)\)} Fix any \(t\), and note that by the definition of \(\cF\ind{t}\), any $f\in\cF\ind{t}$ satisfies for any \(w \in \cW\), the bound 
\begin{align*} 
  \Ehat_{\cD\ind{t}_h}\brk*{
\prn{
  f_h(x_h,a_h)-r_h-\max_{a'}f_{h+1}(x'_{h+1},a')  
    } \cdot   \cw_h(x_h,a_h) 
  }
&  \leq \frac{10} {\gamma\ind{t}}\Ehat_{\cD\ind{t}_h} \brk*{\prn{\cw_h(x_h,a_h)}^2} 
  + \beta \ind{t}.
\end{align*}

Using the above bound with \pref{eq:conc1}, we get that 
\begin{align*}
  \En_{\dbar\ind{t}_h}\brk*{\Df(x_h, a_h) \cdot \cw_h(x_h,a_h)} &\leq \frac{10}{3\gamma\ind{t}} \En_{\dbar\ind{t}_h} \brk*{\prn{\cw_h(x_h,a_h)}^2} + \frac{10}{\gamma\ind{t}}\Ehat_{\cD\ind{t}_h}\brk*{\prn{\cw_h(x_h,a_h)}^2}
  + \frac{13}{12} \beta\ind{t}. 
\end{align*} 
Plugging the bound from \pref{eq:conc3} for the second term above, we get that 
\begin{align*}
  \En_{\dbar\ind{t}_h}\brk*{\Df(x_h, a_h) \cdot \cw_h(x_h,a_h)} &\leq \frac{20}{\gamma\ind{t}} \En_{\dbar\ind{t}_h} \brk*{\prn{\cw_h(x_h,a_h)}^2}  
  + \frac{7 \beta\ind{t}}{18}. 
\end{align*}
Finally, noting that \(\dbar\ind{t+1} = \frac{(t - 1) \dbar\ind{t} + d\ind{t}}{t}\), we can further upper bound as: 
\begin{align*}
\hspace{0.5in} &\hspace{-0.5in} \En_{\dbar\ind{t + 1}_h}\brk*{\Df(x_h, a_h) \cdot \cw_h(x_h,a_h)} \\ 
&\leq \frac{t - 1}{t} \prn*{\frac{20}{\gamma\ind{t}} \En_{\dbar\ind{t}_h} \brk*{\prn{\cw_h(x_h,a_h)}^2} 
  + \frac{7 \beta\ind{t}}{18}} + \frac{1}{t} \En_{d\ind{t}_h}\brk*{\Df(x_h, a_h) \cdot \cw_h(x_h,a_h)} \\ 
  &\leq  2 \prn*{\frac{20}{\gamma\ind{t}} \En_{\dbar\ind{t}_h} \brk*{\prn{\cw_h(x_h,a_h)}^2} 
  + \frac{7 \beta\ind{t}}{18}} + \frac{1}{t}  \En_{d\ind{t}_h}\brk*{\abs{\cw_h(x_h,a_h)}} \\
  &\leq \frac{40}{\gamma \ind{t}}  \En_{\dbar\ind{t}_h} \brk*{\prn{\cw_h(x_h,a_h)}^2}  + \frac{7 \beta\ind{t}}{9} +  \frac{40}{\gamma \ind{t}} \En_{d\ind{t}_h}\brk*{\cw_h(x_h,a_h)^2} + \frac{\gamma \ind{t}}{160 t^2 } \\  
  &= \frac{40}{\gamma \ind{t}} \En_{\dbar\ind{t + 1}_h} \brk*{\prn{\cw_h(x_h,a_h)}^2} + \frac{7 \beta\ind{t}}{9} + \frac{\gamma\ind{t}}{160t^2}, 
\end{align*} 
where the second-to-last line follows from an application of AM-GM inequality. %
\paragraph{Proof of \((c)\)} Plugging in \(f = f\ind{t}\) in \pref{eq:conc4} and noting that \(\max_a Q^\star_1(x_1, a) = Q^\star_1(x, \pistar(x))\) for any \(x \in \cX\),  we get that 
\begin{align*}
J(\pistar) - \En_{x_1 \sim d_1} \brk*{f\ind{t}_1(x_1, \pi_{f\ind{t}_1}(x_1))} &\leq \Ehat_{x_1 \sim \cD\ind{t}_1} \brk*{ Q^\star_1(x_1, \pistar_1(x_1))  - f\ind{t}_1(x_1, \pi_{f\ind{t}_1}(x_1))} + \sqrt{\frac{8  \log(6 \abs{\cF} \abs{\cW} T H / \delta)}{K(t-1)}}. 
\end{align*}
However, note that by definition, \(f\ind{t} \in \argmax_{f} \Ehat \brk*{f_1(x_1, \pi_{f_1}(x_1)}\), and using part-(a), \(Q^\star \in \cF\ind{t}\). Thus, \(\Ehat_{\cD\ind{t}_1} \brk*{ Q^\star_1(x_1, \pistar_1(x_1))  - f\ind{t}_1(x_1, \pi_{f\ind{t}_1}(x_1))} \leq 0\), which implies that  
\begin{align*}
J(\pistar) - \En_{x_1 \sim d_1} \brk*{f\ind{t}_1(x_1, \pi_{f\ind{t}_1}(x_1))} &\leq \sqrt{\frac{8  \log(6 \abs{\cF} \abs{\cW} T H / \delta)}{K(t-1)}}. 
\end{align*} 
\end{proof}

\begin{restatable}[Coverability potential bound]{lemma}{covpotential} 
\label{lem:potential_coverability} 
Let \(d\ind{1}, \dots, d\ind{T}\) be an arbitrary sequence of distributions over \(\cX \times \cA\), such that there exists a distribution \(\mu \in \Delta(\cX \times \cA)\) that satisfies \(\nrm{\nicefrac{{d\ind{t}}}{\mu}}_\infty \leq C\) for all \( 
(x, a) \in \cX \times \cA\)  and \(t \in [T]\). Then, %
\begin{align*} 
 \sum_{t=1}^T \En_{(x, a) \sim d\ind{t}} \brk*{ \frac{{d\ind{t}}(x,a)}{\dtil\ind{t+1}(x,a)}} &\leq 5 \Cc \log(T), 
\end{align*} 
where recall that \(\dtil\ind{t + 1} \ldef{} \sum_{s=1}^t d\ind{t}\) for all \(t \in [T]\).  
\end{restatable} 

\begin{proof} [\pfref{lem:potential_coverability}] 
 Let $\tau(x,a) = \min\{t \mid \dtil\ind{t + 1}(x,a) \geq \Cc
 \mu(x,a)\}$. With this definition, we can bound
\begin{align*}
 \sum_{t=1}^T \En_{d\ind{t}} \brk*{ \frac{{d\ind{t}}(x,a)}{\dtil\ind{t+1}(x,a)}} &=  \sum_{t=1}^T \En_{d\ind{t}} \brk*{ \frac{{d\ind{t}}(x,a)}{\dtil\ind{t+1}(x,a)}  \cdot \indic\{ t < \tau(x, a)\}} \\ 
 &\hspace{1in} + \sum_{t=1}^T \En_{d\ind{t}} \brk*{ \frac{{d\ind{t}}(x,a)}{\dtil\ind{t+1}(x,a)}  \cdot \indic\{ t \geq \tau(x, a)\}} \\ 
 &\leq  \underbrace{ \sum_{t=1}^T  \En_{d\ind{t}} \brk*{\indic\{ t < \tau(x, a)\}}}_{\text{(I): burn-in phase}} + \underbrace{ \sum_{t=1}^T \En_{d\ind{t}} \brk*{ \frac{{d\ind{t}}(x,a)}{\dtil\ind{t+1}(x,a)}  \cdot \indic\{ t \geq \tau(x, a)\}}}_{\text{(II): stable phase}}, 
\end{align*} 
where the second line uses that \(\nicefrac{{d\ind{t}}(x,a)}{\dtil\ind{t+1}(x,a)} \leq 1\). 

For the burn-in phase,  note that  
\begin{align*}
 \sum_{t=1}^T \En_{d\ind{t}} \brk*{\indic\{t < \tau(x, a)\}}   = \sum_{x,a} \sum_{t < \tau(x,a)} d^t(x,a) &= \sum_{x,a}\dtil \ind{\tau(x,a)}(x, a) \leq \sum_{x,a} \Cc \mu(x,a) = \Cc ,
\end{align*} 
where the last inequality uses that by definition, $\tilde{d}^t(x, a) \leq \Cc \mu(x, a)$ for all \(t \leq \tau(x, a)\). 

For the stable phase, whenever 
\( t \geq \tau(x, a)\), by definition, we have $\dtil \ind{t+1}(x,a) \geq C \mu(x,a)$ which implies that $\dtil\ind{t + 1}(x,a) \geq \frac 1 2 \prn{\dtil\ind{t+1}(x,a) + C \mu(x,a)}$. Thus, 
\begin{align} 
\text{(II)} &\leq  2 \sum_{t=1}^T \En_{d\ind{t}} \brk*{\frac{{d\ind{t}}(x,a)}{\dtil\ind{t}(x,a) +\Cc \mu(x,a)}} \label{eq:cov-pot-one}\\  
  &= 2\sum_{x, a}  \sum_{t = 1}^T  \frac{d\ind{t}(x,a) \cdot d\ind{t}(x,a)}{\dtil\ind{t}(x,a) +\Cc \mu(x,a)} \nonumber \\ 
&\leq  2 \sum_{x, a} \max_{t'\in\brk{T}}d\ind{t'}(x,a) \max_{x, a } \prn*{ \sum_{t = 1}^T  \frac{ d\ind{t}(x,a)}{\dtil\ind{t}(x,a) +\Cc \mu(x,a)}} \nonumber.
\end{align}
Using the per-state elliptical potential lemma (\pref{lem:elliptical-potential}) in the above inequality, we get that  
\begin{align}
\text{(II)}  \leq 4 \log(T + 1) \sum_{x, a} \max_{t'\in\brk{T}}d\ind{t'}(x,a)  
\leq 4 \Cc \log(T + 1) \sum_{x, a} \mu(x,a)  = 4 \Cc \log(1+ T) \label{eq:cov-pot-two},
\end{align} 
where the second inequality follows from the fact that
\(\nrm*{\tfrac{d\ind{t}}{\mu}}_\infty \leq \Cc\) (by definition), and the last
equality uses that \(\sum_{x, a} \mu(x,a)  = 1\). Combining the above bound, we get that
\begin{align*}
 \sum_{t=1}^T \En_{d\ind{t}} \brk*{ \frac{{d\ind{t}}(x,a)}{\dtil\ind{t+1}(x,a)}} &\leq 5 \Cc \log(T). 
\end{align*}
\end{proof}

\subsection{Main Technical Result: Bound on Cumulative Suboptimality
  for \Alg}
\label{sec:glow_cumulative}
In this section we prove a key technical lemma,
\cref{lem:cumulative_suboptimality}, which gives a bound on the
cumulative suboptimality of the sequence of policies generated by
\Alg. Both the proofs of \pref{thm:online_main_basic} and
\pref{thm:online_main_faster} build on this result. To facilitate more
general sample complexity bounds that allow for misspecification error
in $\cW$.
In particular, for each $t\in\brk{T}$, we define  \(\xi_t\) as the misspecification error of the clipped
density ratio \(\nicefrac{d\ind{t}_h}{\dbar\ind{t+1}_h}\) in class
\(\cW_h\), defined as 
\begin{align}
  \label{eq:xi}
\xi_t \ldef{} \sup_{h \in [H]} \inf_{w \in \cW_h} \sup_{\pi \in \Pi} \nrm*{\clip{\frac{d\ind{t}_h}{\dbar_h\ind{t+1}}}{\gamma\ind{t}} - \clip{w_h} {\gamma\ind{t}}}_{1,\dbar^\pi_h}, 
\end{align}
where recall that for any function \(u: \cX \times \cA \mapsto \bbR\)
and distribution \(d \in \Delta(\cX \times \cA)\), the norm
\(\nrm{u}_{1, d} \ldef{} \En_{(x, a) \sim d} \brk*{\abs{u(x, a)}}\). Note
that under \pref{ass:W_realizability_proper} or \ref{ass:W_realizability_mixture}, \(\xi_t = 0\) for all \(t \in [T]\). 

\begin{lemma}[Bound on cumulative suboptimality] 
\label{lem:cumulative_suboptimality} Let \(\pi\ind{1}, \dots,
\pi\ind{T}\) be the sequence of policies generated by \Alg, when
executed on classes \(\cF\) and \(\cW\) with parameters \(T, K\) and
\(\gamma\). Then the cumulative suboptimality of the sequence of
policies \(\crl{\pi\ind{t}}_{t \in [T]}\) is bounded as
\begin{align*} 
    \sum_{t=1}^{T}J(\pistar) - J(\pi\ind{t}) 
    &= O\prn*{H\prn*{\frac{\Ccov \log(1+ T)}{\gamma} 
  + \frac{\gamma T \log(\abs*{\cF}\abs*{\cW}HT\delta^{-1})}{K} + \sum_{t=1}^T \xi_t + \gamma \log(T)}}. \end{align*}	

\end{lemma} 
\begin{proof}[\pfref{lem:cumulative_suboptimality}]
  Fix any \(t \geq 2\). We begin by establishing optimism as follows:
  \begin{align*} 
J(\pistar) - J(\pi\ind{t}) &=  \En_{x_1 \sim d_1} \brk*{ \max_a Q^\star_1(x_1, a)} - J(\pi\ind{t}) \\ 
&=  \En_{x_1 \sim d_1} \brk*{ \max_a Q^\star_1(x_1, a)  - f\ind{t}_1(x_1, \pi\ind{t}(x_1))} + \En_{x_1 \sim d_1} \brk*{f\ind{t}_1(x_1, \pi\ind{t}(x_1))} - J(\pi\ind{t}) \\
&\leq \sqrt{\frac{8  \log(6 \abs{\cF} \abs{\cW} T H / \delta)}{K(t-1)}} + \En_{x_1 \sim d_1} \brk*{f\ind{t}_1(x_1, \pi\ind{t}(x_1))} - J(\pi\ind{t}), 
\end{align*}
where the last line follows from \pref{lem:concentration2}-(c). 
Using \pref{lem:pdl} for the second term, we get that 
\begin{align*}
J(\pistar) - J(\pi\ind{t}) &\leq \sqrt{\frac{8  \log(6 \abs{\cF} \abs{\cW} T H / \delta)}{K(t-1)}} + \sum_{h=1}^H \En_{(x_h, a_h) \sim d\ind{t}_h} \brk*{\Dft(x_h, a_h)}, 
\end{align*} 
where recall that \(\Dft(x_h, a_h) \ldef{} f\ind{t}_h(x_h, a_h) - [\cT_h f\ind{t}_{h+1}](x_h, a_h)\). 
Thus, 
  \begin{align*}
    \sum_{t=1}^{T}J(\pistar) - J(\pi\ind{t})
    &\leq{} J(\pistar) - J(\pi\ind{1}) +  \sum_{t=2}^{T}J(\pistar) - J(\pi\ind{t}) \\ 
    &\leq 1 + \sum_{t=2}^T \sqrt{\frac{8  \log(6 \abs{\cF} \abs{\cW} T H / \delta)}{K(t-1)}} +  \sum_{t=2}^T  \sum_{h=1}^H \En_{(x_h, a_h) \sim d\ind{t}_h} \brk*{\Dft(x_h, a_h)},  \numberthis \label{eq:ub_appendix1} 
  \end{align*} 
where the second inequality uses that \(J(\pistar) \leq 1\) and  \(J(\pi\ind{1}) \geq 0\). 

We next bound the expected Bellman error terms that appear in the right-hand-side above. Consider any $t\geq{}2$ and $h\in\brk{H}$, and note that via a straightforward change of measure, 
\begin{align*} 
\En_{{d\ind{t}_h}}\brk*{\Dft(x_h, a_h)} &= \En_{\dbar\ind{t+1}_h}\brk*{{\Dft(x_h, a_h)} \cdot \frac{{d\ind{t}_h}(x_h,a_h)}{\dbar\ind{t+1}_h(x_h,a_h)}} 
\end{align*} 
Since \(u \leq \min\crl{u, v} + u \indic\crl{u \geq v}\) for any \(u, v\), we further decompose as 
\begin{align*} 
\En_{{d\ind{t}_h}}\brk*{\Dft(x_h, a_h)} &\leq \underbrace{\En_{\dbar\ind{t+1}_h} \brk*{\Dft(x_h, a_h) \cdot \min\crl*{\frac{{d\ind{t}_h}(x_h,a_h)}{\dbar\ind{t+1}_h(x_h,a_h)}, \gamma\ind{t}} }}_{\text{(A): Expected clipped Bellman error}} \\ 
&\hspace{0.5in} +  \underbrace{\En_{d\ind{t}_h}\brk*{\indic \crl*{\frac{{d\ind{t}_h}(x_h,a_h)}{\dbar\ind{t+1}_h(x_h,a_h)} \geq  \gamma\ind{t}} }}_{\text{(B): clipping violation}}, 
\end{align*} 
where in the second term we have changed the measure back to \(d\ind{t}_h\) and used that \(\abs{\Dft(x_h, a_h)} \leq 1\). We bound the terms \(\text{(A)}\) and \(\text{(B)}\) separately below. 

\paragraph{Bound on expected clipped Bellman error} Let  $w_h\ind{t}(x_h, a_h) \in \cW$ denote a weight function which satisfies 
\begin{align*}
\sup_{\pi} \nrm*{\clip{\frac{d\ind{t}_h}{\dbar_h\ind{t+1}}}{\gamma\ind{t}} - \clip{w\ind{t}_h}{\gamma\ind{t}}}_{1, d^\pi_h} &\leq \xi_t,  \numberthis \label{eq:ub_appendix2}  
\end{align*} 
which is guaranteed to exist by the definition of \(\xi_t\). 
Then, we have
\begin{align*}
\text{(A)} &= \En_{\dbar\ind{t+1}_h} \brk*{\Dft(x_h, a_h) \cdot \clip{\frac{{d\ind{t}_h}(x_h,a_h)}{\dbar\ind{t+1}_h(x_h,a_h)}}{\gamma\ind{t}}} \\  
&\leq  \En_{\dbar\ind{t+1}_h} \brk*{\Dft(x_h, a_h) \cdot \clip{w_h\ind{t}(x_h, a_h)}{\gamma\ind{t}}} + \nrm*{\clip{\frac{d\ind{t}_h}{\dbar_h\ind{t+1}}}{\gamma\ind{t}} - \clip{w\ind{t}_h}{\gamma\ind{t}}}_{1, \dbar\ind{t+1}_h} \\ 
&\leq  \En_{\dbar\ind{t+1}_h} \brk*{\Dft(x_h, a_h) \cdot \clip{w_h\ind{t}(x_h, a_h)}{\gamma\ind{t}}} + \xi_t,   
\end{align*} 
where the second line uses that \(\abs{\Dft(x_h, a_h)} \leq 1\), and the last line plugs in \pref{eq:ub_appendix2}. Next, using \pref{lem:concentration2}-(b) in the above inequality, we get that 
\begin{align*}
\text{(A)}&\leq \frac{40}{\gamma\ind{t}} \En_{\dbar\ind{t+1}_h} \brk*{\prn*{\clip{w_h\ind{t}(x_h, a_h)}{\gamma\ind{t}}}^2}  + \frac{7 \beta\ind{t}}{9} + \frac{\gamma \ind{t}}{160 t^2 } + \xi_t. \numberthis \label{eq:ub_appendix3} 
\end{align*} 

Further splitting the first term, and using that  \((a + b)^2 \leq 2a^2 + 2b^2\) , we have that 
\begin{align*}
\hspace{0.5in} &\hspace{-0.5in} \En_{\dbar\ind{t+1}_h} \brk*{\prn*{\clip{w_h\ind{t}(x_h, a_h)}{ \gamma \ind{t}}}^2} \\ 
&\leq 2 \En_{\dbar\ind{t+1}_h} \brk*{\prn*{\clip{\frac{{d\ind{t}_h}(x_h,a_h)}{\dbar\ind{t+1}_h(x_h,a_h)}}{\gamma \ind{t}}}^2} + 2 \nrm*{\clip{\frac{{d\ind{t}_h}}{\dbar\ind{t+1}_h}}{\gamma \ind{t}} - \clip{w\ind{t}_h}{\gamma\ind{t}}}^2_{2, \dbar\ind{t+1}_h}  \\  
&\leq 2 \En_{\dbar\ind{t+1}_h} \brk*{\prn*{\clip{\frac{{d\ind{t}_h}(x_h,a_h)}{\dbar\ind{t+1}_h(x_h,a_h)}}{\gamma \ind{t}}}^2} + 2 \gamma \ind{t}\nrm*{\clip{\frac{{d\ind{t}_h}}{\dbar\ind{t+1}_h}}{\gamma \ind{t}} - \clip{w\ind{t}_h}{\gamma\ind{t}}}_{1, \dbar\ind{t+1}_h}  \\  
&\leq 2 \En_{\dbar\ind{t+1}_h} \brk*{\prn*{\clip{\frac{{d\ind{t}_h}(x_h,a_h)}{\dbar\ind{t+1}_h(x_h,a_h)}}{\gamma \ind{t}}}^2} + 2  \gamma \ind{t} \xi_t,  
\end{align*}
where the second line holds since \(\nrm{w}^2_{2, d} \leq \nrm{w}_\infty \nrm{w}_{1, d}\) and \(\clip{w}{\gamma \ind{t}} \leq \gamma\ind{t}\) for any \(w\), and the last line is due to \pref{eq:ub_appendix2}. Using the above bound in \pref{eq:ub_appendix3}, we get that 
\begin{align*}
\text{(A)}&\leq 
\frac{80}{\gamma\ind{t}} \En_{\dbar\ind{t+1}_h} \brk*{\prn*{\min\crl*{\frac{{d\ind{t}_h}(x_h,a_h)}{\dbar\ind{t+1}_h(x_h,a_h)}, \gamma \ind{t}}}^2}   + 4 \xi_t  + 10 \beta \ind{t} + \frac{\gamma \ind{t}}{80 t^2} \\
  &\leq \frac{80}{\gamma\ind{t}} \En_{d\ind{t}_h} \brk*{ \frac{{d\ind{t}_h}(x_h,a_h)}{\dbar\ind{t+1}_h(x_h,a_h)}}    + 4 \xi_t  + 10 \beta \ind{t} + \frac{\gamma \ind{t}}{80 t^2} \\
 &= \frac{80}{\gamma} \En_{d\ind{t}_h} \brk*{ \frac{{d\ind{t}_h}(x_h,a_h)}{\dtil\ind{t+1}_h(x_h,a_h)}}     + 4 \xi_t + 10 \beta \ind{t} + \frac{\gamma}{80 t}, 
\end{align*} 
where the second line simply follows from a change of measure, and the last line holds since \(\gamma \ind{t} = \gamma  t\), and \(\dtil\ind{t+1} = t \dbar\ind{t + 1}\). 

\paragraph{Bound on clipping violation} Since \(\indic\crl{u \geq v} \leq \tfrac{u}{v}\) for any \(u, v \geq 0\), we get that 
\begin{align*}
\text{(B)} 
&\leq \frac{1}{\gamma \ind{t}} \En_{{d\ind{t}_h}}\brk*{\frac{{d\ind{t}_h}(x_h,a_h)}{\dbar\ind{t+1}_h(x_h,a_h)}} =  \frac{1}{\gamma} \En_{{d\ind{t}_h}}\brk*{\frac{{d\ind{t}_h}(x_h,a_h)}{\dtil\ind{t+1}_h(x_h,a_h)}}, 
\end{align*} 
where the last line holds since \(\gamma \ind{t} = \gamma t\). 

Combining the bounds on the terms \((\text{A})\) and \((\text{B})\) above, and summing over the rounds \(t = 2, \dots, T\), we get  
\begin{align*}
\sum_{t=2}^T \En_{{d\ind{t}_h}}\brk*{\Dft(x_h, a_h)}  &\leq \frac{81}{\gamma} \sum_{t=2}^T \En_{d\ind{t}_h} \brk*{ \frac{{d\ind{t}_h}(x_h,a_h)}{\dtil\ind{t+1}_h(x_h,a_h)}}    
 + 10 \sum_{t=2}^T \beta \ind{t} +  4 \sum_{t=2}^T \xi_t  + \frac{\gamma}{80}, \numberthis \label{eq:ub_appendix4} 
\end{align*} 
For the first term, using \pref{lem:potential_coverability}, along with the bound \(\nrm[\big]{\nicefrac{d\ind{t}_h}{\mustar_h}}_\infty \leq \Ccov\), we get that 
\begin{align*}
 \sum_{t=2}^T \En_{d\ind{t}_h} \brk*{ \frac{{d\ind{t}_h}(x_h,a_h)}{\dtil\ind{t+1}_h(x_h,a_h)}} &\leq 5 \Ccov \log(1+ T). 
\end{align*} 

For the second term, we have 
\begin{align*} 
\sum_{t=2}^T  \beta\ind{t} &= \sum_{t=2}^T \frac{36 \gamma t \log(6 \abs*{\cF}\abs*{\cW}HT\delta^{-1})}{K(t-1)}  \leq \frac{72 \gamma T \log(6 \abs*{\cF}\abs*{\cW}HT\delta^{-1})}{K}.  
\end{align*}

Combining these bounds, we get that 
\begin{align*} 
\sum_{t=2}^T \En_{{d\ind{t}_h}}\brk*{\Dft(x_h, a_h)}  &= O\prn*{\frac{\Ccov \log(1+ T)}{\gamma} 
  + \frac{\gamma T \log(\abs*{\cF}\abs*{\cW}HT\delta^{-1})}{K} + \sum_{t=1}^T \xi_t + \gamma \log(T)}, \numberthis \label{eq:ub_appendix5} 
\end{align*}
Plugging this bound in to \pref{eq:ub_appendix1} for each \(h \in [H]\) gives the desired result. 

\end{proof} 

\subsection{Proof of \pref{thm:online_main_faster}} \label{app:online_main_faster}

In this section, we prove a generalization of \pref{thm:online_main_faster} that accounts for misspecification error when the class \(\cW\) can only approximately realize the density ratios of mixed policies. Formally, we make the following assumption on the class \(\cW\): 
\begin{asmmod}{ass:W_realizability_proper}{$^\dag$}[Density ratio
  realizability, mixture version, with misspecification error]
  \label{ass:approximate_realizability} 
Let $T$ be the parameter to \Alg (\cref{alg:main_alg}).  For all $h \in [H]$, $\pi \in \Pi$, $t \in [T]$, and $\pi\ind{1:t}=(\pi\ind{1},\ldots,\pi\ind{t}) \in
\Pi$, there exists a weight function \(w_h^{\pi;\pi\ind{1:t}}(x,a) \in \cW_h\) such that 
\begin{align*}
\sup_{\wt \pi\in\Pi} \nrm*{ \frac{d^\pi_h}{d^{\pi\ind{1:t}}_h} - w_h^{\pi;\pi\ind{1:t}}}_{1, d_h^{\wt \pi}} \leq \epsapp. 
\end{align*}
\end{asmmod} 
Note that setting \(\epsapp = 0\) above recovers \pref{ass:W_realizability_mixture} given in the main body. 

\begin{thmmod}{thm:online_main_faster}{$'$} 
\label{thm:online_main_faster_approximate} 
Let \(\epsilon > 0\) be given, and suppose that   \cref{ass:Q_realizability} holds. Further, suppose that \pref{ass:approximate_realizability} (above) holds with \(\epsapp \leq \nicefrac{\epsilon}{18H}\). Then, \Alg, when executed on classes \(\cF\) and \(\cW\) with hyperparameters 
   \(T = \wt  \Theta \prn[\big]{\prn{\nicefrac{H^2 \Ccov}{\epsilon^2}} \cdot \log\prn{\nicefrac{\abs*{\cF}\abs*{\cW}}{\delta}} }\), \(K = 1\), and \(\gamma = \sqrt{\nicefrac{\Ccov}{(T \log\prn{\nicefrac{\abs*{\cF}\abs*{\cW}}{\delta})}}}\) returns an \(\epsilon\)-suboptimal policy \(\wh \pi\) with probability at least \(1 - \delta\) after collecting  
\begin{align} 
  \label{eq:online_main_faster_pac_app} 
N = \wt O\prn[\bigg]{\frac{H^2 \Ccov}{\epsilon^2} \log\prn*{\nicefrac{\abs*{\cF}\abs*{\cW}}{\delta}}}  
\end{align}
trajectories. In addition, for any $T\in\bbN$, with the same choice for $K$ and
$\gamma$ as above, the algorithm enjoys a regret bound of the form 
\begin{align}
  \label{eq:online_main_faster_regret_app}
\Reg \ldef{}  \sum_{t=1}^T J(\pistar) - J(\pi \ind{t}) &= \wt  O \prn[\big]{H \sqrt{\Ccov T \log\prn*{\nicefrac{\abs*{\cF}\abs*{\cW}}{\delta}}} + H T \epsapp}. 
\end{align}
\end{thmmod}

Clearly, setting the misspecification error \(\epsapp = 0\) above recovers \pref{thm:online_main_faster}. 

\begin{proof}[Proof of \pref{thm:online_main_faster_approximate}]
  First note that by combining  \pref{ass:approximate_realizability}
  with the fact that \(\clip{z}{\gamma}\) is \(1\)-Lipschitz for any
  \(\gamma > 0\), we have that for any \(h \in [H]\) and \(t \in
  [T]\), there exists a weight function \(w_h\ind{t} \in \cW_h\) such that 
\begin{align*}
\sup_{\pi\in\Pi} \nrm*{\clip{\frac{d\ind{t}_h}{\dbar_h\ind{t+1}}}{\gamma\ind{t}} - \clip{w\ind{t}_h}{\gamma\ind{t}}}_{1, d^\pi_h} &\leq \epsapp.  \numberthis \label{eq:faster_misspec_1} 
\end{align*} 
	
Using this misspecification bound, and setting \(K = 1\) in
\pref{lem:cumulative_suboptimality}, we get that with probability at
least $1-\delta$,
\begin{align*} 
   \Reg= \sum_{t=1}^{T}J(\pistar) - J(\pi\ind{t}) 
        = O\prn*{ \frac{H \Ccov \log(1+ T)}{\gamma} 
 + \gamma H T \log(6 \abs*{\cF}\abs*{\cW}HT\delta^{-1})  +  H T \epsapp}. 
\end{align*}

Further setting \(\gamma = \sqrt{\nicefrac{\Ccov}{(T \log(6 \abs*{\cF}\abs*{\cW}HT\delta^{-1}))}}\) implies that 
\begin{align*} 
  \Reg &\leq O\prn*{H \sqrt{\Ccov T \log(T) \log(6 \abs*{\cF}\abs*{\cW}HT\delta^{-1})} + H T \epsapp}. 
\end{align*} 

For the sample complexity bound, note that 
the returned policy $\wh\pi$ is chosen via \(\wh \pi \sim \unif(\crl{\pi\ind{1}, \dots, \pi\ind{T}})\), and thus 
\begin{align*}
\En \brk*{J(\pistar) - J(\wh \pi)} =  \frac{1}{T} \sum_{t=1}^{T}J(\pistar) - J(\pi\ind{t}) 
    \leq O\prn*{H \sqrt{\frac{\Ccov}{T} \log(T) \log(6 \abs*{\cF}\abs*{\cW}HT\delta^{-1})} +  H \epsapp}. 
\end{align*}

Hence, when \(\epsapp \leq O\prn*{\nicefrac{\epsilon}{H}}\), setting  \(T = \wt{\Theta}\prn*{\frac{H^2 \Ccov}{\epsilon^2} \log(6 \abs*{\cF}\abs*{\cW}HT\delta^{-1})}\) implies that the returned policy \(\wh \pi\) satisfies 
\begin{align*} 
\En \brk*{J(\pistar) - J(\wh \pi)} \leq \epsilon. 
\end{align*}
The total number of trajectories collected to return an \(\epsilon\)-suboptimal policy is given by
\begin{align*}
T \cdot K \leq \wt O\prn*{\frac{H^2 \Ccov}{\epsilon^2} \log(6 \abs*{\cF}\abs*{\cW}H\delta^{-1})}. 
\end{align*} 
\end{proof}

\subsection{Proof of \pref{thm:online_main_basic}}
\label{app:online_main_basic}
In this section, we prove a generalization of
\pref{thm:online_main_basic} in 
\cref{thm:online_main_basic_approximate} (below) which accounts for misspecification error when the class \(\cW\) can only approximately realize the density ratios of pure policies. Formally, we make the following assumption on the class \(\cW\).

\begin{asmmod}{ass:W_realizability_proper}{$^\ddag$}[Density ratio
  realizability, with misspecification error]
\label{ass:approximate_realizability_pure} 
For any policy pair \(\pi_1, \pi_2 \in \Pi\) and \(h \in [H]\), there
exists some weight function \(w\ind{\pi_1, \pi_2}_h \in \cW_h\) such that 
\begin{align*}
\sup_{\pi} \nrm*{\frac{d_h^{\pi_1}}{d_h^{\pi_2}} - w\ind{\pi_1, \pi_2}_h}_{1, d^\pi_h} \leq  \epsapp. \end{align*} 
\end{asmmod} 

Setting \(\epsapp = 0\) above recovers \pref{ass:W_realizability_proper} given in the main body. 

Note that \pref{ass:approximate_realizability_pure} only states that
density ratios of pure policies are approximately realized by
\(\cW_h\). On the other hand, the proof of
\pref{lem:cumulative_suboptimality}, our key tool in sample complexity
analysis, requires (approximate) realizability for the ratio
\(\nicefrac{d\ind{t}}{\dbar\ind{t+1}}\) in \(\cW_h\), which involves a
mixture of occupancies. We fix this problem by running \Alg on a larger class \(\bcW\) that is constructed using \(\cW\) and has small misspecification error for \(\nicefrac{d\ind{t}}{\dbar\ind{t+1}}\) for all \(t \leq T\). Before delving into the proof of \pref{thm:online_main_basic_approximate}, we first describe the class \(\bcW\). 

\paragraph{{Construction of the class \(\protect\widebar{\cW}\)}}  Define an operator  \(\Mixture\) that takes in a sequence of weight functions \(\crl{w\ind{1}, \dots, w\ind{t}}\) and a parameter \(t \leq T\),  and outputs a function \([\Mixture(w\ind{1}, \dots, w\ind{t}; t)]\) such that for any \(x, a \in \cX \times \cA\),  \begin{align*} 
[\Mixture(w\ind{1}, \dots, w\ind{t}; t)]_h(x, a) &\ldef \frac{1}{\En_{s \sim \unif([t])} \brk*{w_h\ind{s}(x, a)}}.  
\end{align*} 
                                                                                                                                                                                                                   Using the operator \(\Mixture\), we define \(\bcW\ind{t}\) via
\begin{align*}
\bcW\ind{t} &=  \crl*{ \Mixture(w\ind{1}, \dots, w\ind{t}; t)  \mid w\ind{1}, \dots, w\ind{t} \in \cW}, 
\end{align*} 
and then define
\begin{align*}
\bcW = \cup_{t \leq T}  \bcW\ind{t}.  \numberthis \label{eq:Wbar_def} 
\end{align*}
As a result of this construction, we have that
\begin{align*}
\abs{\bcW} &\leq \prn{\abs{\cW} + 1}^T \leq (2\abs{\cW})^T.  \numberthis \label{eq:bcw_bound} 
\end{align*}
In addition, we define $\wb{\cW}_h=\crl*{w_h\mid{}w\in\bcW}$. The following lemma shows that \(\bcW\) has small misspecification error for density ratios of mixture policies.

\begin{lemma} 
\label{lem:approximation_transfer_lemma} Let \(t \geq 0\) be given, and suppose \pref{ass:approximate_realizability_pure} holds.  
For any sequence of policies \(\pi\ind{1}, \dots, \pi\ind{t} \in
\Pi\), and \(h \in [H]\), there exists a weight function \(\barw_h \in \bcW_h\ind{t}\) such that for any \(\gamma > 0\), 
\begin{align*}
\sup_{\pi} \nrm*{\clip{\frac{d\ind{t}_h}{\dbar_h\ind{t+1}}}{\gamma} - \clip{\barw_h}{\gamma}}_{1, d_h^\pi} &\leq \gamma^2 \epsapp, 
\end{align*} 
where recall that $\dbar_h\ind{t + 1} = \frac{1}{t}\sum_{s=1}^t d_h\ind{s}$. 
\end{lemma}

The following theorem, which is our main sample complexity bound under misspecification error, is obtained by running \Alg on the weight function class \(\bcW\). 

\begin{thmmod}{thm:online_main_basic}{$'$}  
\label{thm:online_main_basic_approximate}
Let \(\epsilon > 0\) be given, and suppose
that \cref{ass:Q_realizability} holds. Further, suppose that
\pref{ass:approximate_realizability_pure} holds with \(\epsapp \leq
\wt O \prn{\nicefrac{\epsilon^5}{\Ccov^3 H^5}}\). Then, \Alg, when executed on classes \(\cF\) and \(\bcW\) (defined in \eqref{eq:Wbar_def}) with hyperparameters  \(T = \wt  \Theta\prn{\nicefrac{H^2 \Ccov}{\epsilon^2}}\), \(K = \wt  \Theta\prn{T  \log\prn{\nicefrac{\abs*{\cF}\abs*{\cW}}{\delta}}}\), and \(\gamma = \sqrt{\nicefrac{\Ccov}{T}}\) 
returns an \(\epsilon\)-suboptimal policy \(\wh \pi\) with probability at least \(1 - \delta\) after collecting 
\begin{align*}
N = \wt  O\prn[\Big]{\frac{H^4 \Ccov^2}{\epsilon^4}  \log\prn*{\nicefrac{\abs*{\cF}\abs*{\cW}}{\delta}}}. 
\end{align*} 
trajectories. 
\end{thmmod}
Setting the misspecification error \(\epsapp = 0\) above recovers \pref{thm:online_main_basic} in the main body. 

\begin{proof}[{Proof of \pref{thm:online_main_basic_approximate}}] 
Using the misspecification bound from 
\pref{lem:approximation_transfer_lemma} in
\pref{lem:cumulative_suboptimality} implies that with probability at
least $1-\delta$,
  \begin{align*} 
    \sum_{t=1}^{T}J(\pistar) - J(\pi\ind{t}) 
    &= O\prn*{\frac{H \Ccov \log(1+ T)}{\gamma} 
 + \frac{\gamma H T \log(6 \abs*{\cF}\abs*{\bcW}HT\delta^{-1})}{K} +  H \gamma^2 T^3 \epsapp  + \gamma H \log(T)}. 
\end{align*}  

Using the relation in \eqref{eq:bcw_bound}, we get that \(\abs{\bcW} \leq (2\abs{\cW})^T\) and thus  
  \begin{align*} 
    \sum_{t=1}^{T}J(\pistar) - J(\pi\ind{t}) 
    &= O\prn*{\frac{H \Ccov \log(1+ T)}{\gamma} 
 + \frac{\gamma H T \log(6 \abs*{\cF}\abs*{\bcW}HT\delta^{-1})}{K} +  H \gamma^2 T^3 \epsapp  + \gamma H \log(T)}.  
\end{align*} 

Setting \(K = 2T  \log(6 \abs*{\cF}\abs*{\cW}HT\delta^{-1})\) and \(\gamma = \sqrt{\tfrac{\Ccov}{T}}\) in the above bound, we get 
\begin{align*}  
    \sum_{t=1}^{T}J(\pistar) - J(\pi\ind{t}) 
    &\leq O\prn*{\prn*{H \sqrt{\Ccov T} \log(T) + H \Ccov T^2 \epsapp}}. 
\end{align*} 

Finally, observing that the returned policy \(\wh \pi \sim \unif(\crl{\pi\ind{1}, \dots, \pi\ind{T}})\), we get   
\begin{align*}  
    \sum_{t=1}^{T}J(\pistar) - J(\pi\ind{t}) 
    &\leq O\prn*{\prn*{H \sqrt{\frac{\Ccov}{T}} \log(T) +  H \Ccov T^2 \epsapp}}. 
\end{align*} 

Thus, when \(\epsapp \leq \wt O \prn{\nicefrac{\epsilon^5}{\Ccov^3 H^5}}\), setting \(T = \widetilde{\Theta}\prn*{\frac{\Ccov}{\epsilon^2} \log^2\prn{\frac{\Ccov}{\epsilon^2}}}\) in the above bound implies that  
\begin{align*}
\En \brk*{J(\pistar) - J(\wh \pi)} \leq \epsilon. 
\end{align*}
The total number of trajectories collected to return \(\epsilon\)-suboptimal policy is given by: 
\begin{align*}
T \cdot K = O\prn*{\frac{H^4 \Ccov^2}{\epsilon^4} \log(\abs*{\cF}\abs*{\cW}HT\delta^{-1})}. 
\end{align*} 
\end{proof}

\begin{proof}[Proof of \pref{lem:approximation_transfer_lemma}] Fix any \(h \in [H]\) and \(t \in [T]\). Using \pref{ass:approximate_realizability_pure}, we have that for any \(s \leq t\), there exists a function \(w_h \ind{s, t} \in \cW_h\) such that  
\begin{align*} 
\sup_{\pi\in\Pi} \nrm*{\frac{d_h\ind{s}}{d_h\ind{t}} - w\ind{s, t}_h}_{1,  d^\pi_h} \leq  \epsapp. \numberthis \label{eq:approximation1} 
\end{align*} 

Let \(\barw_h = \brk[\big]{\Mixture(w\ind{1, t}, \dots, w\ind{t, t}; t)}_h\in\bcW_h\), and recall that \(\dbar\ind{t + 1} = \En_{s \sim \unif([t])} \brk*{d\ind{s}_h}\). For any  \(x, a \in \cX \times \cA\), define  
\begin{align*} 
\zeta(x, a) \ldef{} \abs*{\min\crl*{\frac{1}{\En_{s \sim \unif([t])}  \brk*{\nicefrac{d\ind{s}_h(x, a)}{d\ind{t}_h(x, a)}}}, \gamma} - \min\crl*{\frac{1}{\En_{s \sim \unif([t])}  \brk*{w \ind{s, t}(x, a)}}, \gamma}}. 
\end{align*} 

Using that the function $g (z) = \min\crl*{\frac{1}{z}, \gamma} = \frac{1}{\max\crl{z, \gamma^{-1}}} $ is \(\gamma^2\)-Lipschitz, we get that 
\begin{align*}
\zeta(x, a) &\leq \gamma^2 \cdot  \abs*{\En_{s \sim \unif([t])}  \brk*{\frac{d\ind{s}_h(x, a)}{d\ind{t}_h(x, a)}} - \En_{s \sim \unif([t])}  \brk*{w_h\ind{s, t}(x, a)}} \\ 
&\leq \gamma^2 \En_{s \sim \unif([t])}  \brk*{ \abs*{\frac{d\ind{s}_h(x, a)}{d\ind{t}_h(x, a)} - w_h\ind{s, t}(x, a)}},  
\end{align*} 
where the last inequality follows from the linearity of expectation, and by using Jensen's inequality. 

Thus, 
\begin{align*} 
\hspace{1.8in}&\hspace{-1.8in} \sup_{\pi\in\Pi}  \En_{x_h, a_h \sim d^\pi_h} \brk*{\abs*{\min\crl*{\frac{d\ind{t}_h(x_h, a_h)}{\dbar_h\ind{t+1}(x_h, a_h)}, \gamma} - \min\crl*{\barw_h(x_h, a_h), \gamma}}} \\ 
&= \sup_{\pi\in\Pi}  \En_{x_h, a_h \sim d^\pi_h}[\zeta(x_h, a_h)]  \\
&\leq \gamma^2 \sup_{\pi\in\Pi}  \En_{s \sim \unif([t])}  \brk*{\nrm*{\frac{d_h\ind{s}}{d_h\ind{t}} - w_h \ind{s, t}}_{1, d^\pi_h}} \\ 
&\leq \gamma^2   \En_{s \sim \unif([t])}  \sup_{\pi\in\Pi}\brk*{\nrm*{\frac{d_h\ind{s}}{d_h\ind{t}} - w_h \ind{s, t}}_{1, d^\pi_h}} \\  
&\leq \gamma^2 \epsapp, 
\end{align*} 
where the second-to-last line is due to Jensen's inequality, and the last line follows from the bound in \pref{eq:approximation1}. 
\end{proof}  

\clearpage

\clearpage

\section{Proofs and Additional Results from \creftitle{sec:hybrid} (Hybrid RL)}
\label{app:hybrid-proofs}
This section is organized as follows:
\begin{itemize}
\item \cref{app:hybrid_examples} gives additional details and further
  examples of offline algorithms with which we can apply \HtO,
  including an example that uses \textsc{FQI} as the base algorithm.
  
\iclr{\item \cref{sec:comparison-offline-rl} includes a detailed comparison between \HyAlg and purely offline algorithms based on single-policy  concentrability.
	\item \cref{sec:generic-offline-online} includes a discussion on generic reductions from offline to online RL using a class of algorithms we call \emph{optimistic offline RL algorithms}.
}
  \item \cref{app:hto} contains the proof of the main result for \HtO,
    \cref{thm:htoregret}.
  \item \cref{app:hybrid_offline} contains supporting proofs for the
    offline RL algorithms (\cref{app:hybrid_examples}) we use within \HtO.
\end{itemize}

\subsection{Examples for \HybridAlg}
\label{app:hybrid_examples}
This section contains additional examples of base algorithms that can be applied within the \HtO reduction.

\subsubsection{\HyAlg Algorithm}\label{app:hyglow}

For completeness, we state full pseudocode for the \HyAlg algorithm described in \cref{sec:applying-reduction} as \cref{alg:hyglow}. As described in the main body, this algorithm simply invokes \HtO with a clipped and regularized variant of the \textsc{Mabo} algorithm (\Crmabo, \eqref{eq:clip-mabo}) as the base algorithm.

We will invoke \Crmabo with a certain augmented weight function class $\widebar{\cW}$, which we now define. For any $w \in \cW$, let $
w\ind{h} \in (\cX \times \cA \times [H] \rightarrow \bbR \cup \{+\infty\})$
 be defined by 
 \begin{equation}\label{eq:w-ind-h}
 w\ind{h}_{h'}(x,a) \coloneqq \begin{cases} 
 		w_{h}(x,a) & \text{if } h=h'\\
 		0 & \text{if } h \neq h'
	 \end{cases},
\end{equation}
and let $\cW\ind{h} = \{ w\ind{h}\mid w \in \cW\}$. Define the set
	 \begin{equation}\label{eq:augmented-w}
	 	\widebar{\cW}
\coloneqq \cW \cup (-\cW) \cup_{h \in [H]} \prn*{\cW\ind{h} \cup (-\cW\ind{h})}.
	 \end{equation}
Note that the size satisfies $\widebar{\cW}$ is $\abs*{\widebar{\cW}} \leq 2(H+1)|\cW| \leq 4H|\cW|$. 

We recall that the \Crmabo algorithm with dataset $\cD$ and parameters $\gamma, \cF, \cW$ is defined via
\begin{equation}
	\hspace{-0.09in} {\small \wh{f} \in \argmin_{f \in \cF} \max_{w \in \cW} \sum_{h=1}^H \left| \wh{\bbE}_{\cD_h} \left[\check{w}_h(x_h,a_h)\whDf(x_h, a_h, r_h, x'_{h+1})\right]\right| - \alpha\ind{n} \wh{\bbE}_{\cD_h}\left[\check{w}^2_h(x_h,a_h) \right]},
      \end{equation}
where $\alpha\ind{n}\ldef{}\nicefrac{8}{\gamma\ind{n}}$ and $\check{w}_h \ldef{}\clip{w_h}{\gamma\ind{n}}$.

In \cref{app:maboofline} we will prove the following result.

\begin{restatable}[\Crmabo is \Ccbounded]{theorem}{mabooffline}\label{thm:mabooffline}
Let $\cD=\crl{\cD_h}_{h=1}^{H}$ consist of $H\cdot{}n$ samples from $\mu\ind{1}, \ldots, \mu\ind{n}$. %
	For any $\gamma \in \bbR_+$, the {\normalfont \Crmabo} algorithm (\eqref{eq:clip-mabo}) with parameters $\cF$, augmented class $\widebar{\cW}$ defined in \eqref{eq:augmented-w} in \cref{app:hyglow}, and $\gamma$ is \Ccbounded at scale $\gamma$ under the $\Cond$ that $Q^\star \in \cF$ and that for all $\pi \in \Pi$ and $h \in [H]$, $\nicefrac{d^{\pi}_h}{\mu\ind{1:n}_h} \in \cW$.  

\end{restatable}

The risk bound for \HyAlg will follow by \cref{thm:htoregret}. Namely, we have the following.

\hyglow*

\iclr{
\begin{algorithm} 
  \caption{\HyAlg: \HybridAlg $+$ \Crmabo} 
    \label{alg:hyglow} 
\begin{algorithmic}[1]  
\Statex[0] \textbf{input:} Parameter $T \in \bbN$, value function class \(\cF\), weight
function class \(\cW\), parameter  \(\gamma \in \brk{0,1}\), offline datasets $\cD_{\mathrm{off}} = \{\cD_{\mathrm{off},h}\}_h$ each of size $T$. 
\State Set \(\gamma\ind{t} = \gamma\cdot t\), \(\alpha\ind{t} = \nicefrac{4}{\gamma\ind{t}}\) and \(\beta\ind{t} = \prn{\nicefrac{12 \gamma}{K}} \cdot \log\prn*{\nicefrac{6 \abs{\cF} \abs{\cW} T H }{ \delta}}\). \ascomment{Constants are difference from GLOW algorithm!}
\State Initialize \(\cD\ind{1}_{\mathrm{on},h} = \cD\ind{1}_{\mathrm{hybrid},h} = \emptyset\) for all \(h \in [H]\).
\For{$t=1,\ldots,T$}
\State \multiline{Compute value function \(f\ind{t}\) such that 
  {\small \begin{equation}
      \label{eq:hyglow_alg}
\hspace{-0.2in} f\ind{t}\in \argmin_{f \in \cF} \max_{w \in \cW} \sum_{h=1}^H \left| \wh{\bbE}_{\cD\ind{t}_{\mathrm{hybrid},h}} \left[\check w_h(x_h,a_h)\whDf(x_h, a_h, r_h, x'_{h+1})\right]\right| - \alpha\ind{n} \wh{\bbE}_{\cD\ind{t}_{\mathrm{hybrid},h}}\left[\check w^2_h(x_h,a_h) \right], 
\end{equation}}
where \(\check w \ldef \clip{w}{\gamma\ind{t}}\) and \(\whDf(x, a, r, x') \ldef{} f_h(x,a)- r - \max_{a'}f_{h+1}(x', a')\).} 
\State Compute policy \(\pi\ind{t} \leftarrow \pi_{f\ind{t}}\).
\State \mbox{Collect \arxiv{trajectory }$(x_1,a_1,r_1),\ldots, (x_H,a_H,r_H)$ using \(\pi \ind{t}\); \(\cD\ind{t+1}_{\mathrm{on},h} \coloneq \cD\ind{t}_{\mathrm{on},h} \cup \crl{(x_h, a_h, r_h, x_{h+1})}\).}
\State Aggregate offline and online data: $\cD\ind{t+1}_{\mathrm{hybrid},h}\coloneqq \restr{\cD_{\mathrm{off},h}}{1:t} \cup \cD\ind{t+1}_{\mathrm{on},h}$ for all $h \in [H]$.  \label{line:data-agg} 
\EndFor
\State \textbf{output:} $\pihat=\unif(\pi\ind{1}, \dots, \pi\ind{T})$.\hfill 
   \end{algorithmic} 
\end{algorithm} 

}
\subsubsection{Computationally efficient implementation for \HyAlg}\label{app:comp-eff-hyglow}
In the following, we expand on the discussion after \pref{thm:mabooffline}  regarding computationally efficient implementation of the optimization problem in \eqref{eq:clip-mabo} and \eqref{eq:hyglow_alg} via reparameterization. Let \(\gamma > 0\) and \(n > 0\) be given to the learner, and suppose that in addition to the class \(\cW\), the learner has access to a function class \(\cW\ind{\gamma, n}\) that satisfies the following assumption. 
\begin{assumption}
\label{ass:reparameterized_class}   The function class \(\cW\ind{\gamma, n}\) satisfies 
\begin{enumerate}[label=\((\alph*)\)]
\item For all \(w \in \cW\ind{\gamma, n}\) and \(h \in [H]\), \(\nrm{w_h}_\infty \leq \gamma n\). 
\item  For all \(h \in [H]\), \(\{\pm\clip{w}{\gamma n} \mid w \in \cW\} \subseteq \cW\ind{\gamma, n}.\)
\item For all \(h \in [H]\), \(\{\pm\clip{w\ind{h}}{\gamma n} \mid w \in \cW\} \subseteq \cW\ind{\gamma, n},\) where $w\ind{h}$ is defined as in \eqref{eq:w-ind-h}.
\end{enumerate}\end{assumption}
Note that \(\cW\ind{\gamma, n}\) also satisfies the density ratio realizability required by \Crmabo (cf. \cref{thm:mabooffline}).
We claim that optimizing directly over the class \(\cW\ind{\gamma, n}\), which does not involve explicitly clipping, leads to the same guarantee as solving \eqref{eq:clip-mabo}. In more detail, consider the following offline RL algorithm, which given offline datasets \(\crl{\cD_h}_{h \leq H}\), returns 
\begin{equation}\label{eq:clip-mabo-efficient} 
	\hspace{-0.09in} {\small \wh{f} \in \argmin_{f \in \cF} \max_{w \in \cW\ind{\gamma, n}} \sum_{h=1}^H \wh{\bbE}_{\cD_h} \left[w_h(x_h,a_h)\whDf(x_h, a_h, r_h, x'_{h+1})\right] - \alpha\ind{n} \wh{\bbE}_{\cD_h}\left[w^2_h(x_h,a_h) \right]}.
      \end{equation}
     Using the next lemma, we show that the function obtained by solving \eqref{eq:clip-mabo-efficient} leads to the same offline RL guarantee as \Crmabo when \(\abs{\log(\cW\ind{\gamma, n})} = O \prn*{\log(\abs{\cW}) + \poly(H)}\). In particular, substituting the bound from  \pref{lem:mabo_concentration-efficient}  in place of the corresponding bound from \pref{lem:mabo_concentration} in the proof of \pref{thm:mabooffline}, while keeping rest of the analysis same, shows that the above described computationally efficient implementation of \Crmabo is also \Ccbounded. Using this fact with \pref{thm:htoregret} implies the desired performance guarantee for \HyAlg (similar to \pref{cor:hyglow}). 
\begin{lemma}
\label{lem:mabo_concentration-efficient} Suppose \(\cF\) and \(\cW\) satisfy \cref{ass:Q_realizability} and \cref{ass:mabo-assumptions}. Additionally, let \(\gamma \in (0, 1)\) and \(n > 0\) be given constants, and for \(h \in [H]\), let \(\cD_h\) be datasets of size \(n\) sampled from the offline distribution \(\mu_h\ind{1:n}\). Furthermore, let \(\cW \ind{\gamma, n}\) be a reparameterized function class that satisfies \pref{ass:reparameterized_class} w.r.t.~\(\cW\).  Then, the function 
 \(\wh f\) returned by \eqref{eq:clip-mabo-efficient}, when executed with datasets \(\crl{\cD_h}_{h \leq H}\), weight function class \(\cW\ind{\gamma, n}\), and parameters \(\alpha\ind{n} = \frac{8}{\gamma n }\), satisfies with probability at least \(1 - \delta\), 
\[
  \sum_{h=1}^H \left| \En_{\mu\ind{1:n}_h}\brk*{\Dfh(x_h, a_h) \cdot \cw_h(x_h, a_h)} \right| \leq \cO\prn*{ \sum_{h=1}^H \frac{1}{\gamma n} \En_{\mu\ind{1:n}_h} \brk*{\prn*{\cw_h(x_h,a_h)}^2}  
  + H^2 \beta\ind{n}}, 
\]
for all $w \in \cW$, where \(\beta\ind{n}  = O\prn*{ \frac{\gamma n}{n - 1}\log(24|\cF||\cW|H^2/\delta)}\).  
\end{lemma}

\begin{proof}[\pfref{lem:mabo_concentration-efficient}]
 Repeating the same arguments as in the proof of \pref{lem:concentration2} with \(K=1\) (we avoid repeating the arguments for conciseness), along with the fact that \(\nrm{w'_h}_\infty \leq \gamma n\) for all \(h \in [H]\) and \(w' \in \cW\ind{\gamma, n}\), we get that the returned function \(\wh f\) satisfies for all \(w' \in \cW\ind{\gamma, n}\), 
 \[
  \sum_{h=1}^H  \En_{\mu\ind{1:n}_h}\brk*{\Dfh(x_h, a_h) \cdot w'_h(x_h, a_h)} \leq \cO\prn*{ \sum_{h=1}^H \frac{1}{\gamma n} \En_{\mu\ind{1:n}_h} \brk*{\prn*{w'_h(x_h,a_h)}^2}  
  +  \frac{H\gamma n}{n - 1}\log(|\cF||\cW\ind{\gamma, n}|H/\delta)}.  
\]

However, as in the proof of \cref{lem:mabo_concentration}, since for any  $w \in \cW$ we have that both \(\clip{w\ind{h}}{\gamma n} \in \cW\ind{\gamma, n}_h\) and \(-\clip{w\ind{h}}{\gamma n} \in \cW\ind{\gamma, n}_h\) for every \(h \in [H]\), the above inequality immediately implies that for every \(w \in \cW\), 
\[
  \sum_{h=1}^H \left| \En_{\mu\ind{1:n}_h}\brk*{\Dfh(x_h, a_h) \cdot \cw_h(x_h, a_h)} \right| \leq \cO\prn*{ \sum_{h=1}^H \frac{1}{\gamma n} \En_{\mu\ind{1:n}_h} \brk*{\prn*{\cw_h(x_h,a_h)}^2}  
  + \frac{H^2\gamma n}{n - 1}\log(|\cF||\cW\ind{\gamma, n}|H/\delta)}, 
\]
where we used that the RHS is independent of the sign. The final statement follows by plugging in that \(\abs{\log(\cW\ind{\gamma, n})} = O \prn*{\log(\abs{\cW}) + \poly(H)}\). 
\end{proof}

\subsubsection{Fitted Q-Iteration}

In this section, we apply \HtO with the Fitted Q-Iteration (\Fqi) algorithm \citep{munos2007performance,munos2008finite,chen2019information} as the base algorithm. For an offline dataset $\cD$ and value function class $\cF$, the \Fqi algorithm is defined as follows: 
\begin{algoshort}[Fitted Q-Iteration (\Fqi)]\label{alg:fqi}\,
\begin{enumerate}
\item Set $\wh f_{H+1}(x,a) = 0$ for all $(x,a)$ 
\item For $h = H, \ldots, 1$: 
\[
	\widehat{f}_h \in \argmin_{f_h \in \cF_h} \widehat{\bbE}_{\cD_h} \brk*{(f_h(x_h, a_h) - r_h - \max_{a'} \widehat{f}_{h+1}(x_{h+1},a')))^2}.
\]
\item Output $\widehat{\pi} = \pi_{\widehat{f}}$.
\end{enumerate}
\end{algoshort}
We analyze \Fqi under the following standard Bellman completeness assumption.
\begin{assumption}[Bellman completeness]\label{ass:completeness} For all \(h \in [H]\), we have that 
\iclr{$\cT_h \cF_{h+1} \subseteq \cF_h$}
\arxiv{
\[
\cT_h \cF_{h+1} \subseteq \cF_h
\]
}
\end{assumption}

\begin{restatable}{theorem}{fqioffline}\label{thm:fqi-offline}
The \Fqi algorithm is \Ccbounded under \cref{ass:completeness} with scaling functions $\ascale = \frac{6}{\gamma}$ and $\bscale = 1024\log(2n|\cF|)\gamma $, for all $\gamma>0$ simultaneously. As a consequence, when invoked within \HtO, we have $\Risk \leq \widetilde{O}\prn[\big]{H \sqrt{\nicefrac{(\Cstar + \Ccov) \log(|\cF|\delta^{-1})}{T}}}$ with probability at least $1- \delta T$ .
\end{restatable}
The proof can be found in \cref{app:fqi-proof}.

The full pseudocode for \HtO with \Fqi is essentially identical to the Hy-Q algorithm of \cite{song2022hybrid}, except for a slightly different data aggregattion strategy in Line \ref{line:data-agg}. Thus, Hy-Q can be interpreted as a special case of the \HtO algorithm when instantiated with \Fqi as a base algorithm. The risk bounds in \citet{song2022hybrid} are proven under an a structural condition known as bilinear rank \citep{du2021bilinear}, which is complementary to coverability. Our result recovers a special case of a risk bound for Hy-Q given in the follow-up work of \citet{liu2023provable}, which analyzes Hy-Q under coverability instead of bilinear rank.

\subsubsection{Model-Based MLE} 

In this section we apply $\HtO$ with Model-Based Maximum Likelihood Estimation (MLE) algorithm as the base algorithm. The algorithm is parameterized by a model class $\cM = \{ \cM_h \}_{h=1}^H$, where $\cM_h \subset \{M_h: \cX \times \cA \rightarrow \Delta(\bbR \times \cX)\}$.  Each model $M = \{M_h\}_{h=1}^H \in\cM$ has the same state space, action space, initial distribution, and horizon, and each $M_h \in \cM_h$ is a conditional distribution over rewards and next states for layer $h$. %
For a dataset $\cD$, the algorithm proceeds as follows.

\begin{algoshort}[Model-Based MLE]\,
\begin{itemize}
 \item For $h \in [H]$:
	\begin{itemize}
	 \item Compute the maximum likelihood estimator for layer $h$ as
			\begin{equation}\label{eq:model-mle}
				\widehat{M}_h = \argmax_{M_h \in \cM_h} \sum_{(x_h,a_h,r_h,x_{h+1}) \in \cD_h} \log\prn*{M_h(r_h,x_{h+1} \mid x_h,a_h)}
			\end{equation}
	\end{itemize}
 \item Output $\pi^\star_{\widehat{M}}$, the optimal policy for $\widehat{M} = \{\widehat{M}_h\}_h$
\end{itemize}
\end{algoshort}
We analyze Model-Based MLE under a standard realizability assumption.
\begin{assumption}[Model realizability]\label{ass:model-realizability}
We have that $M^\star \in \cM$. 
\end{assumption}

\begin{restatable}{theorem}{modelbased}\label{thm:model-based}
The model-based MLE algorithm is \Ccbounded under \cref{ass:model-realizability} for all $\gamma > 0$ simultaneously, with scaling functions $\ascale = \frac{6}{\gamma}$ and $\bscale = 8 \log\prn*{|\cM|H/\delta} \gamma$. As a consequence, when invoked within $\HtO$, we have $\Risk \leq \widetilde{O}\prn[\big]{H \sqrt{\nicefrac{(\Cstar + \Ccov) \log(|\cM|H\delta^{-1})}{T}}}$ with probability at least $ 1- \delta T$.
\end{restatable}

The proof can be found in \cref{app:model-based-proof}.

\iclr{
\subsection{Comparison to Offline RL}\label{sec:comparison-offline-rl}
In this section, we compare the performance of \HyAlg to existing results in purely offline RL which assume access to a data distribution $\nu$ with single-policy concentrability (\cref{ass:pi-star-concentrability})

Let $\mu$ be the offline data distribution. Let us write $w^\pi \coloneqq \nicefrac{d^{\pi}}{\mu}$, $w^\star \coloneqq \nicefrac{d^{\pi^\star}}{\mu}$ and $V^\star$ for the optimal value function \ascomment{Subscript of \(h\) here; What is \(\pi^\star\) here? Where is it defined?}. The most relevant work is the \textsc{PRO-RL} algorithm of \citeauthor{zhan2022offline} \ascomment{Better to cite as \citet{zhan2022offline}}. Their algorithm is computationally efficient and establishes a polynomial sample complexity bound under the realizability of certain \textit{regularized versions} of $w^\star$ and $V^\star$. By contrast, our result requires $Q^\star$-realizability and the density ratio realizability of $w^\pi$ \ascomment{Again, subscript h} for all $\pi \in \Pi$, but for the unregularized problem. These assumptions are not comparable, as either may hold without the other.  Their sample complexity result is also slightly larger, scaling roughly as $\wt O \Big( \frac{H^6 \Cstar^4 C_{\star,\varepsilon}^2}{\varepsilon^6} \Big)$, where $C_{\star,\varepsilon}$ is the single-policy concentrability for the regularized problem, as opposed to our $\widetilde{O}\prn[\big]{\nicefrac{H^2(\Cstar + \Ccov)}{\veps^2}}$. However, our approach requires additional online access while their algorithm does not\arxiv{ (the base offline algorithm for \HyAlg, Equation \cref{eq:clip-mabo}, would likely need to scale with the \textit{all-policy} concentrability coefficient if used purely offline)}.

To the best of our knowledge, all other algorithms for the purely offline setting that only require single-policy concentrability either need stronger representation conditions (such as value-function completeness \citet{xie2021bellman}), or are not known to be computationally efficient in the general function approximation setting due to the need for implementing pessimism (e.g., \cite{chen2022offline}).
}

\iclr{
  \subsection{Generic Reductions from Online to Offline RL?}\label{sec:generic-offline-online}
Our hybrid-to-offline reduction
\HybridAlg and the $\mathsf{CC}$-boundedness definition also shed light on the question of when offline RL methods can be lifted to the \textit{purely online setting}. Indeed, observe that any offline algorithm which satisfies $\mathsf{CC}$-boundedness (\cref{ass:offline-risk-cc}) with only a $\pihat$-coverage term, namely which satisfies an offline risk bound of the form
    \begin{equation}\label{eq:offline_online_risk}
\RiskOff \leq \sum_{h=1}^H \frac{\ascale}{n} \En_{\wh \pi \sim p}\brk*{\Ccc{n}{h}{\wh{\pi}}{\mu\ind{1:n}}{\gamma n}}+ \bscale,
    \end{equation}
    can be repeatedly invoked within $\HybridAlg$ (with $\cD_{\mathrm{off}} = \emptyset$) to achieve a small $\sqrt{\nicefrac{\Ccov}{T}}$-type risk bound for the purely online setting, with no hybrid data. This can be seen immediately by inspecting our proof for the hybrid setting (\cref{thm:htoregret-general} and \cref{thm:htoregret}).
     
   We can think of algorithms satisfying \eqref{eq:offline_online_risk} as \textit{optimistic offline RL algorithms}, since their risk only scales with a term depending on their own output policy; this is typically achieved using optimism. In particular, it is easy to see, that \Alg and \textsc{Golf} \citep{jin2021bellman,xie2022role} can be interpreted as repeatedly invoking such an optimistic offline RL algorithm within the \HybridAlg reduction. This class of algorithms has not been considered in the offline RL literature since they inherit both the computational drawbacks of pessimistic algorithms and the statistical drawbacks of ``neutral'' (i.e. non-pessimistic) algorithms (at least, when viewed only in the context of offline RL). %

   In more detail, as with pessimism, optimism is often not computationally efficient, although it furthermore requires all-policy concentrability (as opposed to single-policy concentrability) to obtain low \textit{offline} risk. On the other hand, neutral (non-pessimistic) algorithms such as \textsc{Fqi} \citep{chen2019information} and \textsc{Mabo} \citep{xie2020q} also require all-policy concentrability, but are more computationally efficient. %
    		However, our reduction shows that these algorithms might merit further investigation. In particular, it uncovers that they can automatically solve the online setting (without hybrid data) under coverability and when repeatedly invoked on datasets generated from their previous policies. We find that this reduction advances the fundamental understanding of sample-efficient algorithms in both the online and offline settings, and are optimistic that this understanding can be used for future algorithm design. \loose
              }

\subsection{Proofs for \HtO (\creftitle{thm:htoregret})}\label{app:hto} 

The following theorem is a slight generalization of
\cref{thm:htoregret}. In the sequel, we prove \cref{thm:htoregret} as
a consequence of this result.

\begin{theorem}\label{thm:htoregret-general}
Let $T\in\bbN$ be given, let $\cD_{\mathrm{off}}$ consist of $H \cdot T$ samples from data distribution $\nu$. Let $\algo$ be \Ccbounded at scale $\gamma \in [0,1]$ under $\Cond(\cdot)$, with parameters $\ascale$ and $\bscale$. Suppose that \arxiv{for all} $\iclr{\forall{}\;}t \in [T]$ and $\pi\ind{1},\ldots,\pi\ind{t} \in \Pi$, $\Cond(\mu\ind{t},M^\star)$ holds for $\mu\ind{t} \ldef \{ \nicefrac{1}{2}(\nu_h + \nicefrac{1}{t}\textstyle\sum_{i=1}^{t}d^{\pi\ind{i}}_h)\}_{h=1}^H$.  
Then, with probability at least $1-\delta{}T$, the risk of \HtO (\cref{alg:offline-to-hybrid}) with \arxiv{inputs }$T$, $\algo$, and $\cD_\mathrm{off}$ is bounded as\iclr{\footnote{We define risk for the hybrid setting as in \eqref{eq:risk}. Our result is stated as a bound on the risk to the optimal policy $\pistar$, but extends to give a bound on the risk of any comparator $\pi$ with $\Cstar$ replaced by coverage for $\pi$.}}

\begin{equation}\label{eq:htorisk-general}
	\Risk \leq \frac{2 \ascale}{T}\sum_{h=1}^H \sum_{t=1}^T \frac{1}{t} \Ccc{t}{h}{\pi^\star}{\nu}{\gamma\ind{t}} + \underbrace{\wt O\prn*{H\prn*{\frac{\Ccov\ascale}{N} + \bscale}}}_{\rdef\piterm}. 
      \end{equation}
\end{theorem}

\begin{proof}[\pfref{thm:htoregret-general}]
Recall the definitions $d\ind{t}_h \coloneqq d^{\pi\ind{t}}_h$, $\widetilde{d}\ind{t}_h \coloneqq \sum_{s=1}^{t-1} d\ind{t}_h$, and $\bar{d}\ind{t}_h \coloneqq \frac{1}{t-1}\widetilde{d}\ind{t}_h$. Furthermore, let \(d^\star_h \coloneqq d^{\pistar}_h\) for all \(h \leq H\). Note that the data distribution for $\cD\ind{t}_{\mathrm{hybrid}}$ is  
$\mu\ind{t} = \{\mu\ind{t}_h\}_{h=1}^H$ where $\mu\ind{t}_h = \frac{1}{2}(\nu_h + \bar{d}\ind{t}_h)$, where $\nu_h$ is the offline distribution. As a result of  \cref{ass:offline-risk-cc}, the offline algorithm $\algo$ invoked on the dataset $\cD\ind{t}_{\mathrm{hybrid}}$ %
	outputs a distribution $p_t \sim \Delta(\Pi)$ that satisfies the bound: 
\begin{align}\label{eq:boing}
 \bbE_{\pi\ind{t} \sim p_t}\left[J(\pi^\star) - J(\pi\ind{t})\right] 
&\leq \sum_{h=1}^H \frac{\ascale}{t} \left(\Ccc{t}{h}{\pi^\star}{\mu\ind{t}}{\gamma t} + \En_{\pi\ind{t} \sim p_t}\brk{\Ccc{n}{h}{\pi\ind{t}}{\mu\ind{t}}{\gamma t}}\right)+ \bscale,
\end{align}
with probability at least $1 - \delta$, where $\ascale$ and $\bscale$ are the scaling functions for which $\algo$ is \Ccbounded at scale $\gamma$. By taking a union bound over $T$, the number of iterations, we have that the event in \eqref{eq:boing} occurs for all $t \leq T$ with probability greater than $1 - \delta T$. 

Plugging in the definition for the clipped concentrability
coefficient above and summing over \(t\) from \(1, \dots, T\), we get that 
\begin{align*}
\Reg &= \bbE\brk*{\sum_{t=1}^T J(\pi^\star) - J(\pi\ind{t})} \\ 
&\leq \sum_{h=1}^H \Bigg(\underbrace{\sum_{t=1}^T \frac{\ascale}{t}\left\|\clip{\frac{d^{\pi^\star}_h}{\mu\ind{t}_h}}{\gamma t}\right\|_{1,d^\star_h}}_{\mathrm{(I)}} + \underbrace{\sum_{t=1}^T \frac{\ascale}{t} \En_{\pi\ind{t} \sim p_t}\brk*{\left\|\clip{\frac{d\ind{t}_h}{\mu\ind{t}_h}}{\gamma t}\right\|_{1,d\ind{t}_h}}}_{\mathrm{(II)}} \Bigg) + HT\bscale.
\end{align*} 
For each \(h \in [H]\), we bound the two terms $\mathrm{I}$ and
$\mathrm{II}$ separately below. 

\paragraph{Term (I)} 
Note that $\mu\ind{t}_h(x, a) \geq \nu_h(x, a)/2$ for any \(x, a\). Thus, 

\[
\mathrm{(I)} \leq 2 \ascale \sum_{t=1}^T \frac{1}{t} \left\|\clip{\frac{d^{\pi^\star}_h}{\nu_h}}{\gamma t}\right\|_{1,d^{\pi^\star}_h} = 2 \ascale \sum_{t=1}^T \frac{1}{t} \Ccc{t}{h}{\pi^\star}{\nu}{\gamma t}.  %
\]

\paragraph{Term (II)} 
We bound this term uniformly for any $\pi\ind{t} \sim p_t$. So, fix $\pi\ind{t}$ and note that 
\begin{align*}
\text{(II)} %
 &= \ascale \sum_{t=1}^T \frac{1}{t} \En_{d\ind{t}_h} \left[\min\crl*{\frac{d\ind{t}_h(x, a)}{\mu\ind{t}_h(x, a)},\gamma t}\right] \\ 
 &= \ascale \Bigg(\underbrace{\sum_{t=1}^T \frac{1}{t} \En_{d\ind{t}_h}\left[\frac{d\ind{t}(x, a)}{\mu\ind{t}_h(x, a)}\indic\left\{\frac{d\ind{t}(x, a)}{\mu\ind{t}_h(x, a)} \leq\gamma  t\right\}\right]}_{\mathrm{(II.A)}} + \underbrace{\sum_{t=1}^T \frac{1}{t} \En_{d\ind{t}_h}\left[\gamma t \cdot \indic\left\{\frac{d\ind{t}(x, a)}{\mu\ind{t}_h(x, a)} > \gamma t\right\}\right]}_{\mathrm{(II.B)}} \Bigg), 
\end{align*}
where the second line holds since $\min\{u,v\} = u \indic\crl*{u \leq v} + v \indic\crl*{v < u}$ for all \(u, v \in \bbR\). 

 In order to bound the two terms appearing above, we use certain
 properties of coverability, similar to the analysis of \Alg
(\cref{app:online}). \dfedit{For a parameter $\lambda\in(0,1)$,} let us define a burn-in time
\begin{equation}\label{eq:burn-in}
  \tau\ind{\lambda}_h(x,a) = \min\crl*{t\mid{}\dtil\ind{t}_h(x,a)\geq \frac{\Ccov\cdot{}\mu^\star_h(x,a)}{\lambda}},
 \end{equation}
 and observe that
\begin{equation}\label{eq:t-less-tau}
  \sum_{t=1}^T \En_{d\ind{t}_h}[ \indic\crl*{t < \tau\ind{\lambda}_h(x,a)}] = \sum_{x,a} \sum_{t < \tau\ind{\lambda}_h(x,a)} d\ind{t}_h(x,a) \leq \frac{2\Ccov}{\lambda}, 
\end{equation}
which holds for any $\lambda \in (0,1)$. This bound can be derived by noting that
\begin{align*}
\sum_{x,a} \sum_{t < \tau\ind{\lambda}_h(x,a)} d\ind{t}_h(x,a) &= \sum_{x,a} \tilde{d}\ind{\tau\ind{\lambda}_h(x,a)}(x,a) \\
&= \sum_{x,a} \tilde{d}\ind{\tau\ind{\lambda}_h(x,a)-1}(x,a) + \sum_{x,a}d\ind{\tau\ind{\lambda}_h(x,a)}(x,a) \\
&\leq \sum_{x,a} \frac{\Ccov}{\lambda} \mu^\star_h(x,a) + \sum_{x,a} \Ccov \mu^\star_h(x,a) \\
&\leq \Ccov\left( \frac{1}{\lambda} + 1\right) \leq \frac{2\Ccov}{\lambda}.
\end{align*}

 We also recall the follow bound, which is a corollary of the
 elliptical potential lemma (\cref{lem:elliptical-potential}):
 \begin{equation}\label{eq:elliptical-hybrid}
   \sum_{t=1}^T  \sum_{x,a} d\ind{t}_h(x,a)
   \frac{d\ind{t}_h(x,a)}{\widetilde{d}\ind{t}(x,a) } \indic\crl*{t >
     \tau\ind{\lambda}_h(x,a)}
   \leq 
\sum_{t=1}^T  \sum_{x,a} d\ind{t}_h(x,a) \frac{d\ind{t}_h(x,a)}{\widetilde{d}\ind{t}(x,a) } \indic\crl*{t > \tau\ind{1}_h(x,a)} \stackrel{\mathrm{(i)}}{\leq} 5\log(T) \Ccov.
\end{equation}
The inequality $\mathrm{(i)}$ can be seen derived by noting that, under the event in the indicator, we have $\wt d \ind{t}_h (x,a) \geq \Ccov \mustar_h(x,a)$ and thus $\wt d\ind{t}_h(x,a) \geq \frac{1}{2} (\Ccov \mustar_h(x,a) + \wt d\ind{t}(x,a))$. This gives
\[
 \sum_{t=1}^T  \sum_{x,a} d\ind{t}_h(x,a) \frac{d\ind{t}_h(x,a)}{\widetilde{d}\ind{t}(x,a) } \indic\crl*{t > \tau\ind{1}_h(x,a)} \leq  2 \sum_{t=1}^T  \sum_{x,a} d\ind{t}_h(x,a) \frac{d\ind{t}_h(x,a)}{\widetilde{d}\ind{t}(x,a) + \Ccov \mustar_h(x,a)},
\]
from which we can repeat the steps from \eqref{eq:cov-pot-one} to \eqref{eq:cov-pot-two}.

\paragraph{Term (II.A)} 
To bound this term, we introduce a split according to the burn-in time  \(\tau\ind{1}_h(x,a)\), i.e. 
\begin{align*}
\mathrm{(II.A)} = \sum_{t=1}^T \frac{1}{t} \En_{d\ind{t}_h}\left[\frac{d\ind{t}_h(x, a)}{\mu\ind{t}_h(x, a)}\indic\crl*{\frac{d\ind{t}_h(x, a)}{\mu\ind{t}_h(x, a)} \leq\gamma  t}\left(\indic\crl*{t \leq \tau\ind{1}_h(x,a)} + \indic\crl*{t > \tau\ind{1}_h(x,a)}\right) \right].
\end{align*}
The first term is bounded via
\begin{align*}
\sum_{t=1}^T \frac{1}{t} \En_{d\ind{t}_h}\left[\frac{d\ind{t}_h(x, a)}{\mu\ind{t}_h(x, a)}\indic\crl*{\frac{d\ind{t}_h(x, a)}{\mu\ind{t}_h(x, a)} \leq\gamma  t}\indic\crl*{t \leq \tau\ind{1}_h(x,a)}\right] &\leq \gamma \sum_{t=1}^T \En_{d\ind{t}_h}\left[\indic\crl*{t \leq \tau\ind{1}_h(x,a)}\right]\\
&\leq 2 \gamma \Ccov,
\end{align*}
by \eqref{eq:t-less-tau} with $\lambda=1$. The second term is bounded via:
\begin{align*}
\sum_{t=1}^T \frac{1}{t} \En_{d\ind{t}_h}\left[\frac{d\ind{t}_h(x, a)}{\mu\ind{t}_h(x, a)}\indic\crl*{\frac{d\ind{t}(x, a)}{\mu\ind{t}_h(x, a)} \leq \gamma  t}\indic\crl*{t > \tau\ind{1}_h(x,a)}\right] &\leq 2 \sum_{t=1}^T \frac{1}{t} \En_{d\ind{t}_h}\left[\frac{d\ind{t}_h(x, a)}{\bar{d}\ind{t}_h(x, a)} \indic\crl*{t > \tau\ind{1}_h(x,a)}\right] \\
&\leq 2 \sum_{t=1}^T \En_{d\ind{t}_h}\left[\frac{d\ind{t}_h(x, a)}{\widetilde{d}\ind{t}_h(x, a)} \indic\crl*{t > \tau\ind{1}_h(x,a)}\right] \\
&\leq 10 \log(T) \Ccov,
\end{align*}
by using that $\mu\ind{t}_h(x, a) \geq \frac{\bar{d}\ind{t}_h(x,
  a)}{2}$ and Equation \cref{eq:elliptical-hybrid}. %

Adding these two terms together gives us the upper bound $\mathrm{(II.A)} \leq 2\gamma\Ccov + 10\log(T)\Ccov$.   

\paragraph{Term (II.B)}

We have 
\begin{align*}
\sum_{t=1}^T \frac{1}{t}\En_{d\ind{t}_h}\left[\gamma t\indic\crl*{\frac{d\ind{t}_h(x, a)}{\mu\ind{t}_h(x, a)} > \gamma t}\right]&=  \gamma \sum_{t=1}^T \En_{d\ind{t}_h}\left[\indic\crl*{\frac{d\ind{t}_h(x, a)}{\mu\ind{t}_h(x, a)} > \gamma t}\right] \\
&\stackrel{\mathrm{(i)}}{\leq} \gamma \sum_{t=1}^T  \En_{d\ind{t}_h}\left[\indic\crl*{\frac{\Ccov \mu^\star_h(x, a)}{\widetilde{d}\ind{t}_h(x, a)} > \frac{\gamma}{2} }\right] \\ 
&\stackrel{\mathrm{(ii)}}{=} \gamma \sum_{t=1}^T  \En_{d\ind{t}_h}\left[\indic\crl*{t \leq \tau\ind{\gamma/2}_h(x,a) }\right] \\
&\leq \gamma \cdot \frac{4 \Ccov}{\gamma} = 4 \Ccov,
\end{align*}
where the inequality $\mathrm{(i)}$ follows from applying the upper bounds $d\ind{t}_h(x,a) \leq \Ccov \mustar_h(x,a)$ and $\mu\ind{t}_h(x,a) \geq \frac{1}{2} \dbar\ind{t}_h(x,a)$, and the inequality $\mathrm{(ii)}$ follows from the definition of the burn-in time (\eqref{eq:burn-in}) with $\lambda = \nicefrac{\gamma}{2}$.

Combining all the bounds above, we get that  
\begin{align*}
\mathrm{\textbf{(II)}} &\leq 2\ascale\left(2 \gamma \Ccov + 10\log(T)\Ccov + 2\Ccov \right) \\
&= 4\ascale \Ccov \left(\gamma + 10 \log(T) + 1\right).
\end{align*}

Adding together the terms so far, we can conclude the regret bound: 
\[
	\Reg \leq 2 \ascale \sum_{h=1}^H \sum_{t=1}^T \frac{1}{t} \Ccc{t}{h}{\pi^\star}{\nu}{\gamma t} + H \left(4\ascale \Ccov \left(\gamma + 10 \log(T) + 1\right) + T \bscale\right).
\]
It follows that the policy $\widehat{\pi} = \unif(\pi\ind{1},\ldots,\pi\ind{T})$ satisfies the risk bound
\begin{align*}
	\Risk &\leq \frac{2 \ascale}{T}\sum_{h=1}^H \sum_{t=1}^T \frac{1}{t} \Ccc{t}{h}{\pi^\star}{\nu}{\gamma t} + \underbrace{H \left(4\frac{\ascale \Ccov}{T} \left(\gamma + 10 \log(T) + 1\right) + \bscale\right)}_{\coloneqq \piterm}\\
	 &= \frac{2 \ascale}{T}\sum_{h=1}^H \sum_{t=1}^T \frac{1}{t} \Ccc{t}{h}{\pi^\star}{\nu}{\gamma t} + \wt O\prn*{H\prn*{\frac{\Ccov\ascale}{T} + \bscale}},
\end{align*}
with probability at least $1 - \delta T$, where in the last line we have used that $\gamma \in [0,1]$.
\end{proof}

We now prove \cref{thm:htoregret} as a consequence of \cref{thm:htoregret-general}.

\htoregret*

\begin{proof}[\pfref{thm:htoregret}]
Under the assumptions in the theorem statement,
\cref{thm:htoregret-general} implies that
\[
	\Risk \leq \frac{2 \ascale}{T}\sum_{h=1}^H \sum_{t=1}^T \frac{1}{t} \Ccc{t}{h}{\pi^\star}{\nu}{\gamma t} + H \left(4\frac{\ascale \Ccov}{T} \left(\gamma + 10 \log(T) + 1\right) + \bscale\right). 
\]
Since $\max_h \max_{x,a} \left|\frac{d^{\pi^\star}_h(x, a)}{\nu_h(x, a)}\right| \leq \Cstar$, we have $\Ccc{t}{h}{\pi^\star}{\nu}{\gamma t} \leq \Cstar$, so we can simplify the first term above to
\[
\frac{2 \ascale}{T}\sum_{h=1}^H \sum_{t=1}^T \frac{1}{t} \Ccc{t}{h}{\pi^\star}{\nu}{\gamma t} \leq \frac{2 \ascale \Cstar}{T} H \sum_{t=1}^T \frac{1}{t} \leq \frac{6 \ascale \Cstar}{T} H \log(T),
\]
using the bound on the harmonic number $\sum_{t=1}^T \nicefrac{1}{t} \leq 3
\log(T)$. Combining with the remainder of the risk bound in
\cref{thm:htoregret-general} gives us
\[
	\Risk \leq \frac{H\ascale}{T}\prn*{6 \Cstar \log(T) + 4\Ccov(\gamma + 10\log(T) +1)}  + H \bscale = \wt{O}\left(H\left(\frac{\ascale(\Cstar + \Ccov)}{T} + \bscale \right) \right),
\]
where we have used the fact that $\gamma \in [0,1]$.
\end{proof}

\begin{corollary}\label{cor:alphabetagamma}
Let $T \in \bbN$ and $\cD_{\mathrm{off}}$ consist of $H \cdot T$ samples from data distribution $\nu$. Suppose that for all $t \in [T]$ and $\pi\ind{1}, \ldots, \pi\ind{t} \in \Pi$, $\Cond(\mu\ind{t},M^\star)$ holds for $\mu\ind{t} \coloneqq \{\nicefrac{1}{2}(\nu_h + \nicefrac{1}{t} \sum_{i=1}^t d^{\pi\ind{i}}_h)\}_{h=1}^H$. Let $\algo$ be \Ccbounded at scale $\gamma \in [0,1]$ under $\Cond(\cdot)$ and with parameters $\ascale = \frac{\fa}{\gamma}$ and $\bscale = \fb \gamma$. Consider the \HtO algorithm with inputs $T$, $\algo$, and $\cD_{\mathrm{off}}$. Then,
\begin{itemize}

\item If $\algo$ is \Ccbounded at scale $\gamma = \widetilde{\Theta}\left(\sqrt{\nicefrac{\fa \Ccov}{\fb T}} \right)$ and $T$ is such that $\gamma \in [0,1]$, we have
\[
	\Risk \leq 2\sqrt{\frac{\fa \fb}{T \Ccov}}\sum_{h=1}^H \sum_{t=1}^T \frac{1}{t} \Ccc{t}{h}{\pi^\star}{\nu}{\gamma\ind{t}} + \widetilde{O}\prn*{H\sqrt{\frac{\Ccov{}\fa\fb}{T}}}.
\]
	with probability greater than $1- \delta T$.

\item If $\nu$ satisfies $\Cstar$-single-policy concentrability and $\algo$ is \Ccbounded at scale $\gamma = \widetilde{\Theta}\left(\sqrt{\nicefrac{\fa(\Cstar + \Ccov)}{\fb T}}\right)$ and $T$ is such that $\gamma \in [0,1]$, then
\[
\Risk \leq \widetilde{O}\left(H \sqrt{\nicefrac{(\Cstar + \Ccov)\fa\fb}{T}}\right),
\]
	with probability greater than $1 - \delta T$.
	
\end{itemize}
\end{corollary}

\begin{proof}[\pfref{cor:alphabetagamma}] 
We start with the first case. We recall the definition of $\piterm$ appearing in \cref{thm:htoregret-general}:
\[
\piterm \coloneqq H \left(\frac{4\ascale \Ccov}{T} \left(\gamma + 8 \log(T) + 1\right) + \bscale\right) = \wt O\prn*{H\prn*{\frac{\Ccov\ascale}{T} + \bscale}}.
\]
and the risk bound from \cref{thm:htoregret-general}.
\begin{equation}
	\Risk \leq \frac{2 \ascale}{T}\sum_{h=1}^H \sum_{t=1}^T \frac{1}{t} \Ccc{t}{h}{\pi^\star}{\nu}{\gamma\ind{t}} + \piterm. 
\end{equation}
Plugging in $\ascale = \frac{\fa}{\gamma}$ and $\bscale = \fb \gamma$ into $\piterm$ gives
\[
\piterm = H\left(4\frac{\fa}{\gamma T} \Ccov \left(\gamma + 8 \log(T) + 1\right) + \fb \gamma\right).
\]
The above is optimized by picking $\gamma = 2\sqrt{\frac{\fa}{\fb}\frac{\Ccov}{T}(8\log(T)+1)}$. Plugging this in gives us
\[
\piterm = H\left(4\sqrt{\frac{\fa\fb \Ccov (8\log(T) +1)}{T}} + 4 \frac{\fa \Ccov}{T}\right) = \wt O \prn*{H\sqrt{\frac{\fa\fb \Ccov}{T}}}, 
\] 
as desired.
For the second result, recall from \cref{thm:htoregret} that the risk bound when $\nu$ satisfied $\Cstar$-policy-concentrability is
\[
\Risk \leq \frac{H\ascale}{T}\prn*{6 \Cstar \log(T) + 4\Ccov(\gamma + 8\log(T) +1)}  + H \bscale.
\]
Plugging in $\ascale = \frac{\fa}{\gamma}$ and $\bscale = \fb \gamma$ gives us 
\[
\Risk \leq \frac{H\fa}{T\gamma}\prn*{6 \Cstar \log(T) + 4\Ccov(\gamma + 8\log(T) +1)}  + H \fb \gamma.
\]
This expression is optimized by $\gamma = \sqrt{\frac{\fa \left(6 \Cstar \log(T) + 4\Ccov(\gamma + 8\log(T) +1)\right)}{T \fb}}$, which when substituted gives us the risk bound
\[
	\Risk \leq 2H\sqrt{\frac{\fa \fb \left(6 \Cstar \log(T) + 4\Ccov(\gamma + 8\log(T) +1)\right)}{T}} = \wt O \prn*{H \prn*{\sqrt{\frac{(\Cstar + \Ccov)\fa\fb}{T}}}},
\]
as desired.
\end{proof}

\subsection{Proofs for \HtO Examples (\creftitle{app:hybrid_examples})}
\label{app:hybrid_offline}

\subsubsection{Supporting Technical Results} 
\begin{lemma}[{Telescoping Performance Difference (\citet[Theorem 2]{xie2020q}; \citet[Lemma 3.1]{jin2021pessimism})}]\label{lemma:telescope-pd} For any \(f \in \cF\), we have that 
\[
J(\pi^\star) - J(\pi_f) \leq \sum_{h=1}^H \bbE_{d^{\pi^\star}_h}[\cT_h f_{h+1} (x_h, a_h) - f_h(x_h, a_h)] + \bbE_{d^{\pi_f}_h}[f_h(x_h, a_h) - \cT_h f_{h+1}(x_h, a_h)]. 
\]
\end{lemma}
This bound follows from a straightforward adaptation of the proof of  \citet[Theorem 2]{xie2020q} to the finite horizon setting. %

\begin{lemma} \label{lemma:clip-decomp} 
For all policy $\pi \in \Pi$, value function $f \in \cF$, timestep $h \in [H]$, data distribution $\mu = \{\mu_h\}_{h=1}^H$ where $\mu_h \in \Delta(\cX \times \cA)$, and $\gamma \in \bbR_+$, we have 
\[
  \En_{d^\pi_h}[\Df(x_h, a_h)] \leq \En_{\mu_h}\left[ \Df(x_h, a_h) \cdot \clip{\frac{d^{\pi}_h(x_h, a_h)}{\mu_h(x_h, a_h)}}{\gamma}\right] + 2\bbP^{\pi}\brk*{\frac{d^{\pi}_h(x_h, a_h)}{\mu_h(x_h, a_h)} > \gamma}. 
\] 
Similarly, 
\[
  \En_{d^\pi_h}[- \Df(x_h, a_h)] \leq \En_{\mu_h}\left[(- \Df(x_h, a_h)) \cdot  \clip{\frac{d^{\pi}_h(x_h, a_h)}{\mu_h(x_h, a_h)}}{\gamma}\right] + 2\bbP^{\pi}\brk*{\frac{d^{\pi}_h(x_h, a_h)}{\mu_h(x_h, a_h)} > \gamma}, 
\]
where recall that \(\Df(x, a) \ldef{} f_h(x, a) - [\cT_h f_{h+1}](x, a)\). 
\end{lemma}
\begin{proof}[\pfref{lemma:clip-decomp}] In the following, we prove the first inequality. The second inequality follows similarly. Using that \(\abs{\Df(x, a)} \leq 1\) for any \(x, a \in \cX \times \cA\), we have 
\begin{align*}
  \En_{d^\pi_h}[\Df(x_h, a_h)] &\leq \En_{d^\pi_h}[ \Df(x_h, a_h) \cdot \indic\crl*{ \mu_h(x_h, a_h) \neq 0}] + \En_{d^\pi_h}[ \indic\crl*{ \mu_h(x_h, a_h) = 0}].
  \end{align*}
For the second term, for any \(\gamma > 0\), 
\begin{align*}
\En_{d^\pi_h}[ \indic\crl*{ \mu_h(x_h, a_h) = 0}] &\leq \En_{d^\pi_h}\brk*{\indic\crl*{\frac{d^{\pi}_h(x_h, a_h)}{\mu_h(x_h, a_h)} > \gamma}} = \bbP^{\pi}\brk*{\frac{d^{\pi}_h(x_h, a_h)}{\mu_h(x_h, a_h)} > \gamma}. 
\end{align*}
For the first term, using that $u \leq \min\{u,v\} + u \indic\crl{u \geq v}$ for all \(u, v \geq 0\), and that \(\abs{\Df(x, a)} \leq 1\) for any \(x, a \in \cX \times \cA\), we get that  
\begin{align*}
\hspace{1in}&\hspace{-1in} \En_{d^\pi_h}\left[ \Df(x_h, a_h) \cdot \indic\crl*{ \mu_h(x_h, a_h) \neq 0}\right] \\  
  &= \En_{\mu_h}\left[ \Df(x_h, a_h) \cdot \frac{d^{\pi}_h(x_h, a_h)}{\mu_h(x_h, a_h)} \indic\crl*{ \mu(x_h, a_h) \neq 0}\right] \\
  &\leq \En_{\mu_h}\left[ \Df(x_h, a_h) \cdot \clip{\frac{d^{\pi}_h(x_h, a_h)}{\mu_h(x_h, a_h)}}{\gamma}\indic\crl*{ \mu_h(x_h, a_h) \neq 0}\right] \\ 
  &\hspace{1in} + \En_{\mu_h}\left[\frac{d^{\pi}_h(x_h, a_h)}{\mu_h(x_h, a_h)}\indic\crl*{\frac{d^{\pi}_h(x_h, a_h)}{\mu_h(x_h, a_h)} > \gamma}\indic\crl*{ \mu_h(x_h, a_h) \neq 0}\right] \\ 
  &= \En_{\mu_h}\left[ \Df(x_h, a_h) \cdot \clip{\frac{d^{\pi}_h(x_h, a_h)}{\mu_h(x_h, a_h)}}{\gamma} \indic\crl*{ \mu_h(x_h, a_h) \neq 0}\right] + \En_{d^\pi_h}\brk*{\indic\crl*{\frac{d^{\pi}_h(x_h, a_h)}{\mu_h(x_h, a_h)} > \gamma}}. 
\end{align*} 

Furthermore, also note that 
\begin{align*}
	  \En_{\mu_h}\left[ \Df(x_h, a_h) \cdot \clip{\frac{d^{\pi}_h(x_h, a_h)}{\mu_h(x_h, a_h)}}{\gamma}\indic\crl*{ \mu_h(x_h, a_h) = 0 }\right] \\
  &\hspace{-2in}= \sum_{(x_h, a_h) ~\text{s.t.}~\mu_h(x_h, a_h)=0} \mu_h(x_h, a_h) \cdot \Df(x_h, a_h) \cdot  \clip{\frac{d^{\pi}_h(x_h, a_h)}{\mu_h(x_h, a_h)}}{\gamma} = 0.
\end{align*}

The final bound follows by combining the above three terms. 
\end{proof}

\begin{lemma}\label{lem:bloop}
For any policy \(\pi\), data distribution \(\mu = \{\mu_h\}_{h=1}^H\) where $\mu_h \in \Delta(\cX \times \cA)$,  scale \(\gamma \in \bbR_+\) and horizon \(h \in [H]\), we have 
\[
\bbP^{\pi}\left[\frac{d^{\pi}_h(x_h, a_h)}{\mu_h(x_h, a_h)} > \gamma\right] \leq %
\frac{2}{\gamma} \left\| \clip{\frac{d^\pi_h}{\mu_h}}{\gamma}\right\|_{1,d^\pi_h}. 
\]
\end{lemma}
\begin{proof}[\pfref{lem:bloop}] 
Note that
\begin{align*}
\bbP^{\pi}\left[\frac{d^{\pi}_h(x_h, a_h)}{\mu_h(x_h, a_h)} > \gamma\right] &= \En_{d^\pi_h}\brk*{\indic\crl*{\frac{d^{\pi}_h(x_h, a_h)}{\mu_h(x_h, a_h)} > \gamma}} \\ 
&\leq \En_{d^\pi_h}\brk*{\indic\crl*{\frac{d^{\pi}_h(x_h, a_h)}{\mu_h(x,a) + \gamma^{-1}d^\pi_h(x_h, a_h)} > \frac{\gamma}{2}}} \\
&\leq \frac{2}{\gamma} \En_{d^\pi_h}\left[\frac{d^{\pi}_h(x_h, a_h)}{\mu_h(x_h,a_h) +  \gamma^{-1}d^\pi_h(x_h, a_h)}\right] \\
&\leq \frac{2}{\gamma} \En_{d^\pi_h}\left[\clip{\frac{d^{\pi}_h(x_h, a_h)}{\mu_h(x_h, a_h)}}{\gamma}\right] = \frac{2}{\gamma} \left\| \clip{\frac{d^\pi_h}{\mu_h}}{\gamma}\right\|_{1,d^\pi_h}, 
\end{align*}
where the first inequality follows from 
\[
\gamma^{-1} d^{\pi}_h(x,a) > \mu_h(x,a) \implies \gamma^{-1}d^{\pi}_h(x,a) > \frac{1}{2}(\mu_h(x,a) + \gamma^{-1} d^{\pi}_h(x,a)),
\]
and the second inequality follows by Markov's inequality.
\end{proof}

\begin{lemma}\label{lem:blorp}
For any policy \(\pi\), data distribution \(\mu = \{\mu_h\}_{h=1}^H\) where $\mu_h \in \Delta(\cX \times \cA)$, scale \(\gamma \in \bbR_+\), and horizon \(h \in [H]\), we have 
\[
\left\|\clip{\frac{d^\pi_h}{\mu_h}}{\gamma}\right\|_{2,\mu_h}^2 \leq 2\left\|\clip{\frac{d^\pi_h}{\mu_h}}{\gamma}\right\|_{1,d^\pi_h}.
\]
\end{lemma}
\begin{proof}[\pfref{lem:blorp}] 
Beginning with the left-hand side, we have,
\begin{align*}
\En_{\mu_h} \brk*{ \min\crl*{ \frac{d^{\pi}_h(x_h, a_h)}{\mu_h(x_h, a_h)},\gamma}^2} &\stackrel{\mathrm{(i)}}{\leq} 2 \En_{\mu_h} \brk*{  \prn*{\frac{d^{\pi}_h(x_h, a_h)}{\mu_h(x_h, a_h)}}^2 \indic\crl*{ \frac{d^{\pi}_h(x_h, a_h)}{\mu_h(x_h, a_h)} \leq \gamma}} \\ 
&\hspace{1in} + 2 \En_{\mu_h} \brk*{\gamma \cdot \frac{d^{\pi}_h(x_h, a_h)}{\mu_h(x_h, a_h)}\cdot \indic\crl*{\frac{d^{\pi}_h(x_h, a_h)}{\mu_h(x_h, a_h)} > \gamma }} \\   
&\stackrel{\mathrm{(ii)}}{\leq} 2 \En_{d^\pi_h}\brk*{ \frac{d^{\pi}_h(x_h, a_h)}{\mu_h(x_h, a_h)} \indic\crl*{ \frac{d^{\pi}_h(x_h, a_h)}{\mu_h(x_h, a_h)} \leq \gamma}}  + 2 \En_{d^\pi_h}\brk*{\gamma \cdot \indic\crl*{\frac{d^{\pi}_h(x_h, a_h)}{\mu_h(x_h, a_h)} > \gamma }}\\ 
&\stackrel{\mathrm{(iii)}}{\leq} 2 \En_{d^\pi_h}\brk*{\min\crl*{\frac{d^{\pi}_h(x_h, a_h)}{\mu_h(x_h, a_h)}, \gamma}}.
\end{align*}
\pacomment{actually I just noticed this constant can be improved to $2$ (was $4$), but not worth it to change this everywhere.} \ascomment{Not worth it!} 

In $\mathrm{(i)}$, we have used that for all $u,v \in \bbR^+$, $\min\{u,v\} \leq u \indic\crl*{u \leq v} + \sqrt{u v} \indic\crl*{v < u}$ and that $(u + v)^2 \leq 2(u^2 + v^2)$, thus that $\min\{u,v\}^2 \leq 2(u^2 \indic\{u \leq v\} + uv \indic\{v < u\})$.
In $\mathrm{(ii)}$, we have done a change of measure from $\mu_h$ to $d^\pi_h$.
In $\mathrm{(iii)}$, we have used that $\min\{u,v\} = u \indic\{u \leq v\} + v \indic\{v < u\}$.

\end{proof}

\subsubsection{Proofs for \Crmabo (Proof of \cref{thm:mabooffline})} \label{app:maboofline}

Suppose the dataset \(\cD = \crl{\cD_h}_{h \leq H}\) is sampled from the offline distribution \(\mu\ind{1:n}\). In this section, we analyze our regularized and clipped variant of \textsc{Mabo} (\eqref{eq:clip-mabo}). %
We analyze \Crmabo under the following density ratio assumption. Recall for any sequence of data distributions $\mu\ind{1}, \ldots \mu\ind{n}$, we denote the mixture distribution by $\mu\ind{1:n} = \{\mu\ind{1:n}_h\}_{h=1}^H$, defined by $\mu\ind{1:n}_h \coloneqq \frac{1}{n} \sum_{i=1}^n \mu\ind{i}_h$.
\begin{assumption}\label{ass:mabo-assumptions}
	For a given sequence of data distributions $\mu\ind{1}, \ldots \mu\ind{n}$, we have that for all $\pi \in \Pi$, and for all $h \in [H]$, 
	\[
		\frac{d^\pi_h}{\mu\ind{1:n}_h} \in \cW.
	\]
	
\end{assumption}

We first note the following bound for the hypothesis \(\wh f\) returned by \Crmabo. 

\begin{lemma}
\label{lem:mabo_concentration}
Let $\cD = \{\cD_h\}_{h=1}^H$ be a dataset consisting of $H \cdot n$ samples from $\mu\ind{1}, \ldots, \mu\ind{n}$. Suppose that $\cF$ satisfies \cref{ass:Q_realizability} and that $\cW$ satisfies \cref{ass:mabo-assumptions} with respect to $\mu\ind{1:n}$. %
Let \(\wh f\) be the function returned by \Crmabo, given in \eqref{eq:clip-mabo}, when executed on \(\crl{\cD_h}_{h \leq H}\) with parameters $\gamma, \cF,$ and the augmented weight function class $\widebar{\cW}$ defined in \eqref{eq:augmented-w}. %
Then, with probability at least \(1 - \delta\), we have 
\begin{equation}\label{eq:mabo-concentration}
  \sum_{h=1}^H \left| \En_{\mu\ind{1:n}_h}\brk*{\Dfh(x_h, a_h) \cdot \cw_h(x_h, a_h)} \right| \leq  \sum_{h=1}^H \frac{20}{\gamma\ind{n}} \En_{\mu\ind{1:n}_h} \brk*{\prn*{\cw_h(x_h,a_h)}^2}  
  + \frac{7}{18} H^2 \beta\ind{n}, 
\end{equation}
for all $w \in \cW$, where \(\beta\ind{n} \ldef{} \frac{36 \gamma\ind{n}}{n-1} \log(24 \abs{\cF} \abs{\cW} H^2 / \delta)\). 
\end{lemma}

\begin{proof}[\pfref{lem:mabo_concentration}]
 Repeating the argument of \cref{lem:concentration2} (a), we can establish that $Q^\star$ satisfies the following bound for all $h \in [H], w \in \widebar{\cW}$:
\begin{equation}\label{eq:byol}
  \wh \En_{\cD_h} \brk*{
    \prn[\big]{ 
    \brk{\wh \Delta_h Q^\star}(x_h, a_h, r_h, x'_{h+1})} \cdot 
    \check w _h(x_h,a_h)} \leq{} \wh \En_{\cD_h} \brk*{ \alpha\ind{n}\cdot \prn[\big]{ \check w _h(x_h,a_h)}^2 } +  \beta\ind{n},
\end{equation}
with probability at least $1-\delta$, where $\alpha\ind{n} = \nicefrac{8}{\gamma\ind{n}}$ and $\beta\ind{n} = \nicefrac{36 \gamma\ind{n}}{n-1} \log(6 \abs{\cF} \abs{\widebar{\cW}} H / \delta) \leq \nicefrac{36 \gamma\ind{n}}{n-1} \log(24 \abs{\cF} \abs{\cW} H^2 / \delta)$. Going forward, we condition on the event that this holds. Now, since the right-hand side of \eqref{eq:byol} is independent of the sign of $\wtil_h$, we can apply this to $-\wtil \in \widebar{\cW}$ to conclude that 
\[
\abs*{  \wh \En_{\cD_h} \brk*{
    \prn[\big]{ 
    \brk{\wh \Delta_h Q^\star}(x_h, a_h, r_h, x'_{h+1})} \cdot 
    \check w _h(x_h,a_h)}} \leq{} \wh \En_{\cD_h} \brk*{ \alpha\ind{n}\cdot \prn[\big]{ \check w _h(x_h,a_h)}^2 } +  \beta\ind{n}.
\]
  Summing over $h \in [H]$ and taking the max over $w \in \widebar{\cW}$, we can conclude that:
\[
 \max_{w \in \widebar{\cW}} \sum_{h=1}^H \abs*{\wh \En_{\cD_h} \brk*{
    \prn[\big]{ 
    \brk{\wh \Delta_h Q^\star}(x_h, a_h, r_h, x'_{h+1})} \cdot 
    \check w _h(x_h,a_h)}} - \wh \En_{\cD_h} \brk*{ \alpha\ind{n}\cdot \prn[\big]{ \check w _h(x_h,a_h)}^2} \leq  H\beta\ind{n}.
\]
By \cref{ass:Q_realizability} and the definition of the hypothesis $\wh f$ returned by \Crmabo, we have:
\[
 \max_{w \in \widebar{\cW}} \sum_{h=1}^H \abs*{\wh \En_{\cD_h} \brk*{
    \prn[\big]{ 
    \whDwhf(x_h, a_h, r_h, x'_{h+1})} \cdot 
    \check w _h(x_h,a_h)}} - \wh \En_{\cD_h} \brk*{ \alpha\ind{n}\cdot \prn[\big]{ \check w _h(x_h,a_h)}^2} \leq  H\beta\ind{n},
\]
and in particular the bound that for all $w \in \widebar{\cW}$
\[
 \sum_{h=1}^H \wh \En_{\cD_h} \brk*{
    \prn[\big]{ 
    \whDwhf(x_h, a_h, r_h, x'_{h+1}) \cdot 
    \check w _h(x_h,a_h)}} \leq \sum_{h=1}^H \wh \En_{\cD_h} \brk*{ \alpha\ind{n}\cdot \prn[\big]{ \check w _h(x_h,a_h)}^2} + H\beta\ind{n}.
\]
Repeating the argument for \cref{lem:concentration2} (b), we can conclude that for all $w \in \widebar{\cW}$
\[
  \sum_{h=1}^H \En_{\mu\ind{1:n}_h}\brk*{\brk{\Delta_h \wh f}(x_h, a_h) \cdot \cw_h(x_h,a_h)} \leq \sum_{h=1}^H \frac{20}{\gamma\ind{n}} \En_{\mu\ind{1:n}_h} \brk*{\prn{\cw_h(x_h,a_h)}^2} 
  + \frac{7 H \beta\ind{n}}{18}.
\]
Applying this to $w\ind{h} \in \widebar{\cW}$ and to $-w\ind{h} \in \widebar{\cW}$, and again noting that the right-hand side is independent of the sign of $\wtil$, we can conclude that for each $h \in [H], w \in \widebar{\cW}$:
\[
  \abs*{ \En_{\mu\ind{1:n}_h}\brk*{\brk{\Delta_h \wh f}(x_h, a_h) \cdot \cw_h(x_h,a_h)}} \leq \frac{20}{\gamma\ind{n}} \En_{\mu\ind{1:n}_h} \brk*{\prn{\cw_h(x_h,a_h)}^2} 
  + \frac{7 H\beta\ind{n}}{18},
\]
Summing over $h \in [H]$ gives the desired bound. 
\end{proof}

\mabooffline*

\begin{proof}[\pfref{thm:mabooffline}] 
Let \(\wh \pi = \pi_{\wh f}\). Using the performance difference lemma (\cref{lemma:telescope-pd}), we note that 
\begin{align*}
J(\pi^\star) - J(\wh \pi) \leq {\sum_{h=1}^H \En_{d^{\pistar}_h}[- \Dfh(x_h,a_h)]} + {\sum_{h=1}^H \En_{d^{\pihat}_h}[\Dfh(x_h,a_h)]}.  \numberthis \label{eq:cmabo1} 
\end{align*} 
However, note that due to \cref{lemma:clip-decomp} and \cref{lem:bloop}, we have that  
\begin{align}
  \En_{d^{\pihat}_h}[(\Dfh(x_h,a_h))] &\leq \underbrace{\En_{\mu\ind{1:n}_h}\brk*{ ( \Dfh(x_h,a_h)) \clip{\frac{d^{\pihat}_h(x_h,a_h)}{\mu\ind{1:n}_h(x_h,a_h)}}{\gamma n}}}_{\text{(I)}} +  \frac{4}{\gamma n} \left\| \clip{\frac{d^{\pihat}_h}{\mu\ind{1:n}_h}}{\gamma n} \right\|_{1,d^{\pihat}_h} \label{eq:mabo-hat-pi}, 
  \intertext{and,} 
    \En_{d^{\pistar}_h} [(- \Dfh(x_h,a_h))] &\leq \underbrace{\En_{\mu\ind{1:n}_h}\brk*{(- \Dfh(x_h,a_h)) \clip{\frac{d^{\pi^\star}_h(x_h,a_h)}{\mu\ind{1:n}_h(x_h,a_h)}}{\gamma n}}}_{\text{(II)}} + \frac{4}{\gamma n} \left\| \clip{\frac{d^{\pistar}_h}{\mu\ind{1:n}_h}}{\gamma n} \right\|_{1,d^{\pistar}_h} \label{eq:mabo-pi-star}. 
\end{align}

We bound the terms \(\text{(I)}\) and \(\text{(II)}\) separately below. Before we delve into these bounds, note that using \pref{lem:mabo_concentration}, we have with probability at least \(1 - \delta\), 
\begin{equation}\label{eq:poplossmabo}
  \max_{w \in \cW} \sum_{h=1}^H \left(\left| \En_{\mu\ind{1:n}_h} \left[ \check{w}_h (\Dfh(x_h,a_h)) \right]\right| - \frac{20}{\gamma\ind{n}}\En_{\mu\ind{1:n}_h} \brk*{(\check{w}_h(x_h,a_h))^2} \right) \leq \frac{7}{18}H^2\beta\ind{n}, 
\end{equation} 
where \(\beta\ind{n} \ldef{} \frac{36\gamma n}{n - 1}\log(24|\cF||\cW|H^2/\delta)\). 
Moving forward, we condition on the event under which \eqref{eq:poplossmabo} holds.  

\paragraph{Bound on Term (I)} Define $w$ via $w_h \coloneqq \frac{d^{\pihat}_h}{\mu\ind{1:n}_h}$ and $\check{w}_h \coloneqq \clip{\frac{d^{\pihat}_h}{\mu\ind{1:n}_h} }{\gamma n}$, and note that due to \cref{ass:mabo-assumptions}, we have that  \(w \in \cW\). Thus, using \pref{eq:poplossmabo}, we get that 
\begin{align*} 
\sum_{h=1}^H \En_{\mu\ind{1:n}_h}[ (- \Dfh(x_h,a_h)) \check{w}_h(x_h,a_h)] &\leq \sum_{h=1}^H \left| \En_{\mu\ind{1:n}_h}[ (- \Dfh(x_h,a_h)) \check{w}_h(x_h,a_h)] \right| \\
&\leq \sum_{h=1}^H \frac{20}{\gamma\ind{n}} \En_{\mu\ind{1:n}_h} \brk*{ (\check{w}_h(x_h,a_h))^2 } + \frac{7}{18} H^2\beta\ind{n} \\  
&= \sum_{h=1}^H \frac{20}{\gamma\ind{n}} \nrm*{\clip{w_h}{\gamma n}}^2_{2,\mu\ind{1:n}_h} + \frac{7}{18}H^2\beta\ind{n} \\ 
&\leq \sum_{h=1}^H  \frac{40}{\gamma\ind{n}} \nrm*{ \clip{w_h}{\gamma n}}_{1,d^{\pihat}_h} + \frac{7}{18} H^2 \beta\ind{n},  
\end{align*} 
where the last step follows by \cref{lem:blorp} and the definition of $w_h$.

\paragraph{Bound on Term (II)} Define $w^\star_h \coloneqq \frac{d^{\pistar}_h}{\mu\ind{1:n}_h}$ and $\check{w}^\star_h \coloneqq \clip{\frac{d^{\pistar}_h}{\mu\ind{1:n}_h} }{\gamma n}$, and again note that due to \cref{ass:mabo-assumptions}, we have that  \(w_h \in \cW_h\). Thus, using \pref{eq:poplossmabo}, and repeating the same arguments as above, we get that 
\begin{align*}
\text{(II)} &\leq \sum_{h=1}^H \frac{40}{\gamma\ind{n}} \nrm*{\clip{w^\star_h}{\gamma n}}_{1,d^{\pi^\star}_h} + \frac{7}{18} H^2 \beta\ind{n},  
\end{align*} 

Plugging the two bounds above in \pref{eq:mabo-hat-pi} and \pref{eq:mabo-pi-star}, and then further in \pref{eq:cmabo1}, and using the definitions for \(w_h\) and \(w_h^\star\), we get that with probability at least \(1 - \delta\), 
\begin{align*}
J(\pi^\star) - J(\wh \pi) &\leq \sum_{h=1}^H \frac{40}{\gamma{}n} \left\| \clip{\frac{d^{\pi^\star}_h}{\mu\ind{1:n}_h}}{\gamma n} \right\|_{1,d^{\pi^\star}_h}  +  \sum_{h=1}^H  \frac{5}{\gamma{}n} \left\| \clip{\frac{d^{\pihat}_h}{\mu\ind{1:n}_h}}{\gamma n} \right\|_{1,d^{\pihat}_h} + \frac{14}{18} H^2 \beta\ind{n} \\
&= \sum_{h=1}^H \frac{40}{\gamma\ind{n}} \left(\left\| \clip{\frac{d^{\pi^\star}_h}{\mu\ind{1:n}_h}}{\gamma n} \right\|_{1,d^{\pi^\star}_h}  +  \left\| \clip{\frac{d^{\pihat}_h}{\mu\ind{1:n}_h}}{\gamma n} \right\|_{1,d^{\pihat}_h} \right) + \frac{14}{18} H^2 \beta\ind{n} \\
&= \sum_{h=1}^H \frac{40}{\gamma\ind{n}} \left(\Ccc{n}{h}{\pi^\star}{\mu\ind{1:n}}{\gamma n}  +  \Ccc{n}{h}{\pihat}{\mu\ind{1:n}}{\gamma n}  \right) + 28 H^2\gamma\frac{n}{n-1}\log(24|\cF||\cW|H^2/\delta) \\
&\leq \sum_{h=1}^H \frac{40}{\gamma{}n} \left(\Ccc{n}{h}{\pi^\star}{\mu\ind{1:n}}{\gamma n}  +  \Ccc{n}{h}{\pihat}{\mu\ind{1:n}}{\gamma n}  \right) + 56 H^2 \gamma \log(24|\cF||\cW|H^2/\delta),\\
\end{align*}
which establishes that the algorithm is \Ccbounded for scale $\gamma$, with scaling functions $\ascale = \frac{40}{\gamma}$ and $\bscale = 56 H^2 \gamma \log(24|\cF||\cW|H^2/\delta)$.

\end{proof}

\hyglow*
\begin{proof}[\pfref{cor:hyglow}]
This follows by combining \cref{thm:mabooffline} with \cref{cor:alphabetagamma}.
\end{proof}

\subsubsection{Proofs for Fitted Q-Iteration (\Fqi)}\label{app:fqi-proof}

In this section we prove \cref{thm:fqi-offline}.

We quote the following generalization bound for least squares regression in the adaptive setting.  

\begin{lemma}[{Least squares generalization bound; \citet[Lemma 3]{song2022hybrid}}]\label{lemma:least-squares}
Let $R > 0, \delta \in (0,1)$, and $\cH : \cX \mapsto [-R,R]$ a class of real-valued functions. Let $\cD = \{(x_1,y_1)\ldots(x_T,y_T)\}$ be a dataset of $T$ points where $x_t \sim \rho_t(x_{1:t-1},y_{1:t-1})$ and $y_t = h^\star(x_t) + \varepsilon_t$ for some realizable $h^\star \in \cH$ and $\varepsilon_t$ is conditionally mean-zero, i.e. $\bbE[y_t \mid x_t] = h^\star(x_t)$. Suppose $\max_t |y_t| \leq R$ and $\max_x |h^\star(x)| \leq R$. Then the least squares solution $\widehat{h} \in \argmin_{h \in \cH} \sum_{t=1}^T (h(x_t) - y_t)^2$ satisfies that with probability at least $1-\delta$, 
\[
\sum_{t=1}^T \En_{x \sim \rho_t}\left[ (\widehat{h}(x) - h^\star(x))^2\right] \leq 256 R^2 \log(2 |\cH| / \delta).
\]
\end{lemma}

Using the above theorem, we can show the following concentration result for \Fqi.

\begin{lemma}[Concentration bound for \Fqi]\label{lem:fqi-martingale}
With probability at least $1-\delta$, we have that for all $h \in [H]$,
\begin{align*}
 \En_{\mu\ind{1:n}_h} \brk*{(\Dfh(x_h,a_h))^2} &= \frac{1}{n}\sum_{i=1}^n \En_{(x_h,a_h) \sim \mu\ind{i}_h}\brk*{ \prn*{\widehat{f}_h(x_h,a_h) - \brk{\cT_h \widehat{f}_{h+1}}(x_h,a_h)}^2} \\
 	&\leq 1024 \frac{\log(2|\cF|H/\delta)}{n}.
\end{align*}
\end{lemma}

\begin{proof}[\pfref{lem:fqi-martingale}] 
Fix $h+1$. Consider the regression problem induced by the dataset $\cD_h = \{(z\ind{i}_h,y\ind{i}_h)\}_{i=1}^n$ where $z\ind{i}_h = (x\ind{i}_h,a\ind{i}_h) \sim \mu\ind{i}_h$ and $y\ind{i}_h = r\ind{i} + \max_{a'}\widehat{f}_{h+1}(x\ind{i}_{h+1},a')$. This problem is realizable via the regression function $\bbE[y\ind{i}_h \mid z\ind{i}_h] = h^\star(z\ind{i}_h) = \cT \widehat{f}_{h+1}(z\ind{i}_h) \in \cF$, and satisfies that $|y\ind{i}_h| \leq 2$, $|h^\star(z\ind{i}_h)| \leq 2$. In this regression problem, the least squares solution from \cref{lemma:least-squares} is precisely the \Fqi solution, so by \cref{lemma:least-squares} we have
\begin{align*}
\En_{(x_h,a_h) \sim \mu\ind{1:n}_h}\prn*{\widehat{f}_{h}(x_h,a_h) - \brk{\cT \widehat{f}_{h+1}}(x_h,a_h))}^2 &= \frac{1}{n} \sum_{i=1}^n \En_{(x_h,a_h) \sim \mu\ind{i}_h} \prn*{\widehat{f}_{h}(x_h,a_h) - \brk{\cT \widehat{f}_{h+1}}(x_h,a_h)}^2 \\
	&\leq 1024 \frac{\log(2|\cF|/\delta)}{n},
\end{align*} 
with high probability. Taking a union bound over $h \in [H]$ gives the desired result.
\end{proof}

\fqioffline*

\begin{proof}[\pfref{thm:fqi-offline}] 
Let \(\wh \pi = \pi_{\wh f}\). Using the performance difference lemma (given in \cref{lemma:telescope-pd}), we note that 
\begin{align*}
J(\pi^\star) - J(\wh \pi) \leq {\sum_{h=1}^H \En_{d^{\pistar}_h}[- \Dfh(x_h,a_h)]} + {\sum_{h=1}^H \En_{d^{\pihat}_h}[\Dfh(x_h,a_h)]}.  \numberthis \label{eq:cmabo1} 
\end{align*} 
However, note that due to \cref{lemma:clip-decomp} and \cref{lem:bloop}, we have that  
\begin{align}
  \En_{d^{\pihat}_h}[(\Dfh(x_h,a_h))] &\leq \underbrace{\En_{\mu\ind{1:n}_h}\brk*{ ( \Dfh(x_h,a_h)) \clip{\frac{d^{\pihat}_h(x_h,a_h)}{\mu\ind{1:n}_h(x_h,a_h)}}{\gamma n}}}_{\text{(I)}} +  \frac{4}{\gamma n} \left\| \clip{\frac{d^{\pihat}_h}{\mu\ind{1:n}_h}}{\gamma n} \right\|_{1,d^{\pihat}_h} \label{eq:mabo-hat-pi}, 
  \intertext{and,} 
    \En_{d^{\pistar}_h} [(- \Dfh(x_h,a_h))] &\leq \underbrace{\En_{\mu\ind{1:n}_h}\brk*{(- \Dfh(x_h,a_h)) \clip{\frac{d^{\pi^\star}_h(x_h,a_h)}{\mu\ind{1:n}_h(x_h,a_h)}}{\gamma n}}}_{\text{(II)}} + \frac{4}{\gamma n} \left\| \clip{\frac{d^{\pistar}_h}{\mu\ind{1:n}_h}}{\gamma n} \right\|_{1,d^{\pistar}_h} \label{eq:mabo-pi-star}. 
\end{align} 
	
\paragraph{Bound on Term (I)} Note that 

\begin{align*}
 \En_{\mu\ind{1:n}_h}\brk*{ ( \Dfh(x_h,a_h)) \clip{\frac{d^{\pihat}_h(x_h,a_h)}{\mu\ind{1:n}_h(x_h,a_h)}}{\gamma n}}   &\leq \sqrt{\En_{\mu\ind{1:n}_h}\brk*{\prn*{\Dfh(x_h,a_h)}^2}\En_{\mu\ind{1:n}_h}\brk*{\left(\clip{\frac{d^{\pihat}_h(x_h,a_h)}{\mu\ind{1:n}_h(x_h,a_h)} }{\gamma n}\right)^2}} \\ 
  &\leq \sqrt{2048 \frac{\log(2|\cF|H )}{n}} \left\| \clip{\frac{d^{\pihat}_h}{\mu\ind{1:n}_h} }{\gamma n} \right\|_{2,\mu\ind{1:n}_h}, 
\end{align*}
where the second line follows from Cauchy-Schwarz and the last line follows by \cref{lem:fqi-martingale}. AM-GM inequality implies that 
\begin{align*}
\sqrt{\frac{2048 \log(2|\cF| H)}{n}} \left\| \clip{\frac{d^{\pihat}_h}{\mu\ind{1:n}_h} }{n} \right\|_{2,\mu\ind{1:n}_h} &\leq \frac{1}{2}\left( \frac{1}{\gamma n} \left\| \clip{\frac{d^{\pihat}_h}{\mu\ind{1:n}_h} }{\gamma n} \right\|^2_{2,\mu\ind{1:n}_h} + 2048 \log(2|\cF| H) \gamma \right) \\
&\leq \frac{1}{2}\left( \frac{4}{\gamma n} \left\| \clip{\frac{d^{\pihat}_h}{\mu\ind{1:n}_h} }{\gamma n} \right\|_{1,d^{\widehat{\pi}}_h} + 2048 \log(2|\cF| H) \gamma\right),
\end{align*}
where the last line is due to \cref{lem:blorp}. 
\ascomment{Since the constants will change in \cref{lem:blorp}, they should also be changed here!} 
\paragraph{Bound on Term (II)} Repeating the same argument above for \(\pistar\), we get that 
\begin{align*}
\text{(II)} &\leq \frac{2}{\gamma n} \left\| \clip{\frac{d^{\pistar}_h}{\mu\ind{1:n}_h} }{\gamma n} \right\|_{1,d^{\pistar}_h} + 1024 \log(2|\cF| H) \gamma. 
\end{align*}

Combining the above bounds implies that 
\begin{align*}
J(\pistar) - J(\wh \pi) &\leq \sum_{h=1}^H \frac{2}{\gamma n}\left(\left\| \clip{\frac{d^{\pi^\star}_h}{\mu\ind{1:n}_h}}{\gamma n} \right\|_{1,d^{\pi^\star}_h} + \left\| \clip{\frac{d^{\pihat}_h}{\mu\ind{1:n}_h} }{\gamma n} \right\|_{1,d^{\widehat{\pi}}_h}\right) + 2048\log(2|\cF|H) \gamma \\
&=\sum_{h=1}^H \frac{2}{\gamma n}\left(\Ccc{n}{h}{\pi^\star}{\mu\ind{1:n}}{\gamma} + \Ccc{n}{h}{\pihat}{\mu\ind{1:n}}{\gamma}\right) + 2048\log(2|\cF|H) \gamma,
\end{align*}
which shows that \Fqi is \Ccbounded at scale $\gamma$ under \cref{ass:completeness}, with scaling functions $\ascale = \frac{2}{\gamma}$ and $\bscale = 2048\log(2|\cF|H) \gamma$.

It remains to show the stated risk bound when \textsc{Fqi} is applied within \HtO. This follows by applying \cref{cor:alphabetagamma}. 
\end{proof}

\subsubsection{Proofs for Model-Based MLE}\label{app:model-based-proof}

In this section we prove \cref{thm:model-based}. We quote the following generalization bound for maximum likelihood estimation (MLE) in the adaptive setting.

\begin{lemma}[{MLE generalization bound; \citet[Theorem 18]{agarwal2020flambe}}]
	Consider a sequential conditional probability estimation setting with an instance space $\cZ$ and a target space $\cY$ and conditional density $p(y \mid z)$. We are given $\cF: (\cZ \times \cY) \rightarrow \bbR$. We are given a dataset $D = \{(z_t,y_t)\}_{t=1}^T$ where $z_t \sim \cD_t = \cD_t(z_{1:t-1},y_{1:t-1})$ and $y_t \sim p(\cdot \mid z_t)$. Fix $\delta \in (0,1)$, assume $|\cF| < \infty$ and there exists $f^\star \in \cF$ such that $f^\star(y \mid z) = p(y \mid z)$. Then, with probability at least $1-\delta$, 
	\[
		\sum_{t=1}^T \En_{z \sim \rho_t} \nrm*{\widehat{f}(\cdot\mid z) - f^\star(\cdot\mid z)}^2_{\mathrm{tv}} \leq 2 \log\prn{|\cF|/\delta},
	\]	
	where 
	\[
		\widehat{f} \in \argmax_{f \in \cF} \sum_{t=1}^T \log f(y_t \mid z_t).
	\]
\end{lemma} 

Recall the notation that each model $M\in\cM$ defines a conditional probability distribution of the form $M(r, x' \mid x,a)$. We apply the above generalization bound to our setting to obtain the following.

\begin{corollary}\label{cor:model-bellman}
	Under \cref{ass:model-realizability}, we have that the model-based maximum likelihood estimator of Equation \cref{eq:model-mle} satisfies
	\[
	\forall h \in [H]: \,\,	\En_{\mu\ind{1:n}_h}\brk*{\nrm*{\widehat{M}_h(\cdot \mid x_h,a_h) - M^\star_h(\cdot \mid x_h,a_h)}^2_{\mathrm{tv}}} \leq 2 \frac{\log(|\cM|H/\delta)}{n}.
	\]
\end{corollary}
Note that we have taken an extra union bound over $H$ so that the MLE succeeds at each layer.

This is easily seen to imply an error bound between the associated Bellman operators.
\begin{corollary}\label{cor:asdf}
Let $[\wh{\cT}f](x,a) = \En_{(r,x') \sim \widehat{M}(\cdot\mid x,a)}\brk*{r + \max_{a'}f(x',a')}$ denote the Bellman optimality operator of $\widehat{M}$, and $\cT$ denote the Bellman optimality operator of $M^\star$. Then we have that for all $h \in [H]$ and for all $f: \cX \times \cA \rightarrow [0,1]$, 
\[
\En_{\mu\ind{1:n}_h}\brk*{\prn*{\brk{\wh{\cT}_h f}(x_h,a_h) - \brk{\cT_h f}(x_h,a_h)}^2} \leq 8\frac{\log(|\cM|H/\delta)}{n}.
\]
\end{corollary}
\begin{proof}[\pfref{cor:asdf}]
Notice that
\begin{align*}
	\brk{\wh{\cT}_hf}(x,a) - \brk{\cT_h f}(x,a) &= \En_{(r,x') \sim \widehat{M}(x,a)}\brk*{r + \max_{a'}f(x',a')} - \En_{(r,x') \sim M^\star(x,a)}\brk*{r + \max_{a'}f(x',a')} \\
	&\leq 2\nrm*{\widehat{M}_h(\cdot \mid x,a) - M^\star_h(\cdot \mid x,a)}_{\mathrm{tv}}.
\end{align*}
\end{proof}

\modelbased*

\begin{proof}[\pfref{thm:model-based}] We note that \cref{cor:model-bellman} implies a squared Bellman error bound for $Q^\star_{\widehat{M}}$, the optimal value function for $\widehat{M}$, since 
\[
	\forall x,a: \,\prn*{\brk{\wh{\cT}_hQ^\star_{\widehat{M},h+1}}(x,a) - \brk{\cT_h Q^\star_{\widehat{M},h+1}}(x,a)}^2 = 	\prn*{Q^\star_{\widehat{M},h}(x,a) - \brk{\cT_h Q^\star_{\widehat{M},h+1}}(x,a)}^2,
\]
by the optimality equation for $Q^\star_{\widehat{M}}$. Thus we have 
\begin{equation}\label{eq:squared-error-hatM}
	\forall h \in [H]: \, \En_{\mu\ind{1:n}_h}\brk*{\prn*{Q^\star_{\widehat{M},h}(x_h,a_h) - \brk{\cT_h Q^\star_{\widehat{M},h+1}}(x_h,a_h)}^2} \leq 8\frac{\log(|\cM|H/\delta)}{n}.
\end{equation}
This is enough to repeat the proof of $\mathsf{CC}$-boundedness for \Fqi (\cref{thm:fqi-offline}), with $Q^\star_{\widehat{M}}$ taking the place of the \Fqi solution $\widehat{f}$. Indeed, the only algorithmic property that we used for \Fqi was \cref{lem:fqi-martingale}, which also holds for $Q^\star_{\widehat{M}}$ by Equation \cref{eq:squared-error-hatM}. Tracking the slightly different constants resulting gives us the desired values for $\ascale$ and $\bscale$.
\end{proof}

\end{document}